\documentclass[a4paper,english,11pt]{article}

\usepackage{./style/arxiv}

\usepackage{amsmath,amsfonts,bm,amssymb,nicefrac,algpseudocode}
\usepackage{algorithm}
\usepackage{amsthm}
\usepackage{booktabs}          
\usepackage{graphicx}          
\usepackage{subcaption}        
\usepackage{multirow}        
\usepackage{stmaryrd}        
\usepackage{yhmath}        

\def\eqref#1{Eq.~\ref{#1}}   

\def\vzero{{\bm{0}}}

\def\vc{{\bm{c}}}
\def\vd{{\bm{d}}}
\def\ve{{\bm{e}}}

\def\vn{{\bm{n}}}

\def\vr{{\bm{r}}}
\def\vs{{\bm{s}}}

\def\vx{{\bm{x}}}
\def\vy{{\bm{y}}}

\def\mA{{\bm{A}}}
\def\mB{{\bm{B}}}
\def\mC{{\bm{C}}}
\def\mD{{\bm{D}}}
\def\mE{{\bm{E}}}

\def\mI{{\bm{I}}}

\def\mK{{\bm{K}}}
\def\mL{{\bm{L}}}
\def\mM{{\bm{M}}}

\def\mP{{\bm{P}}}
\def\mQ{{\bm{Q}}}

\def\mS{{\bm{S}}}
\def\mT{{\bm{T}}}

\def\mW{{\bm{W}}}
\def\mX{{\bm{X}}}
\def\mY{{\bm{Y}}}

\def\mSigma{{\bm{\Sigma}}}

\newcommand{\E}{\mathbb{E}}

\newcommand{\bbP}{\mathbb{P}}
\newcommand{\R}{\mathbb{R}}

\DeclareMathOperator*{\argmin}{\arg\!\min}

\newtheoremstyle{mythmstyle} 
{\topsep}    
{\topsep}    
{\itshape}   
{0pt}        
{\bfseries}  
{}           
{ }          
{}           
\theoremstyle{mythmstyle}
\newtheorem{theorem}{Theorem}

\newtheorem{lemma}[theorem]{Lemma}

\newtheorem{proposition}[theorem]{Proposition}

 \newtheorem{assumption}{Assumption}
 \newtheorem*{proposition*}{Proposition}
 \newtheorem*{theorem*}{Theorem}
  \newtheorem*{lemma*}{Lemma}

\newcommand{\cN}{\mathcal N}

\newcommand{\cG}{\mathcal G}
\newcommand{\cL}{\mathcal L}

\usepackage{physics}
\usepackage{hyperref}
\usepackage{cleveref}

\newcommand{\pairwisecomparisons}{recursive residuals }
\newcommand{\pairwisecomparisonsnospace}{recursive residuals}

\usepackage{mathtools}
\usepackage{lipsum}
\usepackage{wrapfig}
\usepackage{ragged2e}

\usepackage[inline]{enumitem}   
\makeatletter
\newcommand{\inlineitem}[1][]{%
\ifnum\enit@type=\tw@
    {\descriptionlabel{#1}}
  \hspace{\labelsep}%
\else
  \ifnum\enit@type=\z@
       \refstepcounter{\@listctr}\fi
    \quad\@itemlabel\hspace{\labelsep}%
\fi}
\makeatother
\begin{document}
\title{Multi-View Causal Discovery without Non-Gaussianity: Identifiability and Algorithms}
\date{}

\author{
    {Ambroise Heurtebise}$^{1}$\footnote{
    Corresponding author \href{mailto:[ambroise.heurtebise@inria.fr]]}{ambroise.heurtebise@inria.fr}},
    {Omar Chehab}$^{2,3}$,
    {Pierre Ablin}$^{4}$, \\
    {Alexandre Gramfort}$^{1}$,
    {Aapo Hyv{\"a}rinen}$^{5}$
    \\ [.8em]
    $^1${Inria, CEA, University of Paris-Saclay, France}
    \\
    $^2${Carnegie Mellon University, Pittsburgh PA, USA}
    \\
    $^3${ENSAE, CREST, IP Paris, France}
    \\
    $^4${Apple}
    \\
    $^5${University of Helsinki, Finland}
}

\addtocontents{toc}{\protect\setcounter{tocdepth}{0}}

\maketitle

\begin{abstract}
    
Causal discovery is a difficult problem that typically relies on strong assumptions on the data-generating model, such as non-Gaussianity. In practice, many modern applications provide multiple related views of the same system, which has rarely been considered for causal discovery. Here, we leverage this multi-view structure to achieve causal discovery with weak assumptions. We propose a multi-view linear Structural Equation Model (SEM) that extends the well-known framework of non-Gaussian disturbances by alternatively leveraging correlation over views. We prove the identifiability of the model for acyclic SEMs. Subsequently, we propose several multi-view causal discovery algorithms, inspired by single-view algorithms (DirectLiNGAM, PairwiseLiNGAM, and ICA-LiNGAM). The new methods are validated through simulations and applications on neuroimaging data, where they enable the estimation of causal graphs between brain regions.

\end{abstract}



\section{Introduction}
\label{sec:intro}

Causal discovery is a fundamental problem in scientific data analysis as well as several technological applications~\citep{peters2017elements}. The basic problem is that we are given a number of observed variables, and we need to infer which variables cause which, and with what “connection strengths”. In the typical case, we only observe the data passively without any possibility of performing interventions, and this purely observational quality of the data makes the problem difficult.

Often, the problem is formalized as a structural equation model (SEM)~\citep{bollen1989structural}, also called a functional causal model~\citep{pearl2009causalitybook}. Then the question is how to estimate the SEM parameters using statistical theory. In fact, before even considering estimation methods, we need to know if it is at all possible to solve the problem. This is the question of \textit{identifiability} of the model: can the parameters that describe the causal relations and directions be (uniquely) estimated?

While deep learning research has recently proposed completely nonlinear methods for causal discovery \citep{Monti19UAI,JMLR:v26:24-0194,Morioka24,pmlr-v235-xu24ac}, it is probably true that real-world applications overwhelmingly use \textit{linear} methods. This is in no small part because nonlinear models have their own limitations when deployed: they require much more data and more computation to train, the identifiability results need more assumptions, and such methods are likely to get stuck in local minima. Furthermore, linear methods are easier to interpret. This is why we focus on linear SEMs in what follows.

A well-known fact is that if all variables are Gaussian and the model is linear, identifiability of the SEM is very problematic. Some of the earliest work showed that the directions can sometimes 
be recovered by an analysis of the \textit{conditional independencies} between variables~\citep{spirtes2001causation}; however, this is only possible in some special cases and notably not possible when observing only two variables. 
Recent literature has, therefore, focused on different departures from the linear-Gaussian framework to achieve identifiability. A major advance was to consider a linear model together with \textit{non-Gaussianity} \citep{shimizu2006linear}, which leads to full identifiability of the model under weak conditions such as acyclicity; however, such non-Gaussianity may not be strong enough in many data sets.
A related framework was proposed by assuming that the cause undergoes a \textit{nonlinear transform}, while the disturbance is still additive \citep{hoyer2008nonlinear}; however the nonlinearity may not be strong enough in many applications, and such nonlinearity slightly contradicts the linear additivity of the disturbance. 
Further related frameworks were proposed by \citep{peters2014identifiability,Zhang09UAI,khemakhem2020variational};
continuous optimization approaches have also been considered by \citet{zheng2018dags,zhu2019causal}.

An alternative recent framework is given by the \textit{multi-view} setting, where data is collected from different sources so that the data points are paired. The correlations between the views can then be leveraged for estimation and identifiability. However, current approaches using multi-view structure are scarce and limited;  
either they 
ignore dependencies across views~\citep{shimizu2012joint}
or make strong assumptions about these dependencies \citep{chen2024individual}. 
We further note that methods in the \emph{multi-domain} setting are fundamentally different from the multi-view case since multi-domain data is not paired: the distinction 
is detailed in Section~\ref{sec:discussion} below.

Here, we consider the linear 
multi-view SEM, providing a general theory and algorithms that leverage the multiple views. 
Our main contributions are:
\begin{enumerate}
    
    \item We generalize three
    of the most well-known \textit{algorithms} for single-view causal discovery to the multi-view setting: PairwiseLiNGAM, DirectLiNGAM and ICA-LiNGAM. The generalizations of PairwiseLiNGAM  and DirectLiNGAM result in completely new and fast 
    causal discovery using second-order statistics only, while our adaptation of ICA-LiNGAM improves on \citet{chen2024individual}. 
    
    \item We show that our algorithms are \textit{consistent} under weak conditions: we can  replace the non-Gaussianity assumption on the disturbances by conditions on correlations between the views and ``diversity" of such Second-Order Statistics (SOS). Still, if the disturbances are non-Gaussian, we generalize \citet{shimizu2012joint,chen2024individual}.

    \item The proofs above directly lead to rigorous \textit{identifiability} results on the model, including cases where all the disturbances are Gaussian, or, more generally, identifiability results based on SOS only. Importantly, only weak conditions on SOS are necessary for identifiability, greatly generalizing previous work. 
    
    \item We make our algorithms usable by practitioners by providing them as \textit{open-source}, and demonstrate their usefulness by benchmarking them on real neuroimaging data.
    
\end{enumerate}

\section{Background}
\label{sec:background}

\paragraph{Causal ordering}
In the following, we consider a random vector $\vx \in \R^p$ whose entries are nodes of a directed acyclic graph (DAG). 
The DAG is represented by an adjacency matrix $\mB \in \R^{p \times p}$, where non-zero values indicate the edges of the graph. The entries of $\vx$ can be reordered such that each ``causes" the next. This is formalized by the adjacency matrix admitting a special decomposition $\mB = \mP^\top \mT \mP$, where $\mT$  is  a strictly lower triangular matrix  and $\mP$ is a permutation matrix referred to as the \emph{causal ordering}. In general, the causal ordering is not unique. 




\paragraph{LiNGAM: A single-view model for causal discovery}
We start by defining a linear structural equation model (SEM), or a functional causal model.
The data follows
\begin{align}
    \label{eq:causal_model_original}
    \vx = \mB \vx + \ve
\end{align}
and the entries $e_1,\ldots,e_p$ of the random vector $\ve \in \R^p$ are independent noise terms, or \emph{disturbances}.
Causal discovery consists in inferring the parameters $\mB$ of the model, from observations of $\vx$.
Yet, the identifiability of the
model and therefore the uniqueness of $\mB$ is not straightforward.
In fact, it is well-known that with Gaussian disturbances, the model is unidentifiable in general~\citep{richardson2002ancestral,genin2021identifiability}.
A major advance by \citet{shimizu2006linear} was assuming \textit{non-Gaussian} disturbances, leading to their Linear Non-Gaussian Acyclic Model (LiNGAM). 

\paragraph{Identifiability of LiNGAM} 
\citet{shimizu2006linear} showed that the LiNGAM model is identifiable in terms of the matrix $\mB$. 
A rigorous re-statement of this result is given in  Appendix~\ref{app:ssec:identifiability_lingam}.
A question that has received less attention is under what conditions the model is identifiable in terms of the causal ordering $\mP$. 
In fact, it is not in general: there may exist many permutation matrices $\mP$ and strictly lower triangular matrices $\mT$ such that the generated data has the same distribution and the generating permutation cannot be identified.
For instance, in the degenerate case $\mB = \mT = \vzero$, any permutation matrix $\mP$ gives the same data distribution. 
As is well-known, a DAG in general defines only a \emph{partial} ordering in the sense that for some pairs of variables, we cannot necessarily say which is “earlier” and which is “later”. Thus, if we want to make the causal ordering unique --- \textit{i.e.} to obtain a \emph{total} ordering --- we need further assumptions, as considered in later sections. 

 \paragraph{Estimation of LiNGAM}
Since the LiNGAM model was proposed, two important categories of algorithms for estimating this model have been developed. 
First, algorithms based on \textit{recursively} computing the \textit{residuals} of pairwise regressions include DirectLiNGAM \citep{shimizu2011directlingam} and PairwiseLiNGAM \citep{Hyva13JMLR}. These compare pairs of variables from $\vx$ and deduce the causal direction between them based on regressing the variables in the two directions. By aggregating this information, they finally determine a causal ordering $\mP$. Once an ordering is found, the matrix $\mT$ can easily be obtained with conventional methods such as least-squares (LS) regression, which then gives $\mB = \mP^\top \mT \mP$.

\begin{figure*}
    \centering
    \includegraphics[width=\linewidth]{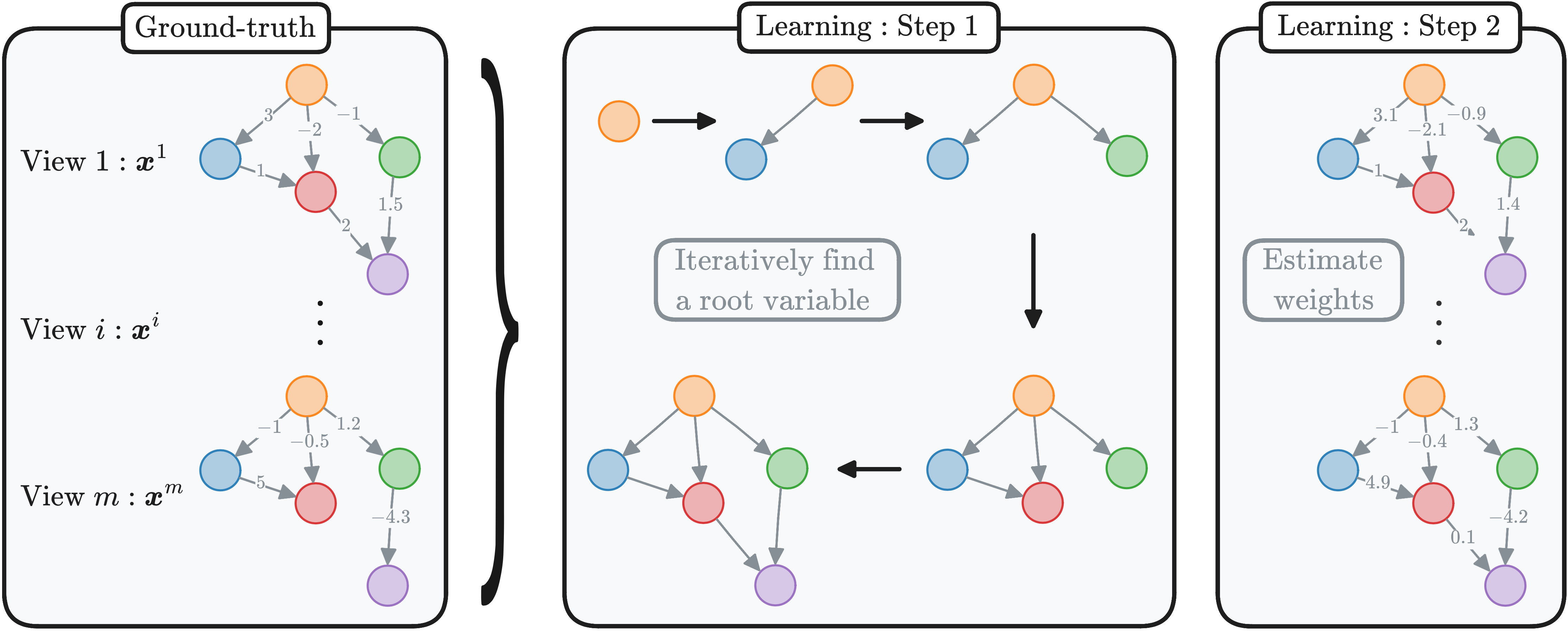}
    \caption{
    Left panel:
    Visual representation of our multi-view LiMVAM model in~\eqref{eq:multivariate_model}. 
    Each view of the data is described by a causal graph, nodes represent variables, and disturbances are not illustrated.
    We see that the causal ordering given by the DAGs is shared over views $1\ldots m$, but causal weights can differ.
    Middle and right panels: Main steps of the PairwiseLiMVAM and DirectLiMVAM algorithms. In Step 1, the shared causal ordering is progressively determined. In Step 2, view-specific causal weights are estimated.}
    \label{fig:didactic}
    \vspace{-0.4cm}
\end{figure*}

The second category is the ICA-based algorithm that was proposed in the original LiNGAM paper~\citep{shimizu2006linear}. It is based on the observation that the LiNGAM model can be rewritten as a latent variable model, in particular an ICA model~\citep{Hyvabook} as
\begin{align}
    \label{eq:ica_model}
    \vx = \mA \ve \;.
\end{align}
Again, the entries in $\ve$ are independent and non-Gaussian, and the “mixing” matrix $\mA$ expresses how the data is generated from the latent variable $\ve$. Many methods developed for ICA can be used to estimate the matrix $\mA$. However, it is important to consider the special structure of the mixing, since $\mA = (\mI - \mB)^{-1}$, where $\mB$ is a DAG~\citep{shimizu2006linear}. 
This method first recovers the adjacency matrix $\mB$, and then estimates a causal ordering $\mP$ from $\mB$.



\section{A Multi-View Model for Causal Discovery} 
\label{sec:model}


\paragraph{LiMVAM: A multi-view model for causal discovery}
We now extend the single-view
model in~\eqref{eq:causal_model_original} to the multi-view setting, leading to a Linear Multi-View Acyclic Model (LiMVAM). Suppose we observe $m$ different views of a $p$-dimensional vector. We write this as
%
\begin{equation}
\label{eq:multivariate_model}
    \vx^i = \mB^i \vx^i + \ve^i
\end{equation}
for views $i \in [\![1, m]\!]$. The  $\mB^i$ are view-specific adjacency matrices of DAGs, and $\ve^i$ are view-specific disturbance vectors (with entries of non-zero variance).

Similarly to~\eqref{eq:ica_model}, the LiMVAM model can equivalently be written as a latent-variable model
\begin{align}
    \label{eq:multi_view_ica_model}
    \vx^i = \mA^i \ve^i
\end{align}
where the mixing matrix has a special structure, $\mA^i = (\mI - \mB^i)^{-1}$, which we will leverage later in Section~\ref{sec:estimating_model_using_ICA}. 

\paragraph{Model assumptions} To exploit the multi-view structure, we follow~\citet{shimizu2012joint, chen2024individual} and assume that all adjacency matrices $\mB^i$ \emph{share the same causal ordering}.
Formally, they admit a decomposition $\mB^i = \mP^\top \mT^i \mP$ where $\mT^i$ is strictly lower triangular and $\mP$ is a permutation matrix that is shared among the views $i$. 
Note that this is not merely a modeling convenience; it reflects many realistic scenarios such as brain responses of different subjects to a common stimulus, where it is expected that the causal structure $\mP$ is shared (directions of effect are similar), while variability arises in the strength of the responses $\mT^i$.
Moreover, as is commonly assumed in causal discovery, we require the components of each $\ve^i$ to be mutually independent, for any given $i$. (In contrast, correlations between views $i$ and $i'$ will be seen to be essential for identifiability.)
Contrary to most existing approaches, we do not impose any restriction on the distribution of the disturbances, such as non-Gaussianity, which greatly broadens the applicability of the model. 
Figure~\ref{fig:didactic} (left panel) illustrates the proposed model. 


We emphasize that the multi-view framework considered here is fundamentally different from the multi-domain or multi-environment settings, as it assumes that observations across different views are paired; a detailed discussion is deferred to Section~\ref{sec:discussion}.

The next two sections develop efficient methods for estimating the adjacency matrices $\mB^i$, together with identifiability theory. 
We begin in Section~\ref{sec:estimating_model_using_pairwise_comparisons} with methods that estimate the causal ordering using \pairwisecomparisons and using only second-order statistics: these lead to our PairwiseLiMVAM and DirectLiMVAM algorithms, which extend the single-view algorithms PairwiseLiNGAM~\citep{Hyva13JMLR} and DirectLiNGAM~\citep{shimizu2011directlingam}. In Section~\ref{sec:estimating_model_using_ICA}, we propose a modified version of the ICA-based multi-view model from~\citet{chen2024individual}.


\section{Estimation by SOS and \pairwisecomparisons}
\label{sec:estimating_model_using_pairwise_comparisons}

Here, we show how to estimate the model parameters using Second-Order Statistics (SOS) only, and how the SOS alone lead to identifiability. To do so, we will rely on a principle
that we call \textit{\pairwisecomparisonsnospace}. While it has been used by \citet{shimizu2011directlingam,Hyva13JMLR,shimizu2012joint}, we provide here a unified treatment.

\subsection{Causal ordering using \pairwisecomparisons}
\label{ssec:general_principle}

Methods based on what we call \pairwisecomparisons 
are built on the following principle: 1) we can find a root variable (\textit{i.e.} one with no parents) by analyzing the causal directions between all the pairs of variables, and 2) when all other variables are regressed on a root variable, the resulting residuals preserve the causal ordering of the original
variables~\citep[Lemma 2 and Corollary~1]{shimizu2011directlingam}. In the multi-view case, the same principles apply 
since all views are assumed to share a common ordering. A formal proof is provided in Appendix~\ref{app:sssec:proof_of_theorem_part_1} and~\ref{app:sssec:proof_of_theorem_part_2}.

Algorithms based on \pairwisecomparisons therefore follow the same recursive scheme to recover the causal ordering: (i) perform pairwise regressions to obtain residuals, (ii) identify a root variable according to a specific criterion of causal direction computed on the residuals, and (iii) remove the selected root and repeat the procedure on the residuals until all variables are ordered. In practice, we will use Ordinary LS to compute the residuals in step (i). The main difference between algorithms lies in step (ii), \textit{i.e.} in how the root variable is selected. 


Since a root variable has no parents, we can find one by testing the causal direction for every pair of variables and taking as a root a variable that is never inferred to be an effect. 
(A practical way to select the root variable from such tests with finite samples is described in Appendix~\ref{app:ssec:find_root_from_M}.)
Thus, root selection reduces to designing a criterion  of causal direction capable of reliably distinguishing the causal direction between two variables.  In the following, we introduce two new such criteria in the multi-view setting, after first discussing 
a condition to ensure that these criteria are consistent. 

\subsection{Correlation assumption (bivariate case)}

\begin{wrapfigure}{r}{0.3\columnwidth}
  \vspace{-0.951cm}
  \centering
  \includegraphics[width=0.3\columnwidth]{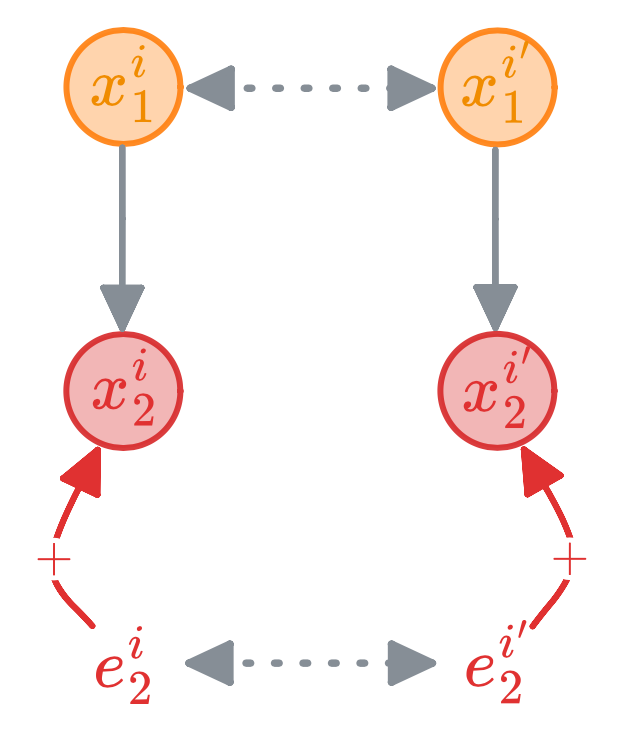}
  \caption{
  Example of two views ($i$ and $i'$) in the bivariate case where $\vx_1$ causes $\vx_2$.
  Assumption~\ref{assumption:simple_diversity_condition_bivariate} holds when correlations associated with grey vertical arrows differ, and correlations associated with grey horizontal arrows differ as well.
  }
  \label{fig:visualization_assumption_1}
  \vspace{-0.6cm}
\end{wrapfigure}

Now, we consider the case of two variables, which is the foundation of the recursive residuals procedure.
For simplicity of notation, we can consider these two random variables aggregated across views: $\vx_1=(x_1^1, \dots, x_1^m)^\top$ and $\vx_2=(x_2^1, \dots, x_2^m)^\top$, with 
either $\vx_1$ causing $\vx_2$, or the converse.
Covariance matrices are denoted by $\mSigma_{\vx_1} = \E[\vx_1 \vx_1^\top]$, 
and likewise for $\vx_2$, $\ve_1$ and~$\ve_2$. 

We introduce a fundamental assumption
that makes the causal direction recoverable.
Essentially, the views need to be correlated and correlations need to be diverse. 

\begin{assumption}[Correlation and diversity across views, bivariate case]
    \label{assumption:simple_diversity_condition_bivariate}
    Consider without loss of generality that $\vx_1 \rightarrow \vx_2$ (in the sense that $\mB_{21}^i\neq 0$ for at least one $i$).
    There exist two distinct views $i \ne i'$ such that a) $\vx_1$ and $\vx_2$ are correlated in at least one of those views, 
    $\mathrm{corr}(x_1^i, x_{2}^i) \ne 0
    \text{ or }
    \mathrm{corr}(x_1^{i'}, x_{2}^{i'}) \ne 0$,
    and b) their correlations fulfill the “diversity” condition 
    $|\mathrm{corr}(x_1^i, x_1^{i'})|
    \ne
    |\mathrm{corr}(e_{2}^i, e_{2}^{i'})|$.
    %
\end{assumption}
%



Figure~\ref{fig:visualization_assumption_1} illustrates this bivariate case.
Discussion on the intuitive meaning of such an assumption is deferred to Section~\ref{sec:identifiability_sos}.
In practice, this assumption becomes increasingly mild as the number of views grows.
More general conditions are discussed in Appendix~\ref{app:ssec:conditions_for_strict_positivity_of_likelihood_based_criterion}.

\subsection{Likelihood-based criterion: PairwiseLiMVAM}
\label{ssec:likelihood_based_criterion}

We first derive a new criterion that extends the PairwiseLiNGAM criterion of \citet{Hyva13JMLR} to the multi-view setting. 
Assuming Gaussian disturbances, we compute the expected log-likelihoods of two models: $\vx_1 \rightarrow \vx_2$ and $\vx_2 \rightarrow \vx_1$, where an arrow denotes a directed edge between two adjacent variables in the causal graph.
This leads to the following likelihood ratio (LR) criterion: 
\begin{proposition}[LR criterion and its consistency]
\label{proposition:consistency_lr_criterion}
The log-likelihood ratio of the two models is given by:
%
\begin{equation}
\label{eq:criterion}
    \mathrm{LR}_{12}
    = 
    - \log \det \mSigma_{\vx_1}
    - \log \det \mSigma_{\ve_2}
    + \log \det \mSigma_{\vx_2}
    + \log \det \mSigma_{\ve_1} \enspace.
\end{equation}
Under Assumption~\ref{assumption:simple_diversity_condition_bivariate}, this criterion is positive when $\vx_1 \rightarrow \vx_2$, and negative in the opposite direction.
\end{proposition}
We prove this in Appendix~\ref{app:sec:theory_bivariate_model}. Note that this consistency result applies \emph{no matter whether the disturbances are Gaussian or non-Gaussian}. Importantly, this makes the causal direction 
identifiable under very general circumstances.
In practice, we replace the covariance matrices $\mSigma_{\vx_1}$, $\mSigma_{\ve_1}$, $\mSigma_{\vx_2}$, and $\mSigma_{\ve_2}$ with consistent estimators (see Appendix~\ref{app:ssec:implementation_of_likelihood_based_criterion}). 



\subsection{Cross-covariance-based criterion: DirectLiMVAM}
\label{ssec:covariance_based_criterion}

Our next criterion is a new multi-view extension of the DirectLiNGAM criterion~\citep{shimizu2011directlingam}. 
In a single-view setting, \citet{shimizu2011directlingam} relied on the fact that $x_1$ and $e_2$ are \textit{independent} if the causal direction is $x_1 \rightarrow x_2$, 
and measured the independence 
using a kernel-based estimator of mutual information called KGV~\citep{bach2002kernel}. A multi-view extension was actually proposed by~\citet{shimizu2012joint}, where the correct direction was chosen by comparing the \emph{sums over views} of these scores for the two directions.
However, such 
summation over views does not fully  take advantage of the multi-view structure. Thus, the disturbances still needed to be non-Gaussian. 

Here, we efficiently exploit the multi-view framework while only looking at SOS. 
We choose to evaluate the \emph{cross-covariance} between the \emph{entire} view-aggregated vectors $\vx_1$ 
and $\ve_2$ 
for the direction $\vx_1 \rightarrow \vx_2$, and between $\vx_2$ and $\ve_1$ for the direction $\vx_2 \rightarrow \vx_1$. 
The correct direction is chosen by comparing the Frobenius norms of the cross-covariances between regressors and residuals.
%
Thus, we obtain the following “FC” criterion that is consistent: 
\begin{proposition}[FC criterion and its consistency]
\label{proposition:consistency_fc_criterion}
Consider the following criterion:
\begin{equation}
\label{eq:criterion_covariance}
    \mathrm{FC}_{12}
    =
    \left\Vert \E \big[\vx_2 \ve_1^\top \big] \right\Vert_F
    -
    \left\Vert \E \big[ \vx_1 \ve_2^\top \big] \right\Vert_F 
    \enspace.
\end{equation}
Under Assumption~\ref{assumption:simple_diversity_condition_bivariate}, this criterion is positive when $\vx_1 \rightarrow \vx_2$, and negative in the opposite direction.
\end{proposition}
We prove this in Appendix~\ref{app:sec:cross_covariance_based_criterion}.






\subsection{Estimating the causal coefficients}

Applying the recursive scheme explained in Section~\ref{ssec:general_principle},
and using the criteria from~\eqref{eq:criterion} or~\eqref{eq:criterion_covariance}, we obtain a causal ordering encoded by $\mP$, in the general multivariate case.
Importantly, since the ordering is now known, the causal matrices $\mB^i$ can be uniquely recovered with little effort.
In fact, the model becomes 
\begin{equation}
    \vx'^i = \mT^i \vx'^i + \ve'^i \;, \qquad i \in [\![1, m]\!]
\end{equation}
where $\vx'^i = \mP \vx^i$ are the reordered observations and $\ve'^i = \mP \ve^i$ are the reordered disturbances. By construction, each variable $x'^i_j$ only depends on its predecessors $x'^i_1, \dots, x'^i_{j-1}$.
For each $j \in [\![2, p]\!]$, we estimate the $j$-th row of all matrices $\mT^i$ jointly using one-step Feasible Generalized Least Squares (FGLS)~\cite{zellner1962efficient}, which is more efficient than Ordinary LS. A detailed description of one-step FGLS is provided in Appendix \ref{app:ssec:one_step_FGLS}.
Finally, the adjacency matrices are recovered as $\mB^i = \mP^\top \mT^i \mP$.

The overall procedure is summarized in Algorithm~\ref{alg:pairwise} 
and illustrated in Figure~\ref{fig:didactic} (middle and right panels); it notably requires no hyperparameter tuning.
Its worst-case \textit{computational complexity} is in $O(m^3 \cdot p^3 \cdot n)$, where $n$ is the number of observations, which we show in Appendix~\ref{app:ssec:computational_complexity_pairwise_direct_limvam}. In practice, these algorithms are highly parallelizable as they rely on basic algebraic operations (\textit{e.g.} multiplying pairs of matrices), and fast as they do not require iterative optimization of an objective function. 

\begin{algorithm}[H]
    \caption{PairwiseLiMVAM or DirectLiMVAM}
    \label{alg:pairwise}

    \textbf{Input:} Observations $\mX = [\vx^1, \dots, \vx^m] \in \R^{m \times p \times n}$.
    
    \begin{enumerate}

        
        \item \textit{Estimate the causal ordering} $\mP$.
                
        \begin{enumerate}

            \item Standardize observations over the sample axis.

            \item Perform pairwise regressions between any two variables: $\vx_j$ on $\vx_i$, such that  $i \neq j$. Compute the residuals $\ve_{j}$. 
                         
            \item 
            For all $i < j$, compute a criterion for the causal direction between $\vx_i$ and $\vx_j$ using \eqref{eq:criterion} or \eqref{eq:criterion_covariance} and the regression results $(\ve_{i}, \ve_{j})$, and store this criterion in the entry $M_{ij}$ of a skew-symmetric matrix $\mM \in \R^{p \times p}$. Then, set $M_{ji} = -M_{ij}$. Having $M_{ij}$ positive means  $\vx_i$  causes $\vx_j$, negative means $\vx_j$ causes $\vx_i$.
                
            \item Determine the root variable $k$: it causes all others, so $M_{kj}$ should be positive for all $j$.
            \begin{equation}
                k = \argmin_i \sum_j \min(0, M_{ij})^2 \enspace.
            \end{equation}
        
            \item Remove the  root variable $k$ and replace observations with residuals obtained when regressing on~$k$. Repeat.

            \item Store the estimated ordering in a permutation matrix $\mP$, and reorder the variables in $\mX$ according to this permutation. This yields a new model $\vx'^i = \mT^i \vx'^i + \ve'^i$, where $\vx'^i$ and $\ve'^i$ are reordered versions of $\vx^i$ and $\ve^i$, respectively, and the $\mT^i$ are strictly lower triangular matrices.

        \end{enumerate}

        \item \textit{Estimate the adjacency matrices} $\mB^i$. For each variable $j \in [\![2, p]\!]$, estimate the $j$-th row of all matrices $\mT^i$ with one-step Feasible Generalized Least Squares~\cite{zellner1962efficient}.
        Recover matrices $\mB^i = \mP^\top \mT^i \mP$.
        
    \end{enumerate}

    \textbf{Output:} Adjacency matrices $(\mB^1, \dots, \mB^m) \in \R^{m \times p \times p}$.
\end{algorithm}


\subsection{General identifiability using second-order statistics}
\label{sec:identifiability_sos}



To theoretically justify when our algorithm can find \textit{unique} parameters $\mB^i$ and $\mP$, we next present a general identifiability theory. 
It uses the following generalization of Assumption~\ref{assumption:simple_diversity_condition_bivariate} from two variables to many variables.  

\begin{assumption}[Correlation and diversity across views]
    \label{assumption:simple_diversity_condition}
    For any two variables $\vx_j$ and $\vx_{j'}$ 
    such that $\vx_j \rightarrow \vx_{j'}$ (in the sense that $\mB_{j'j}^i\neq 0$ for at least one $i$), 
    there exist two distinct views $i \ne i'$ such that a) those variables 
    are correlated in at least one of those views,
    %
     $   \mathrm{corr}(x_j^i, x_{j'}^i) \ne 0
        \text{ or } 
        \mathrm{corr}(x_j^{i'}, x_{j'}^{i'}) \ne 0,
   $ 
    and b) their correlations fulfill the “diversity” condition:
    \begin{equation}
    \label{eq:simple_condition}
        |\mathrm{corr}(e_j^i, e_j^{i'})|
        \ne
        |\mathrm{corr}(e_{j'}^i, e_{j'}^{i'})| \enspace.
    \end{equation}
\end{assumption}

%
Even milder sufficient conditions are discussed in Appendix~\ref{app:ssec:multivariate_conditions_strict_positivity} but they are harder to interpret. 
To better understand this assumption, note first that it requires some \emph{correlation across views}: it rules out the trivial case where all cross-view correlations are zero. Moreover, it demands some \emph{diversity} in these correlations. Consider the extreme situation where, for standardized variables, $x^i_j = \alpha x^{i'}_j$ and $x^i_{j'} = \alpha x^{i'}_{j'}$ (and analogously for the residuals), for some scalar $\alpha$. In this case, all correlations equal $1$, so~\eqref{eq:simple_condition} is violated. Appendix~\ref{app:sec:condition_intepretation} provides more realistic examples of violations, while Appendix~\ref{app:sssec:assumption_1_2} introduces a simple score to assess whether Assumption~\ref{assumption:simple_diversity_condition} holds in practice.

We next introduce an additional assumption that ensures identifiability of the total causal ordering.


\begin{assumption}[Dense connectivity of the graph when pooled across views]
\label{assumption:dense_connectivity_sos}
Let $\cG$ denote the union graph of the $m$ views, obtained by taking the union of their edge sets. Assume that the ordering induced by $\cG$ is \emph{total}, i.e. there exists a directed path between any two variables.  
\end{assumption}

With these assumptions, we prove in Appendix~\ref{app:ssec:proof_of_thm1} the following identifiability result.
%
\begin{theorem}[Identifiability of the ordering and adjacency matrices $\mB^i$]
\label{thm:identifiability_causal_matrices_pairwise}
    Assume the data follows the model in~\eqref{eq:multivariate_model}, and that Assumption 
    \ref{assumption:simple_diversity_condition} holds.
    Then, the causal (partial) ordering and causal coefficients $\mB^i$ are identifiable. If we further make Assumption~\ref{assumption:dense_connectivity_sos}, then $\mP$ is identifiable as well.
\end{theorem}

The identifiability of the \textit{partial} ordering does not necessarily mean that the permutation matrix $\mP$ can be recovered uniquely. If the ordering is only partial (for example, in a three-nodes graph: $\vx_1 \rightarrow \vx_2$ and $\vx_1 \rightarrow \vx_3$, but there is no relation between $\vx_2$ and $\vx_3$, so both orderings $1 \rightarrow 2 \rightarrow 3$ and $1 \rightarrow 3 \rightarrow 2$ are valid), then $\mP$ is only identifiable up to a class of permutations that induce the same partial ordering. The above theorem shows a case where $\mP$ can be (uniquely) identified.

\section{Estimation by multi-view ICA assuming shared disturbances}
\label{sec:estimating_model_using_ICA}

Next, we provide an alternative approach to the LiMVAM model definition and estimation. 
 We utilize possible non-Gaussianity of the disturbances to obtain identifiability conditions that are different from the previous section. However, we still do not impose non-Gaussianity as a necessary condition.

\paragraph{Model definition} The approach here is to propose a particular model for the disturbances by supposing that they can be additively composed into two terms: one term is view-specific, while the other is shared across views and therefore makes the views dependent. Specifically: 
\begin{assumption}[Shared disturbances] \label{assump:disturbance_decomposition}
The disturbances can be decomposed as
\begin{align}
    \label{eq:disturbance_decomposition}
    \ve^i = \mD^i \vs + \vn^i
\end{align}
where the 
 $\mD^i$ are diagonal matrices with positive entries on the diagonal, the $\vs$ has mutually independent entries and second-order moment $\E[\vs \vs^\top] = \mI_p$, and the view-specific noises are Gaussian $\vn^i \sim \cN(\vzero, \mSigma^i)$ and depend on the view $i$ via the second-order moments $\mSigma^i$ that are diagonal matrices. Lastly, the vectors $\vs$ and $\vn^i,i=1,\ldots,m$, are assumed to be mutually independent. 
\end{assumption}

This assumption and \eqref{eq:disturbance_decomposition} are related to~\citet{chen2024individual} who were originally inspired by the multi-view 
Shared ICA of~\citet{richard2021sharedica}. Specifically, 
\citet{chen2024individual} considered the special case  where the matrices $\mD^i = \mI_p$, so that there are the same disturbances~$\vs$ over views. 
We argue here that their model was too restrictive and unrealistic as it imposes an arbitrary scaling for the data variables. Consider the root variable of the DAG, denoting it by $x_i$. Its disturbance is $s_i+n_i$ which has variance of at least one since the variance of $s_i$ is defined as one. Thus, the root variable has variance of at least one which is absurd: the model should be invariant to the scaling of the variables. Likewise, suppose we
  create a new data set by dividing just one view, say $\vx^1$ by, say $10$, giving new data $\vx'^1$. 
This new data does not follow the model by \citet{chen2024individual} anymore since $\vs$ should be rescaled to $\vs/10$ for just this view; now the $\vs/10$ here is different from the $\vs$ in other views. Such problems motivated us to develop a model where the scaling terms $\mD^i$ are new parameters, and thus to the definition in Eq.~\ref{eq:disturbance_decomposition}.


\paragraph{Identifiability}
Now we proceed to show the identifiability of the model based on Assumption~\ref{assump:disturbance_decomposition}. 
The starting point is that our model can be rewritten as a multi-view ICA model as in \eqref{eq:multi_view_ica_model}, just like in the case of LiNGAM.
Thus, the methods developed for multi-view ICA can be used to 
analyze its identifiability.
In the case of Gaussian disturbances, we use the following assumption.
%
\begin{assumption}[Diversity of the views with Gaussian disturbances]
    \label{assump:noise_diversity}
    If $j$ and $j'$, $j \neq j'$, are the indices of two Gaussian common disturbances in $\vs$,     then the number of views is at least three, and there exists a view $i$ such that $\frac{\mSigma^i_{jj}}{(\mD^i_{jj})^2} \neq \frac{\mSigma^i_{j'j'}}{(\mD^i_{j'j'})^2}$.
\end{assumption}
This assumption, inspired by~\citet{richard2021sharedica}, is trivially fulfilled if all disturbances are non-Gaussian. For Gaussian disturbances, it essentially states that the views should be diverse in terms of disturbances' SOS. This follows from the theory of Shared ICA which does not require the common disturbances $\vs$ to be non-Gaussian if there is enough diversity~\citep{richard2021sharedica,anderson2013multiviewidentifiability}.

We can now state our main identifiability result for this variant of the model. 
\begin{theorem}[Identifiability with shared disturbances]
\label{theorem:identifiability_B_tilde_i}
    Assume the LiMVAM model in~\eqref{eq:multivariate_model} together with Assumption~\ref{assump:disturbance_decomposition}. 
    Under Assumption~\ref{assump:noise_diversity},
    the model is identifiable in the sense that the parameters $\mB^i$, $\mD^i$, and $\mSigma^i$ are identifiable. If we further make Assumption~\ref{assumption:dense_connectivity_sos}, 
    $\mP$ is identifiable as well. 
\end{theorem}
%
%
%
The proof of the theorem is given in Appendix~\ref{app:sec:identifiability_limvam_shared_disturbances} {and we provide an extension when the causal ordering $\mP^i$ is view-dependent in Appendix~\ref{app:ssec:identifiability_multiview_different_Pi}}.
We thus see that for non-Gaussian disturbances, identifiability is achieved with weak conditions; in fact, no special conditions on the covariances are needed. Still, the theory of Shared ICA leads to identifiability even for Gaussian disturbances, but with stronger conditions than obtained in Section~\ref{sec:estimating_model_using_pairwise_comparisons}. Moreover, the Shared ICA framework naturally provides an algorithm. 



\paragraph{Estimation} The model 
can be estimated by a combination of the ICA-based estimation procedure by~\citet{shimizu2006linear}, combined with the Shared ICA methods by \citet{richard2021sharedica}.  We call it ICA-LiMVAM and explain it in detail in Algorithm~\ref{alg:ica_limvam}. Surprisingly, the resulting algorithm gives estimates of $\mB_i$ that are identical to \citet{chen2024individual}, even though our model is more general. In contrast, our algorithm estimates different noise variances and scalings $\mD_i$. 
The worst-case \textit{computational complexity} is in $O(tnmp^3 + mp^5)$, for $t$ iterations of the Shared ICA (ML) algorithm, $n$ samples, $m$ views and $p$ components. We detail this in Appendix~\ref{app:ssec:computational_complexity_ica_limvam}.

\section{Experiments}
\label{sec:experiments}

We next apply our algorithms to synthetic data and real-world neuroimaging data.
All the code is written in Python and is available at \url{https://github.com/AmbroiseHeurtebise/LiMVAM}.


\subsection{Simulations}
\label{ssec:simulations}



We benchmark a total of six multi-view causal discovery algorithms on synthetic data. Three of these algorithms are based on \pairwisecomparisonsnospace: our PairwiseLiMVAM and DirectLiMVAM, as well as MultiGroupDirectLiNGAM from~\citet{shimizu2012joint}. The other three algorithms are ICA-based: ICA-LiMVAM which is essentially the same as~\citet{chen2024individual}, as well as ICA-LiNGAM~\citep{shimizu2006linear} which is naively applied to each view, separately. Note that ICA-LiMVAM comes in two variants, each using a different ICA algorithms: either Shared ICA “ML”, a likelihood-based method that can handle both Gaussian and non-Gaussian disturbances, or Shared ICA “J”, which jointly diagonalizes covariance matrices without using non-Gaussianity. 
We do not include comparisons with multi-domain causal discovery methods, as these are not directly comparable to multi-view approaches in settings with cross-view correlations (see Section~\ref{sec:discussion} for a detailed discussion). Nevertheless, for completeness, we provide an additional comparison between DirectLiMVAM and a multi-domain method~\citep{perry2022causal} in Appendix~\ref{app:sssec:comparison_multi_domain}, where we see that our DirectLiMVAM achieves better performance.

The first synthetic experiment is inspired by~\citet{richard2021sharedica}. We monitor the performance of causal discovery algorithms across varying sample sizes and disturbance distributions. 
The data are generated according to the LiMVAM model with shared disturbances from~\eqref{eq:multivariate_model} and~\eqref{eq:disturbance_decomposition}. The performance of the causal discovery algorithms is measured by the $\ell_2$ distance between true and estimated adjacency matrices $\mB^i$, and the percentage of incorrectly estimated causal orders $\mP$ over $50$ random runs.








\begin{figure*}[t]
    \centerline{\includegraphics[width=0.95\linewidth]{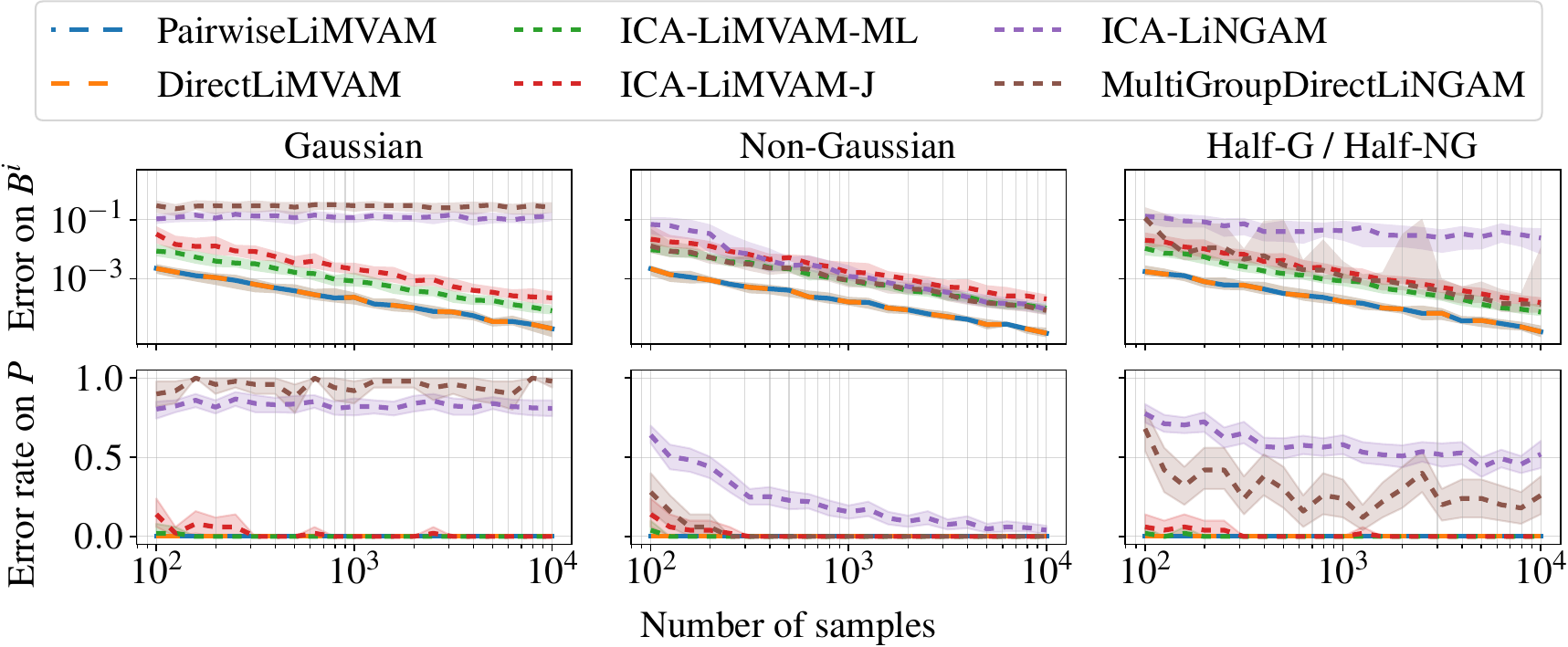}}
    \caption{
    Separation performance of three \pairwisecomparisons algorithms and three ICA-based algorithms. We assume shared disturbances as in~\eqref{eq:disturbance_decomposition}. We used 50 different seeds. Lower is better.
    }
    \label{fig:simulation_study}
    \vspace{-0.2cm}
\end{figure*}

Results of the first synthetic experiment are plotted in Figure~\ref{fig:simulation_study}. Our algorithms PairwiseLiMVAM and DirectLiMVAM consistently outperform all others by a significant margin. 
ICA-LiMVAM yields reliable estimates 
but with larger errors.
In contrast, MultiGroupDirectLiNGAM and ICA-LiNGAM, which are designed for non-Gaussian disturbances, fail entirely in the Gaussian setting, as expected.
Note that the high error rate of ICA-LiNGAM in recovering the causal ordering $\mP$ is likely explained by its inability to exploit that the ordering is shared across views. 

For the second synthetic experiment, we report the trade-off between computational complexity and estimation error of the adjacency matrices. We use a different data generation model, described by~\eqref{eq:multivariate_model} where the disturbances are Gaussian with equal variances, and without~\eqref{eq:disturbance_decomposition}. This breaks the assumptions for all methods but our PairwiseLiMVAM and DirectLiMVAM. As expected, 
we see in Figure~\ref{fig:execution_time} that our algorithms PairwiseLiMVAM and DirectLiMVAM provide a major reduction in estimation errors, at  little or no computational cost,
while the four other algorithms struggle. 


\begin{wrapfigure}{r}{0.5\columnwidth}
  \centering
  \includegraphics[width=0.5\columnwidth]{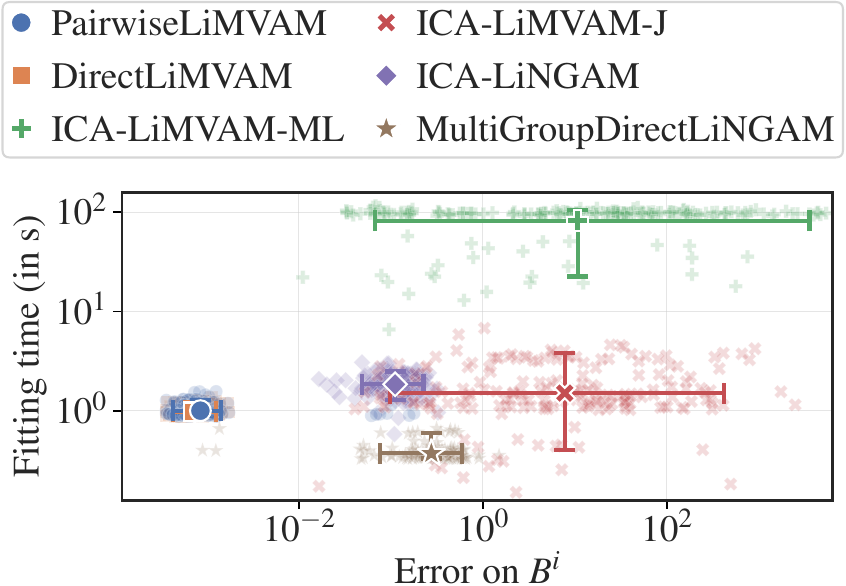}
  \caption{Estimation error of the adjacency matrices $\mB^i$ vs. total fitting time for each method, across 200 random runs with $m = 6$ views, $p = 5$ variables, and $n = 1000$ samples. Disturbances are drawn from a standard Gaussian. Note that the point clouds for PairwiseLiMVAM and DirectLiMVAM largely overlap.}
  \label{fig:execution_time}
  \vspace{-0.5cm}
\end{wrapfigure}

In Appendix~\ref{app:ssec:synthetic_experiments}, we include further experiments. 
They evaluate scalability to larger numbers of views and variables, as well as stress tests assessing the robustness to violations of the assumptions. 
In particular,  we find that DirectLiMVAM outperforms PairwiseLiMVAM in some cases, which is less obvious from Figures~\ref{fig:simulation_study} and~\ref{fig:execution_time} alone.



\subsection{Real data experiments with MEG}
\label{ssec:real_data_experiments}


Next, we performed experiments on magnetoencephalography (MEG) data measuring human brain activity.
MEG data has a high temporal resolution, so typically methods related to Granger causality are used. However, we are here interested in analyzing the \emph{energies} 
of \emph{oscillatory} signals. 
The energies change very slowly (on the time-scale of seconds), while the underlying brain activity being measured is likely to change much more quickly. Thus, it is appropriate to use a model of instantaneous causality like ours. 

We used the Cam-CAN dataset, which is the largest publicly available MEG dataset~\citep{shafto2014camcan,taylor2017camcan}, and considered the “sensorimotor task” during which each participant had to respond with a right index finger button press to auditory/visual stimuli. 
{This setup fits our multi-view framework: since the participants are exposed to the same stimulation, their brain activities are correlated. This is in contrast with resting-state data, where there is no common stimulation and therefore the participants' brain activities are independent; the multi-domain framework would be more suitable  for such data.}

%
%
%
Considering each of the 98 participants to be a view, we applied our PairwiseLiMVAM. 
The experiment ran in 2h 20min using 5 CPUs.
One may reasonably hypothesize that the participants' brains share structural patterns, 
so to facilitate population-level interpretation, we computed the element-wise median of the $\mB^i$ matrices.

\begin{figure*}[htbp]
    \centering
    \begin{subfigure}[b]{0.5\textwidth}
        \centering
        \includegraphics[width=0.82\textwidth]{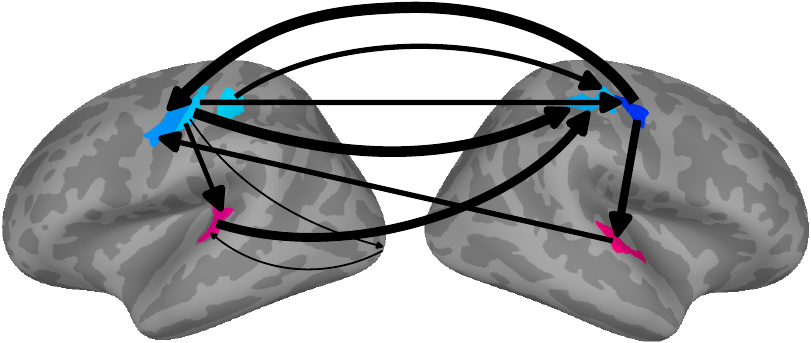}
        \caption{}
        \label{fig:brain_pairwise}
    \end{subfigure}
    \hfill
    \begin{subfigure}[b]{0.2\textwidth}
        \centering
        \includegraphics[width=\textwidth]{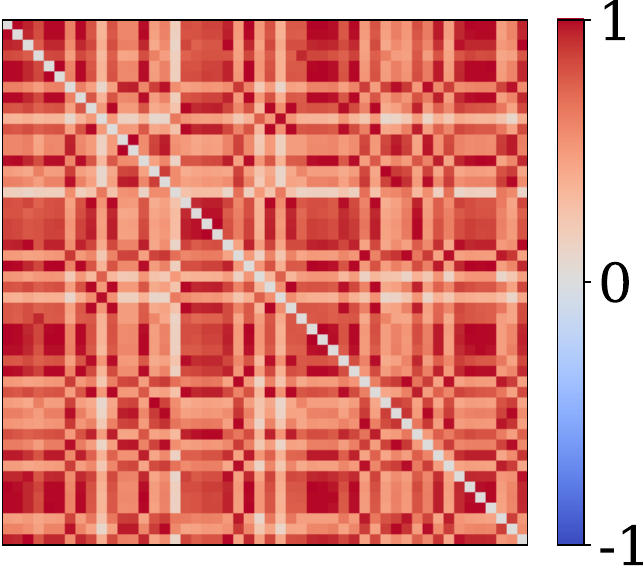}
        \caption{}
        \label{fig:pearson_coefs_B}
    \end{subfigure}
    \hfill
    \begin{subfigure}[b]{0.2\textwidth}
        \centering
        \includegraphics[width=0.67\textwidth]{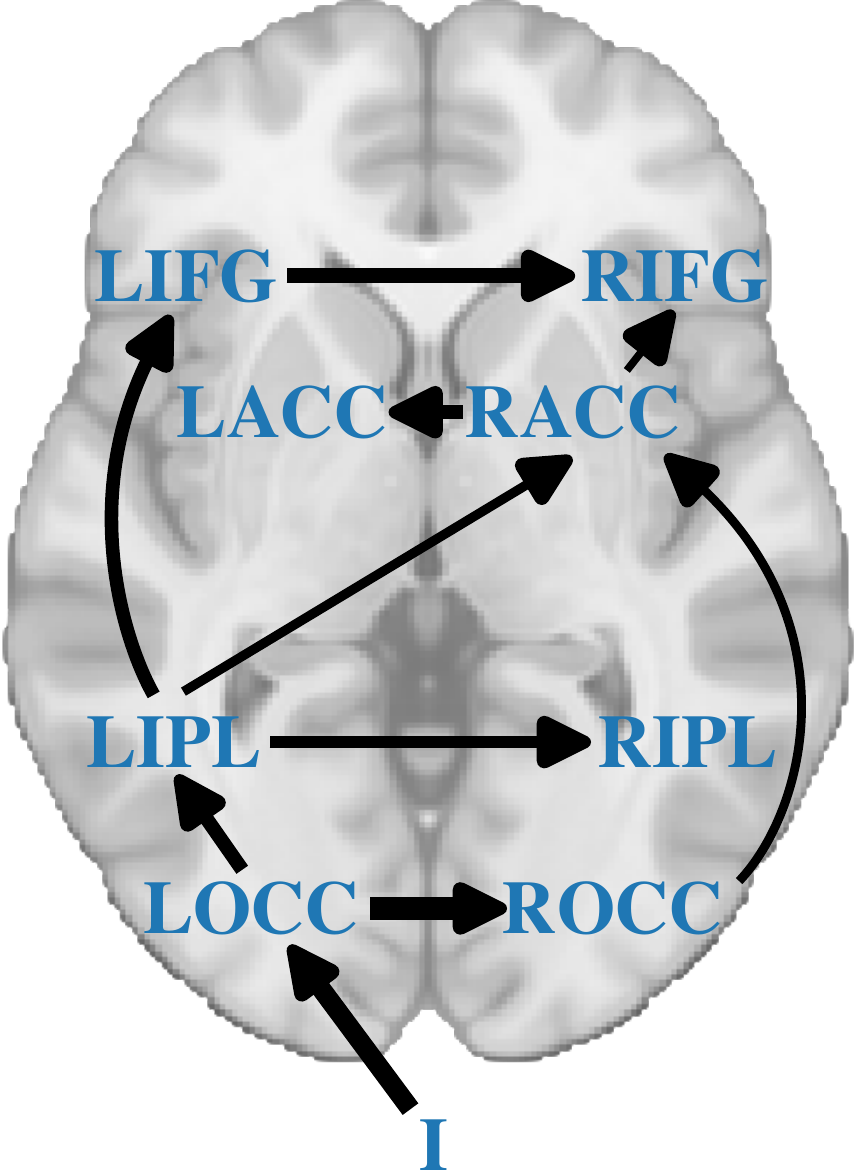}
        \caption{}
        \label{fig:fMRI_drawing}
    \end{subfigure}
    \caption{(a) Top ten strongest median causal effects (estimated by PairwiseLiMVAM) across 98 Cam-CAN (MEG) participants. 
    (b) Pearson correlations between median causal  matrices across 50 runs in MEG. 
    (c) Top ten strongest median causal effects from PairwiseLiMVAM across 9 fMRI participants. Arrows indicate direction and strength of effects, with width proportional to magnitude.}
    \label{fig:real_data_experiment}
\end{figure*}

Fig.~\ref{fig:brain_pairwise} displays the ten strongest (in absolute value) median causal effects. 
Notably, many causal connections are between homologous regions across hemispheres --- particularly within the primary motor and somatosensory cortices --- consistent with prior findings that symmetric brain regions exhibit strong inter-hemispheric correlations \cite{Li651}. Additionally, numerous arrows involve the postcentral regions, 
which have been implicated in fine motor control and hand-related tasks \cite{braun2002functional}.
To evaluate the robustness of the results, we repeated the experiment 50 times, each time randomly selecting 30 participants from the full cohort.
We assessed consistency by measuring the Pearson correlation between the resulting median matrices across runs.
As shown in Figure~\ref{fig:pearson_coefs_B}, the median effects are highly correlated across different subsets (average correlation: 0.67), demonstrating the stability of our method.
Additional experiments 
using ICA-LiMVAM-ML are reported in Appendix~\ref{app:sssec:MEG_using_ICSL}.

\subsection{Real data experiments with fMRI}
\label{sssec:fMRI_experiment}

We also evaluated PairwiseLiMVAM on an fMRI rhyming judgment task dataset~\citep{ramsey2010six}. 
The data consist of recordings from nine participants, each represented by nine variables: one task regressor (Input I), obtained by convolving the boxcar design of the rhyming task with a canonical hemodynamic response function, and eight regional time series from bilateral regions of interest. 
The dataset was acquired on a 3T scanner with a repetition time of 2 s, yielding 160 samples per subject, and is available via the OpenfMRI project 
(preprocessed version from~\url{https://github.com/cabal-cmu/Feedback-Discovery}). 

We report the ten strongest effects from the element-wise median of the estimated adjacency matrices in Figure~\ref{fig:fMRI_drawing}.
The recovered structure is consistent with established findings. 
In particular, we find a strong edge from the task regressor to the left occipital cortex (Input $\to$ LOCC), reflecting the expected visual drive of the stimuli, and prominent connections from left to right homologous regions (\textit{e.g.} LOCC $\to$ ROCC, LIFG $\to$ RIFG, LIPL $\to$ RIPL), in line with the left‐hemisphere dominance and inter‐hemispheric influences reported by~\citet{ramsey2010six}. 
Edges such as LOCC $\to$ LIPL and LIPL $\to$ LIFG, paths such as LOCC $\to$ LACC, and connections between LACC and RACC also match relationships highlighted in their analyses. 

\section{Discussion and related work}
\label{sec:discussion}

\paragraph{Multi-view vs.\ multi-domain distinction}
We wish to emphasize that the multi-view setting considered here is very different from \textit{multi-domain} methods. As an intuitive example, multi-view would correspond to a person seeing the same object from different angles, while multi-domain would correspond to a person seeing different objects. In multi-view data, the data points in different views are \textit{paired} over domains, while in multi-domain methods, no such correspondence between data points exists. In statistical terms, the fundamental difference is that in the multi-view setting, the views are statistically dependent (correlated), while in the multi-domain setting, the domains are independently generated, typically with a distribution shift over domains. In fact, several methods study causal discovery in the multi-domain case \citep{ghassami2018multi,adams2021identification,mooij2020joint};  \textit{multi-environment} methods are a special case of multi-domain methods, where distributional shifts arise from interventions \citep{peters2016causal,perry2022causal}; these are discussed in detail in Appendix~\ref{app:sec:distinction_multi_view_multi_domain}. 





\paragraph{Related work} A multi-view causal model based on LiNGAM was considered by \citet{shimizu2012joint}, but their model was very different. 
The fundamental difference is that we assume the disturbances have some shared information across views and can be Gaussian, whereas \citet{shimizu2012joint} estimate necessarily non-Gaussian disturbances, while ignoring dependencies across views.
Thus, they only showed how to use the multiple views to improve the estimation while still assuming non-Gaussianity.

\citet{chen2024individual} considered multi-view SEM using a special case of our shared disturbances model, which is in its turn a special case of our general LiMVAM model. 
Algorithmically, there is little difference between their work and our ICA-LiMVAM. However, their identifiability proofs are not available nor did they consider the identifiability of the causal ordering, while we do both.
Moreover, as we already argued in Section~\ref{sec:estimating_model_using_ICA}, their model formulation has the serious flaw that it forces an arbitrary scaling of the data variables. 
Crucially, we propose new algorithms, PairwiseLiMVAM and DirectLiMVAM, that use only SOS; they are shown to be empirically superior to the ICA-LiMVAM in most cases. 
{Finally,  \citet{zhu2025causalmeg} studied a special case of our model that is bivariate and where the views correspond to time lags, 
from a purely empirical viewpoint. 


\section{Conclusion}
This work provides a full treatment of multi-view causal discovery that was under-explored in the literature. We start by defining a multi-view model called LiMVAM. We derive a comprehensive identifiability theory of the causal ordering and weights, using second-order statistics (SOS, \textit{i.e.}\ covariance structure) only. Beyond these theoretical guarantees, we also improve an existing algorithm (ICA-LiMVAM) and introduce two novel ones --- PairwiseLiMVAM and DirectLiMVAM --- for estimating the parameters of the multi-view model; again, SOS alone are enough. We rigorously prove the consistency of our algorithms and empirically show that they are both \textit{fast} and result in \textit{highly efficient estimation}. On synthetic data, they outperform comparable algorithms; results on brain imaging are promising. 




\vskip 0.2in
\bibliography{references}

@InProceedings{pmlr-v235-xu24ac,
  title = 	 {A Sparsity Principle for Partially Observable Causal Representation Learning},
  author =       {Xu, Danru and Yao, Dingling and Lachapelle, Sebastien and Taslakian, Perouz and Von K\"{u}gelgen, Julius and Locatello, Francesco and Magliacane, Sara},
  booktitle = 	 {Proceedings of the 41st International Conference on Machine Learning},
  pages = 	 {55389--55433},
  year = 	 {2024},
  volume = 	 {235},
  publisher =    {PMLR}
}

@InProceedings{Morioka24,
author={H.~Morioka and A.\ Hyv\"arinen},
title={Causal Representation Learning Made Identifiable by Grouping of Observational Variables},
booktitle =    {Proc.\ Int.\ Conf.\ on Machine Learning (ICML2024)},
address={Vienna, Austria},
year={2024}
}

@article{JMLR:v26:24-0194,
  author  = {Burak Varici and Emre Acart{{\"u}}rk and Karthikeyan Shanmugam and Abhishek Kumar and Ali Tajer},
  title   = {Score-based Causal Representation Learning: Linear and General Transformations},
  journal = {Journal of Machine Learning Research},
  year    = {2025},
  volume  = {26},
  number  = {112},
  pages   = {1--90}
}

@inproceedings{richard2021sharedica,
	title        = {Shared Independent Component Analysis for Multi-Subject Neuroimaging},
	author       = {Richard, Hugo and Ablin, Pierre and Thirion, Bertrand and Gramfort, Alexandre and Hyvarinen, Aapo},
	year         = 2021,
	booktitle    = {Advances in Neural Information Processing Systems},
	publisher    = {Curran Associates, Inc.},
	volume       = 34,
	pages        = {29962--29971}
}

@book{bollen1989structural,
	title        = {Structural equations with latent variables},
	author       = {Bollen, Kenneth A},
	year         = 1989,
	publisher    = {John Wiley \& Sons},
	volume       = 210
}

@article{zheng2018dags,
	title        = {{DAGs with NO TEARS: Continuous optimization for structure learning}},
	author       = {Zheng, Xun and Aragam, Bryon and Ravikumar, Pradeep K and Xing, Eric P},
	year         = 2018,
	journal      = {Advances in neural information processing systems},
	volume       = 31
}

@article{shimizu2012joint,
	title        = {{Joint estimation of linear non-Gaussian acyclic models}},
	author       = {Shimizu, Shohei},
	year         = 2012,
	journal      = {Neurocomputing},
	publisher    = {Elsevier},
	volume       = 81,
	pages        = {104--107}
}

@article{shimizu2011directlingam,
	title        = {{DirectLiNGAM: A direct method for learning a linear non-Gaussian structural equation model}},
	author       = {Shimizu, Shohei and Inazumi, Takanori and Sogawa, Yasuhiro and Hyvarinen, Aapo and Kawahara, Yoshinobu and Washio, Takashi and Hoyer, Patrik O and Bollen, Kenneth and Hoyer, Patrik},
	year         = 2011,
	journal      = {Journal of Machine Learning Research-JMLR},
	volume       = 12,
	number       = {Apr},
	pages        = {1225--1248}
}

@article{shimizu2006linear,
	title        = {{A linear non-Gaussian acyclic model for causal discovery.}},
	author       = {Shimizu, Shohei and Hoyer, Patrik O and Hyv{\"a}rinen, Aapo and Kerminen, Antti and Jordan, Michael},
	year         = 2006,
	journal      = {Journal of Machine Learning Research},
	volume       = 7,
	number       = 10
}

@book{shimizu2022statistical,
	title        = {Statistical Causal Discovery: LiNGAM Approach},
	author       = {Shimizu, Sh{\=o}hei},
	year         = 2022,
	publisher    = {Springer}
}

@article{anderson2013multiviewidentifiability,
	title        = {Independent Vector Analysis: Identification Conditions and Performance Bounds},
	author       = {Matthew Anderson and Gengshen Fu and Ronald Phlypo and Tulay Adali},
	year         = 2013,
	journal      = {IEEE Transactions on Signal Processing},
	volume       = 62,
	pages        = {4399--4410}
}

@article{comon1994identifiability,
	title        = {Independent component analysis, A new concept?},
	author       = {Pierre Comon},
	year         = 1994,
	journal      = {Signal Processing},
	volume       = 36,
	number       = 3,
	pages        = {287--314},
	issn         = {0165-1684},
	note         = {Higher Order Statistics}
}

@article{shafto2014camcan,
	title        = {{The Cambridge Centre for Ageing and Neuroscience (Cam-CAN) study protocol: a cross-sectional, lifespan, multidisciplinary examination of healthy cognitive ageing}},
	author       = {Meredith A. Shafto and Lorraine K. Tyler and Marie Dixon and Jason R. Taylor and James Benedict Rowe and Rhodri Cusack and Andrew J. Calder and William D. Marslen-Wilson and John S. Duncan and T. Dalgleish and Richard N. A. Henson and Carol Brayne and Fiona E. Matthews},
	year         = 2014,
	journal      = {BMC Neurology},
	volume       = 14
}

@article{taylor2017camcan,
	title        = {The {Cambridge Centre for Ageing and Neuroscience} ({Cam-CAN}) data repository: Structural and functional {MRI, MEG}, and cognitive data from a cross-sectional adult lifespan sample},
	author       = {Jason R. Taylor and Nitin Williams and Rhodri Cusack and Tibor Auer and Meredith A. Shafto and Marie Dixon and Lorraine K. Tyler and  Cam-CAN and Richard N. Henson},
	year         = 2017,
	journal      = {NeuroImage},
	volume       = 144,
	pages        = {262--269},
	issn         = {1053-8119},
	note         = {Data Sharing Part II}
}

@article{gramfort2013mnepython,
	title        = {{MEG and EEG data analysis with MNE-Python}},
	author       = {Gramfort, Alexandre  and Luessi, Martin  and Larson, Eric  and Engemann, Denis A. and Strohmeier, Daniel  and Brodbeck, Christian  and Goj, Roman  and Jas, Mainak  and Brooks, Teon  and Parkkonen, Lauri  and Hämäläinen, Matti},
	year         = 2013,
	journal      = {Frontiers in Neuroscience},
	volume       = 7,
	issn         = {1662-453X}
}

@article{gramfort2014mnepython,
	title        = {{MNE software for processing MEG and EEG data}},
	author       = {Alexandre Gramfort and Martin Luessi and Eric Larson and Denis A. Engemann and Daniel Strohmeier and Christian Brodbeck and Lauri Parkkonen and Matti S. Hämäläinen},
	year         = 2014,
	journal      = {NeuroImage},
	volume       = 86,
	pages        = {446--460},
	issn         = {1053-8119}
}

@article{power2023camcandictionarylearning,
	title        = {Using convolutional dictionary learning to detect task-related neuromagnetic transients and ageing trends in a large open-access dataset},
	author       = {Lindsey Power and Cédric Allain and Thomas Moreau and Alexandre Gramfort and Timothy Bardouille},
	year         = 2023,
	journal      = {NeuroImage},
	volume       = 267,
	pages        = 119809,
	issn         = {1053-8119}
}

@inproceedings{genin2021identifiability,
	title        = {Statistical Undecidability in Linear, Non-{Gaussian} Causal Models in the Presence of Latent Confounders},
	author       = {Genin, Konstantin},
	year         = 2021,
	booktitle    = {Advances in Neural Information Processing Systems},
	publisher    = {Curran Associates, Inc.},
	volume       = 34,
	pages        = {13564--13574},
	editor       = {M. Ranzato and A. Beygelzimer and Y. Dauphin and P.S. Liang and J. Wortman Vaughan}
}

@article{richardson2002ancestral,
  title={{Ancestral graph Markov models}},
  author={Richardson, Thomas and Spirtes, Peter},
  journal={The Annals of Statistics},
  volume={30},
  number={4},
  pages={962--1030},
  year={2002},
  publisher={Institute of Mathematical Statistics}
}

@book{peters2017elements,
	title        = {Elements of causal inference: foundations and learning algorithms},
	author       = {Peters, Jonas and Janzing, Dominik and Sch{\"o}lkopf, Bernhard},
	year         = 2017,
	publisher    = {The MIT Press}
}

@book{spirtes2001causation,
	title        = {Causation, prediction, and search},
	author       = {Spirtes, Peter and Glymour, Clark and Scheines, Richard},
	year         = 2001,
	publisher    = {MIT press},
	edition      = {Second edition}
}

@article{peters2014identifiability,
  title={{Identifiability of Gaussian structural equation models with equal error variances}},
  author={Peters, Jonas and B{\"u}hlmann, Peter},
  journal={Biometrika},
  volume={101},
  number={1},
  pages={219--228},
  year={2014},
  publisher={Oxford University Press}
}

@article{hoyer2008nonlinear,
	title        = {Nonlinear causal discovery with additive noise models},
	author       = {Hoyer, Patrik and Janzing, Dominik and Mooij, Joris M and Peters, Jonas and Sch{\"o}lkopf, Bernhard},
	year         = 2008,
	journal      = {Advances in neural information processing systems},
	volume       = 21
}

@article{Hyva13JMLR,
	title        = {Pairwise Likelihood Ratios for Estimation of Non-{G}aussian Structural Equation Models},
	author       = {Aapo Hyv\"arinen and Stephen M. Smith},
	year         = 2013,
	journal      = {Journal of Machine Learning Research},
	volume       = 14,
	pages        = {111--152}
}

@inproceedings{Monti19UAI,
	title        = {Causal Discovery with General Non-Linear Relationships using Non-Linear {ICA}},
	author       = {Ricardo Pio Monti and Kun Zhang and Aapo Hyv\"arinen},
	year         = 2019,
	booktitle    = {Proc.\ 35th Conf.~ on Uncertainty in Artificial Intelligence (UAI2019)},
	address      = {Tel Aviv, Israel},
	optpages     = {282-289}
}

@inproceedings{Zhang09UAI,
	title        = {On the Identifiability of the Post-Nonlinear Causal Model},
	author       = {Kun Zhang and Aapo Hyv\"arinen},
	year         = 2009,
	booktitle    = {Proc.\ 25th Conference on Uncertainty in Artificial Intelligence (UAI2009)},
	address      = {Montr\'eal, Canada},
	pages        = {647--655}
}

@article{mooij2020joint,
	title        = {Joint causal inference from multiple contexts},
	author       = {Mooij, Joris M and Magliacane, Sara and Claassen, Tom},
	year         = 2020,
	journal      = {Journal of machine learning research},
	volume       = 21,
	number       = 99,
	pages        = {1--108}
}

@inproceedings{sturma2023unpaired,
	title        = {Unpaired Multi-Domain Causal Representation Learning},
	author       = {Sturma, Nils and Squires, Chandler and Drton, Mathias and Uhler, Caroline},
	year         = 2023,
	booktitle    = {Advances in Neural Information Processing Systems},
	publisher    = {Curran Associates, Inc.},
	volume       = 36,
	pages        = {34465--34492},
	editor       = {A. Oh and T. Naumann and A. Globerson and K. Saenko and M. Hardt and S. Levine}
}

@book{Hyvabook,
	title        = {Independent Component Analysis},
	author       = {Aapo Hyv\"arinen and Juha Karhunen and Erkki Oja},
	year         = 2001,
	publisher    = {Wiley Interscience}
}

@article{dale1999cortical,
	title        = {Cortical surface-based analysis: I. Segmentation and surface reconstruction},
	author       = {Dale, Anders M and Fischl, Bruce and Sereno, Martin I},
	year         = 1999,
	journal      = {Neuroimage},
	publisher    = {Elsevier},
	volume       = 9,
	number       = 2,
	pages        = {179--194}
}

@article{hamalainen-mne,
	title        = {Interpreting magnetic fields of the brain: minimum norm estimates},
	author       = {H{\"a}m{\"a}l{\"a}inen, M. S. and Ilmoniemi, R. J.},
	year         = 1994,
	journal      = {Medical \& Biological Engineering \& Computing},
	volume       = 32,
	number       = 1,
	pages        = {35--42},
	isbn         = {1741-0444},
	date         = {1994/01/01},
	id           = {H{\"a}m{\"a}l{\"a}inen1994}
}

@article{Taulu_2006,
	title        = {{Spatiotemporal signal space separation method for rejecting nearby interference in MEG measurements}},
	author       = {Taulu, S and Simola, J},
	year         = 2006,
	month        = {mar},
	journal      = {Physics in Medicine \& Biology},
	publisher    = {},
	volume       = 51,
	number       = 7,
	pages        = 1759
}

@article{Li651,
	title        = {Ipsilateral hemisphere activation during motor and sensory tasks.},
	author       = {Li, A and Yetkin, F Z and Cox, R and Haughton, V M},
	year         = 1996,
	journal      = {American Journal of Neuroradiology},
	publisher    = {American Journal of Neuroradiology},
	volume       = 17,
	number       = 4,
	pages        = {651--655},
	issn         = {0195-6108}
}

@article{KHAN201857,
	title        = {Maturation trajectories of cortical resting-state networks depend on the mediating frequency band},
	author       = {Sheraz Khan and Javeria A. Hashmi and Fahimeh Mamashli and Konstantinos Michmizos and Manfred G. Kitzbichler and Hari Bharadwaj and Yousra Bekhti and Santosh Ganesan and Keri-Lee A. Garel and Susan Whitfield-Gabrieli and Randy L. Gollub and Jian Kong and Lucia M. Vaina and Kunjan D. Rana and Steven M. Stufflebeam and Matti S. Hämäläinen and Tal Kenet},
	year         = 2018,
	journal      = {NeuroImage},
	volume       = 174,
	pages        = {57--68},
	issn         = {1053-8119}
}

@article{braun2002functional,
	title        = {Functional organization of primary somatosensory cortex depends on the focus of attention},
	author       = {Braun, Christoph and Haug, Monika and Wiech, Katja and Birbaumer, Niels and Elbert, Thomas and Roberts, Larry E},
	year         = 2002,
	journal      = {Neuroimage},
	publisher    = {Elsevier},
	volume       = 17,
	number       = 3,
	pages        = {1451--1458}
}

@inproceedings{chen2024individual,
	title        = {Individual causal structure learning from population data},
	author       = {Chen, Wei and Huang, Xiaokai and Li, Zijian and Cai, Ruichu and Huang, Zhiyi and Hao, Zhifeng},
	year         = 2024,
	booktitle    = {Proceedings of the Thirty-Third International Joint Conference on Artificial Intelligence},
	pages        = {7109--7117}
}

@article{zellner1962efficient,
	title        = {An efficient method of estimating seemingly unrelated regressions and tests for aggregation bias},
	author       = {Zellner, Arnold},
	year         = 1962,
	journal      = {Journal of the American Statistical Association},
	publisher    = {Taylor \& Francis},
	volume       = 57,
	number       = 298,
	pages        = {348--368},
	doi          = {10.2307/2281644}
}

@article{bach2002kernel,
	title        = {Kernel independent component analysis},
	author       = {Bach, Francis R and Jordan, Michael I},
	year         = 2002,
	journal      = {Journal of machine learning research},
	volume       = 3,
	number       = {Jul},
	pages        = {1--48}
}

@book{pearl2009causalitybook,
	title        = {Causality: Models, Reasoning and Inference},
	author       = {Pearl, Judea},
	year         = 2009,
	publisher    = {Cambridge University Press},
	address      = {USA},
	isbn         = {052189560X},
	edition      = {2nd}
}

@article{ramsey2010six,
  title={{Six problems for causal inference from fMRI}},
  author={Ramsey, Joseph D and Hanson, Stephen Jos{\'e} and Hanson, Catherine and Halchenko, Yaroslav O and Poldrack, Russell A and Glymour, Clark},
  journal={neuroimage},
  volume={49},
  number={2},
  pages={1545--1558},
  year={2010},
  publisher={Elsevier}
}

@book{greene2003econometric,
  title={Econometric analysis},
  author={Greene, William H},
  year={2003},
  publisher={Pearson education india}
}

@article{ghassami2018multi,
  title={Multi-domain causal structure learning in linear systems},
  author={Ghassami, AmirEmad and Kiyavash, Negar and Huang, Biwei and Zhang, Kun},
  journal={Advances in neural information processing systems},
  volume={31},
  year={2018}
}

@article{adams2021identification,
  title={Identification of partially observed linear causal models: Graphical conditions for the non-gaussian and heterogeneous cases},
  author={Adams, Jeffrey and Hansen, Niels and Zhang, Kun},
  journal={Advances in Neural Information Processing Systems},
  volume={34},
  pages={22822--22833},
  year={2021}
}

@article{peters2016causal,
  title={Causal inference by using invariant prediction: identification and confidence intervals},
  author={Peters, Jonas and B{\"u}hlmann, Peter and Meinshausen, Nicolai},
  journal={Journal of the Royal Statistical Society Series B: Statistical Methodology},
  volume={78},
  number={5},
  pages={947--1012},
  year={2016},
  publisher={Oxford University Press}
}

@article{perry2022causal,
  title={Causal discovery in heterogeneous environments under the sparse mechanism shift hypothesis},
  author={Perry, Ronan and Von K{\"u}gelgen, Julius and Sch{\"o}lkopf, Bernhard},
  journal={Advances in Neural Information Processing Systems},
  volume={35},
  pages={10904--10917},
  year={2022}
}

@article{zhu2019causal,
  title={Causal discovery with reinforcement learning},
  author={Zhu, Shengyu and Ng, Ignavier and Chen, Zhitang},
  journal={arXiv preprint arXiv:1906.04477},
  year={2019}
}

@inproceedings{khemakhem2020variational,
  title={{Variational autoencoders and nonlinear ICA: A unifying framework}},
  author={Khemakhem, Ilyes and Kingma, Diederik and Monti, Ricardo and Hyvarinen, Aapo},
  booktitle={International conference on artificial intelligence and statistics},
  pages={2207--2217},
  year={2020},
  organization={PMLR}
}

@article{zhu2025causalmeg,
    author = {Zhu, Yongjie and Parkkonen, Lauri and Hyvärinen, Aapo},
    title = {{Second-order instantaneous causal analysis of spontaneous MEG}},
    journal = {Imaging Neuroscience},
    volume = {3},
    year = {2025},
    month = {04},
    issn = {2837-6056},
}

\newpage

\appendix
\onecolumn

\section*{Appendix}

The appendix is organized as follows. 

In Section~\ref{app:sec:lingam}, we rigorously restate known identifiability results for LiNGAM. 

In Section~\ref{app:sec:theory_bivariate_model}, we prove the consistency of the likelihood-based criterion for a pair of variables. 

In Section~\ref{app:sec:cross_covariance_based_criterion}, we prove the consistency of the cross-covariance-based criterion for a pair of variables. 

In Section~\ref{app:sec:identifiability_limvam}
, we extend the consistency of these criteria for a pair of variables to \textit{many} pairs, and show that the LiMVAM model is identifiable.

In Section~\ref{app:sec:sos_based_algorithms_limvam}
, we show how to implement algorithms used to identify the causal graph.

In Section~\ref{app:sec:identifiability_limvam_shared_disturbances}, we prove the identifiability of a special case of LiMVAM with shared disturbances. 

In Section~\ref{app:sec:ica_based_algorithm_limvam}
, we present an algorithm for estimating the causal graph in this special case. 

In Section~\ref{app:sec:experiments}
, we detail our experiments on real and synthetic data. 

In Section \ref{app:sec:distinction_multi_view_multi_domain}, we discuss the difference between multi-view and multi-domain frameworks.

\renewcommand{\contentsname}{}
\vspace{-1cm}

\addtocontents{toc}{\protect\setcounter{tocdepth}{2}}

\tableofcontents

\newpage
\section{LiNGAM}
\label{app:sec:lingam}

\paragraph{Defining a domain connecting ICA and SEM parameters}
In the following proofs, we define the set
\begin{align}
    \label{eq:domain_unmixing_matrix}
\begin{split}
    \mathcal{W} 
    = 
    \{ 
    &\mW \in \mathbb{R}^{p \times p} \mid \text{there exist a diagonal matrix $\mD$ with positive diagonal entries} \\
    &\text{and a DAG matrix $\mB$, such that $\mW = \mD^{-1} (\mI - \mB)$} \}
\end{split}
\end{align}
as the domain of the “unmixing matrices” $\mW=\mA^{-1}$ in the ICA model from~\eqref{eq:ica_model}, that are compatible with the DAG structure in the structural equation model (SEM) from~\eqref{eq:causal_model_original}.
In general, unmixing matrices are not constrained in the ICA model, except by invertibility (and possibly by normalizing their rows).
However, in the context of SEM estimation, the requirement that $\mW = \mD^{-1} (\mI - \mB)$ belongs to the set $\mathcal{W}$ is a key structural constraint.
Furthermore, it should be noted that finding $\mW$ is equivalent to finding $\mD$ and $\mB$, since they are related by $\mB = \mD^{-1} (\mI-\mW)$.

Finally, we will use the fact that any DAG matrix $\mB$ can be decomposed into $\mB = \mP^\top \mT \mP$, where $\mP$ is a permutation matrix and $\mT$ a strictly lower triangular matrix.

\subsection{Technical lemmas}

The following two lemmas are about matrices in the context of DAGs. They are used in the proofs of Theorems~\ref{theorem:lingam}, \ref{theorem:identifiability_B_tilde_i}, and~\ref{theorem:identifiability_multiview_different_Pi}, and are proven in Appendix~\ref{app:ssec:lemma}.

The basic identifiability results for the causal matrices $\mB^i$ use the following:
\begin{lemma}
    \label{lemma:identifiability_lingam}
    Let $\mW = \mD^{-1} (\mI - \mB)$ and $\mW' = \mD'^{-1} (\mI - \mB')$ be two matrices that belong to $\mathcal{W}$ (defined in \eqref{eq:domain_unmixing_matrix}), and let $\mQ$ be a sign-permutation matrix. 
    Then 
    \begin{align}
        \mW' = \mQ^\top \mW 
        \implies 
        \mQ = \mI \:, \; 
        \mD' = \mD \:, \; \text{and} \; 
        \mB' = \mB \enspace.
    \end{align}
\end{lemma}
On the other hand, the identifiability results for the causal ordering $\mP$ (or causal orderings $\mP^i$) use the following:
 \begin{lemma}
 \label{app:lemma:Bi_fully_connected}
        Let $\mP$ and $\mP'$ be permutation matrices and $\mT$ and $\mT'$ be strictly lower triangular matrices, such that $\mP^\top \mT \mP = \mP'^\top \mT' \mP'$.
        Assume that $\mT$ contains only non-zero elements in its strictly lower triangular part.
        Then, $\mP = \mP'$ and $\mT = \mT'$.
    \end{lemma}

\subsection{Identifiability of LiNGAM}
\label{app:ssec:identifiability_lingam}

As a special case, we restate here the identifiability of the single-view LiNGAM model (in terms of the matrix $\mB$).


\begin{theorem}[Identifiability of LiNGAM]
\label{theorem:lingam}
In the statistical model defined by~\eqref{eq:causal_model_original}, the parameter $\mB$ is identifiable, provided that the entries in $\vs$ are mutually independent, that at most one of them is Gaussian, and that $\mB$ is a DAG.
\end{theorem}

This proof is based on the one from~\citet{shimizu2006linear}, which we attempt to make a bit more rigorous. 
Identifiability was also shown in~\citet[Section $2.3$]{shimizu2022statistical} in the case with $2$ components.

Let us transform the SEM
\begin{align}
    \label{app:eq:causal_model_original}
    \vx = \mB \vx + \vs
\end{align}
into an ICA model
\begin{align}
    \label{app:eq:ica_model}
    \vx = \mA \vs
\end{align}
where the mixing matrix $\mA$ is constrained as $\mA^{-1}=\mI-\mB$, and where $\mB$ is a DAG. The assumptions on $\vs$ are identical in SEM and ICA. 
Let us parameterize by $\mW = \mA^{-1} = \mI - \mB$, instead of $\mB$.

Consider two matrices $\mW = \mI - \mB$ and $\mW' = \mI - \mB'$, with $\mB$ and $\mB'$ being DAG matrices, that parameterize the same statistical model in~\eqref{app:eq:causal_model_original}.
We have $\det(\mW) = \det(\mW') = 1$, so $\mW$ and $\mW'$ are invertible, which makes them valid unmixing matrices for the same ICA model in \eqref{app:eq:ica_model}.
So, from the identifiability theory of ICA~\cite{comon1994identifiability}, we know that there exist a sign-permutation matrix $\mQ$ and a scaling matrix $\mD$ (\textit{i.e.} a diagonal matrix with positive entries on the diagonal) such that 
\begin{align}
    \label{eq:lingam_indeterminacy_prelim}
    \mW' 
    &= 
    \mD \mQ^\top \mW
\end{align}
hence
\begin{align}
    \label{eq:lingam_indeterminacy}
    \mW''
    &= 
    \mQ^\top \mW
\end{align}
where we defined $\mW'' = \mD^{-1} \mW'$.
We thus have that $\mW'' = \mD^{-1} (\mI - \mB')$ and $\mW = \mI - \mB$ both belong to the set $\mathcal{W}$, so applying Lemma~\ref{lemma:identifiability_lingam} to $\mW$ and $\mW''$ imposes $\mQ = \mI$, $\mD = \mI$, and $\mB' = \mB$.
This makes $\mB$ fully identifiable (“fully” identifiable is in contrast to conventional ICA theory where some indeterminacies always remain), which concludes the proof.

\subsection{Proofs of technical lemmas}
\label{app:ssec:lemma}

\subsubsection{Proof of Lemma~\ref{lemma:identifiability_lingam}}

Consider $\mW, \mW' \in \mathcal{W}$ and $\mQ$ a sign-permutation matrix. Suppose that
\begin{align}\label{eq:lingam_indeterminacy_inappendix}
    \mW' = \mQ^\top \mW \enspace.
\end{align}

\paragraph{A permutation inequality}
Using the definition of $\mathcal{W}$ in~\eqref{eq:domain_unmixing_matrix} and the decompositions of DAG matrices, we have
\begin{align}
    \mW 
    = 
    \mD^{-1} \mP^\top (\mI- \mT) \mP
    \;, \quad
    \mW' 
    = 
    \mD'^{-1}\mP'^\top (\mI - \mT') \mP'
\end{align}
where $\mD$, $\mD'$ are diagonal matrices with positive entries on the diagonal, $\mP, \mP'$ are permutation matrices, and $\mT, \mT'$ are strictly lower triangular matrices. Denote $\mL = \mI - \mT$ and $\mL' = \mI - \mT'$; these are lower triangular matrices that have unit diagonals. We plug the decompositions into~\eqref{eq:lingam_indeterminacy_inappendix}, and get
\begin{align}
    \label{app:eq:lemma_6_decomposition}
    \mD'^{-1} \mP'^\top \mL' \mP'
    =
    \mQ^\top \mD^{-1} \mP^\top \mL \mP
    \enspace.
\end{align}
To exploit the structure of the lower triangular matrices $ \mL$ and $\mL'$ and show how they constrain $\mQ$, we now switch notations from permutation --- or sign-permutation --- matrices $(\mP, \mP', \mQ)$, to their corresponding permutation functions $(\phi, \phi', \psi)$. \eqref{app:eq:lemma_6_decomposition} thus yields
\begin{align}
    \frac{\mL'_{\phi'(i), \phi'(j)}}{\mD'_{ii}}
    =
    \pm \frac{\mL_{\phi(\psi(i)), \phi(j)}}{\mD_{\psi(i), \psi(i)}} \;,
    \quad \forall i, j \in \llbracket 1, p \rrbracket
\end{align}
where the $\pm$ symbol comes from $\mQ$.
In particular, $\mL'$ has unit diagonal, so
\begin{align}
    \mL_{\phi(\psi(i)), \phi(i)}
    =
    \pm \frac{\mD_{\psi(i), \psi(i)}}{\mD'_{ii}}
    \neq
    0 
    \;, \quad \forall i \in \llbracket 1, p \rrbracket
\end{align}
and $\mL$ is lower triangular, so that its non-zero entries must satisfy
\begin{align}
    \phi(i) \leq \phi(\psi(i)) \;,
    \quad \forall i \in \llbracket 1, p \rrbracket \;.
\end{align}
More generally, we can replace $i$ with $\psi(i)$, and so on, and obtain
\begin{align}
    \label{eq:chained_inequalities}
    \phi(i)
    \leq 
    \phi(\psi(i))
    \leq
    \phi(\psi^{(2)}(i))
    \leq
    \ldots
    \quad \forall i \in \llbracket 1, p \rrbracket
\end{align}
where the superscript denotes composition and where we can apply the permutation $\psi$ an arbitrary number of times.

\paragraph{$\mQ$ must be a sign matrix}
Suppose that $\psi$ is not the identity. We can pick an index $k \in \llbracket 1, p \rrbracket$ where $k \neq \psi(k)$. Because $\phi$ and $\psi$ are injective, we can apply them to the inequality any number of times $n \in \mathbb{N}$ and get
\begin{align}
    \phi(\psi^{(n)}(k)) \neq \phi(\psi^{(n+1)}(k))
\end{align}
which together with~\eqref{eq:chained_inequalities} implies
\begin{align}
    \phi(k)
    < \phi(\psi(k)) 
    < \phi(\psi^{(2)}(k)) 
    < \ldots
\end{align}
which can be applied any number of times.
However, each inequality increases the index by at least one, while the index cannot go above $p$. 
Thus, applying the chain at least $p$ times, we have a contradiction.
So, the permutation $\psi$ in $\mQ$ must be the identity, and $\mQ$ boils down to a sign matrix, \textit{i.e.} a diagonal matrix with $1$ and $-1$ on the diagonal.

\paragraph{$\mQ$ must be the identity} Since $\mQ^\top = \mQ$, \eqref{app:eq:lemma_6_decomposition} can be reformulated as
\begin{align}
    \mP'^\top \mL' \mP'
    =
    \mD' \mQ \mD^{-1} \mP^\top \mL \mP
\end{align}
where we now know that $\mD' \mQ \mD^{-1}$ is a diagonal matrix.
The matrix $\mP'^\top \mL' \mP'$ has a diagonal of ones and so too does $\mP^\top \mL \mP$.
It follows that $\mD' \mQ \mD^{-1} = \mI$.
Since $\mD$ and $\mD'$ have positive diagonal entries, it follows that the entries of $\mQ$ cannot equal $-1$. So, $\mQ = \mI$, and thus $\mD' = \mD$.
This concludes the proof of the Lemma.

\subsubsection{Proof of Lemma~\ref{app:lemma:Bi_fully_connected}}
\label{app:sssec:lemma:Bi_fully_connected}

From the equality $\mP^\top \mT \mP = \mP'^\top \mT' \mP'$, we deduce that
\begin{align}
    \label{app:eq:B_Bprime}
    \mT
    =
    \tilde{\mP}^\top \mT' \tilde{\mP}
\end{align}
where $\tilde{\mP} = \mP' \mP^\top$ is the permutation matrix represented by function $\phi$, which is defined such that: $\forall k \in [\![1, p]\!]$, $\tilde{\mP}_{k, \phi(k)} = 1$ and $\forall l \neq \phi(k)$, $\tilde{\mP}_{k, l} = 0$.
Using the notation $\phi$ instead of $\tilde{\mP}$, \eqref{app:eq:B_Bprime} can be rewritten as, $\forall k, l \in [\![1, p]\!]$,
\begin{align}\label{app:eq:B^i_kl}
    \mT_{k, l}
    =
    \mT'_{\phi(k), \phi(l)}
    \enspace.
\end{align}
We proceed by a proof by contradiction: assume that $\tilde{\mP} \neq \mI$.
As a consequence, there exists an index $k$ such that $\phi(k) \neq k$. 
Let us fix such an index as $k$.
We can assume without loss of generality that 
\begin{align}
    \phi(k) > k
\end{align}
otherwise we just have to invert signs and switch rows and columns in the following.
By assumption, the strictly lower triangular part of $\mT$ only has non-zero elements, so we have
\begin{align}
    \mT_{\phi(k), k} 
    \neq 
    0 \enspace.
\end{align}
Yet, from~\eqref{app:eq:B^i_kl}, we know that $\mT_{\phi(k), k} = \mT'_{\phi(\phi(k)), \phi(k)}$ and all the non-zero elements of $\mT'$ are in its strictly lower triangular part, which implies
\begin{align}\label{app:eq:phi_phi>phi}
    \phi(\phi(k))
    >
    \phi(k)
    \enspace.
\end{align}
This logic can be applied repeatedly as in the preceding lemma's proof, and we see that
\begin{align}
    \phi(k)
    <
    \phi^{(2)}(k)
    < \phi^{(3)}(k)
    <
    \dots
    \enspace.
\end{align}
Now, this leads to an infinite strictly increasing sequence, which is contradictory since the index cannot grow greater than $p$.
So, $\tilde{\mP} = \mI$, which means that $\mP' \mP^\top \mP = \mP$, and thus, using the orthogonality of $\mP$,
\begin{align}\label{app:eq:P'_P}
     \mP' = \mP \enspace.
\end{align}
It also follows that $\mT' = \mT$, which concludes the proof of Lemma~\ref{app:lemma:Bi_fully_connected}.

\newpage
\section{Likelihood-based criterion}
\label{app:sec:theory_bivariate_model}

\subsection{Basic setting}
\label{app:ssec:basic_setting}

The fundamental question here is to choose the causal direction for two variables, having many views $i$. We denote the two variables by $x$ and $y$ to avoid too many indices. It could be that the causal direction is $x \rightarrow y$: 
%
\begin{equation}
  y^i = b_i x^i + e^i \qquad \text{for all } i
\end{equation}
or $y \rightarrow x$:
\begin{equation}
  x^i = c_i y^i + d^i \qquad \text{for all } i \enspace.
\end{equation}
Recall that superscripts are for view in the case of the random variables; however, for the regression coefficients we use subscripts because we need to take squares of such quantities. For the direction $x \rightarrow y$, collect the regressor and residuals in the vectors $\vx = (x^1, \dots, x^m)^\top$ and $\ve = (e^1, \dots, e^m)^\top$, respectively; and likewise for $\vy$ and $\vd$ in the opposite direction.
Second-order moment matrices are denoted by $\mSigma_x = \E[\vx \vx^\top]$, $\mSigma_y = \E[\vy \vy^\top]$, and likewise for $\vd$ and $\ve$. 
Note that the independence between $\vx$ and $\ve$ and between $\vy$ and $\vd$, and the fact that $\vx$ and $\vy$ are standardized, imply that 
\begin{equation}
  b_i = c_i = \mathrm{cov}(x_i, y_i) \qquad \text{for all } i \enspace.
\end{equation}
%

\subsection{Derivation of the expected log-likelihood in one direction}
\label{app:ssec:derivation_of_expected_likelihood}

We start by formulating the likelihood of the model in the case where disturbances are assumed to be Gaussian.

Consider the log-likelihood of $(\vx, \vy)$ for the direction $x \rightarrow y$.
Under the model, the prior (marginal) distribution of $\vx$ is $\cN (\vzero, \mSigma_x)$, and the prior distribution of $\ve = \vy - \bm{b} \odot \vx$ is $\cN(\vzero, \mSigma_e)$ --- where~$\odot$ denotes the elementwise product.
Thus, we have
%
\begin{align}
    \log p(\vx,\vy; \bm{b}, \mSigma_x, \mSigma_e, x \rightarrow y)
    &=
    \log p(\vy | \vx; \bm{b}, \mSigma_x, \mSigma_e, x \rightarrow y) 
    + \log p(\vx; \bm{b}, x \rightarrow y) \\
    &= 
    -(\vy - \bm{b} \odot \vx)^T \mSigma_e^{-1} (\vy - \bm{b} \odot \vx) 
    - \log |\mSigma_e|
    - \vx^T \mSigma_x^{-1} \vx  
    - \log |\mSigma_x| \\
    &= 
    - \tr \big( \mSigma_e^{-1} (\vy - \bm{b} \odot \vx) (\vy - \bm{b} \odot \vx)^\top \big)
    - \log |\mSigma_e|
    %
    \nonumber \\ &
    - \tr \big( \mSigma_x^{-1} \vx \vx^\top \big)
    - \log |\mSigma_x|
\end{align}
where $\tr$ is the trace operator, and where we neglect the part of the normalization constant that does not depend on any parameters.


The expected log-likelihood is thus
\begin{align}
    \E [\log p(\vx,\vy; \bm{b}, \mSigma_x, \mSigma_e, x \rightarrow y)] 
    &= 
    - \tr \big( \mSigma_e^{-1} \E[ (\vy - \bm{b} \odot \vx) (\vy - \bm{b} \odot \vx)^\top ] \big)
    - \log |\mSigma_e| \nonumber \\
    &- \tr \big( \mSigma_x^{-1} \E [\vx \vx^\top] \big)
    - \log |\mSigma_x| \\
    &=
    - \log |\mSigma_x| - \log |\mSigma_e| - 2m \enspace.
\end{align}
Denote by $\cL$ this expected log-likelihood. We have (up to constants):
\begin{equation}
\label{app:eq:population_log_likelihood_xtoy}
    \cL (\mSigma_x, \mSigma_e; x \rightarrow y)
    =
    - \log |\mSigma_x| - \log |\mSigma_e| \enspace.
\end{equation}

\subsection{Derivation of the likelihood-based criterion}
\label{app:ssec:derivation_of_likelihood_based_criterion}

We now derive a criterion that can be used to find the causal direction.

The expected log-likelihood of the model for the direction $y \rightarrow x$ can be obtained by switching the roles of $x$ and $y$ in~\eqref{app:eq:population_log_likelihood_xtoy}. Thus, we have
\begin{equation}
\label{app:eq:population_log_likelihood_ytox}
    \cL (\mSigma_y, \mSigma_d; y \rightarrow x)
    =
    - \log |\mSigma_y| - \log |\mSigma_d|
\end{equation}
where we recall that $\mSigma_y = \E[\vy \vy^\top]$ and $\mSigma_d = \E[\vd \vd^\top]$.

The correct direction can then be found by comparing the expected log-likelihoods of both directions. In particular, we calculate the difference between the two log-likelihoods, which can be reformulated as a likelihood ratio (LR):
\begin{equation}
    \mathrm{LR}
    :=
    \cL (\mSigma_x, \mSigma_e; x \rightarrow y)
    -
    \cL (\mSigma_y, \mSigma_d; y \rightarrow x)
    =
    \E \left[ \log \frac{p(\vx,\vy; \bm{b}, \mSigma_x, \mSigma_e, x \rightarrow y)}{p(\vx,\vy; \vc, \mSigma_y, \mSigma_d, y \rightarrow x)} \right] \enspace.
\end{equation}
Using~\eqref{app:eq:population_log_likelihood_xtoy} and~\eqref{app:eq:population_log_likelihood_ytox}, the criterion $\mathrm{LR}$ becomes
\begin{equation}
\label{app:eq:true_criterion}
    \mathrm{LR}
    =
    - \log |\mSigma_x| - \log |\mSigma_e|
    + \log |\mSigma_y| + \log |\mSigma_d| \enspace.
\end{equation}
Intuitively, the criterion in~\eqref{app:eq:true_criterion} will be large if the $x^i$ are highly correlated (more than the $y^i$) over the views, since then the first determinant will be small. This should make intuitive sense since it means the cause is highly structured. 
Likewise, if the residuals for $x \rightarrow y$ are highly structured, that direction is more likely. In other words, the direction that exhibits more structure wins.

Surprisingly, in this formulation, there is actually no particular model for the dependencies of the disturbances:  the likelihood should simply prefer dependent disturbances. It is important to note that the criterion makes sense even if the distribution is non-Gaussian, as will be proven below. We thus obtain a method that works for data of any distribution, but using only second-order statistics.

\subsection{Non-negativity of the likelihood-based criterion}
\label{app:ssec:non_negativity_of_likelihood_based_criterion}

In this section, we prove that the criterion in~\eqref{app:eq:true_criterion} is non-negative for the direction $x \rightarrow y$. Importantly, no assumption of Gaussianity needs to be made here.


Consider the direction $x \rightarrow y$, and let us write $\mB := \mathrm{diag}(b_1, \dots, b_m)$. 
Using the matrix $\mB$ rather than the vector $\bm{b}$ makes notations clearer in the proofs; please note that this notation is different from the main paper. In the following, we use interchangeably $\bm{b}$ and $\mB$.
The covariance matrix of $\vy = \mB \vx + \ve$ is
\begin{equation}
\label{app:eq:covariance_y}
    \mSigma_y = \mB \mSigma_x \mB + \mSigma_e \enspace.
\end{equation}
Given that $b_i = c_i$ for all $i$, we have $\vd = \vx - \mB \vy = \vx - \mB (\mB \vx + \ve) = (\mI - \mB^2) \vx - \mB \ve$. Thus, the covariance matrix of $\vd$ is
\begin{equation}
    \mSigma_d = (\mI - \mB^2) \mSigma_x (\mI - \mB^2) + \mB \mSigma_e \mB \enspace.
\end{equation}
Using these expressions, the criterion in~\eqref{app:eq:true_criterion} becomes
\begin{align}
    \mathrm{LR}
    =
    - \log |\mSigma_x| 
    - \log |\mSigma_e|
    + \log \left| \mB \mSigma_x \mB + \mSigma_e \right| 
    + \log \left| (\mI - \mB^2) \mSigma_x (\mI - \mB^2) + \mB \mSigma_e \mB \right| 
    \enspace.
\end{align}
Collect the variances of the residuals $\ve$ in the diagonal matrix
\begin{equation}
\label{app:eq:definition_of_L}
    \mL := \mathrm{diag} \left( \sqrt{\mathrm{var}(e^1)}, \dots, \sqrt{\mathrm{var}(e^m)} \right)
\end{equation}
and define the correlation matrix of the residuals as $\widetilde{\mSigma}_e$, so that
\begin{equation}
\label{app:eq:identity_covariance_correlation_of_e}
    \mSigma_e = \mL \widetilde{\mSigma}_e \mL \enspace.
\end{equation}
Now, the objective becomes
\begin{multline}
    \mathrm{LR}
    =
    - \log |\mSigma_x| 
    - \log |\widetilde{\mSigma}_e|
    - \log |\mL^2|
    + \log \left| \mB \mSigma_x \mB + \mL \widetilde{\mSigma}_e \mL \right| 
    \\+ \log \left| (\mI - \mB^2) \mSigma_x (\mI - \mB^2) + \mB \mL \widetilde{\mSigma}_e \mL \mB \right| 
    \enspace.
\end{multline}
Importantly, the fact that $x$ and $y$ are standardized means that $\mSigma_x$ and $\mSigma_y$ have unit diagonal, so~\eqref{app:eq:covariance_y} implies 
\begin{equation}
    \mL^2 + \mB^2 = \mI \enspace
\end{equation}
and it follows that the diagonal entries of $\mL$ and $\mB$ are in $[-1, 1]$.
Thus, we have
\begin{align}
    \mathrm{LR}
    &=
    - \log |\mSigma_x| 
    - \log |\widetilde{\mSigma}_e|
    - \log |\mL^2|
    + \log \left| \mB \mSigma_x \mB + \mL \widetilde{\mSigma}_e \mL \right| 
    + \log \left| \mL^2 \mSigma_x \mL^2 + \mL \mB \widetilde{\mSigma}_e \mB \mL \right| \\
    &= 
    - \log |\mSigma_x| 
    - \log |\widetilde{\mSigma}_e|
    + \log \left| \mB \mSigma_x \mB + \mL \widetilde{\mSigma}_e \mL \right| 
    + \log \left| \mL \mSigma_x \mL + \mB \widetilde{\mSigma}_e \mB \right|
\end{align}
where $\mL$ and $\mB$ commuted in the first equality because they are diagonal. 

The following lemma (proven in Appendix~\ref{app:ssec:proof_of_lemma}) shows that this criterion is non-negative, and even positive under some assumptions.

\begin{lemma}
\label{app:lemma:positivity_of_criterion}
    Let $\mSigma_x$ and $\widetilde{\mSigma}_e$ be $m \times m$ real symmetric positive definite matrices with unit diagonal entries. Let $\mB = \mathrm{diag}(b_1, \dots, b_m)$ and $\mL = \mathrm{diag}(l_1, \dots, l_m)$ be diagonal matrices with entries in $[-1, 1]$, satisfying $\mB^2 + \mL^2 = \mI$ (i.e. $b_i^2 + l_i^2 = 1$ for all $i$).
    Define the function
    \begin{align}
        J(\mB, \mL) := - \log \det \mSigma_x - \log \det \widetilde{\mSigma}_e + \log \det \left( \mB \mSigma_x \mB + \mL \widetilde{\mSigma}_e \mL \right) + \log \det \left( \mL \mSigma_x \mL + \mB \widetilde{\mSigma}_e \mB \right).
    \end{align}
    Under the above assumptions, we have $J(\mB, \mL) \ge 0$. Moreover, $J(\mB, \mL) = 0$ if and only if
    \begin{equation}
        \mL \, \widetilde\mSigma_e \, \mB \; = \; \mB \, \mSigma_x \, \mL \enspace.
    \end{equation}
    %
\end{lemma}

\subsection{Proof of a useful lemma}
\label{app:ssec:proof_of_lemma}

Let us prove Lemma~\ref{app:lemma:positivity_of_criterion}.

\paragraph{Step 1: define $\mA$ and $\mC$ ---}
Define
\begin{equation}
    \mA := \mB \mSigma_x \mB + \mL \widetilde \mSigma_e \mL 
    \qquad \text{and} \qquad
    \mC := \mL \mSigma_x \mL + \mB \widetilde \mSigma_e \mB \enspace.
\end{equation}
We have
\begin{equation}
    J(\mB, \mL) 
    := 
    -\log\det \mSigma_x - \log\det \widetilde\mSigma_e + \log\det \mA + \log\det \mC \enspace.
\end{equation}

\paragraph{Step 2: define $\mM$ ---}
Define the $2m \times 2m$ block matrix
\begin{equation}
    \mM := \begin{pmatrix} \mB & \mL \\ -\mL & \mB \end{pmatrix} \enspace.
\end{equation}
Since $\mB$ and $\mL$ are diagonal, they commute, hence $\mB \mL = \mL \mB$. Using $\mB^2 + \mL^2 = \mI$, we compute
\begin{equation}
    \mM \mM^\top
    = \begin{pmatrix} \mB \mB + \mL \mL & -\mB \mL + \mL \mB \\ -\mL \mB + \mB \mL & \mL \mL + \mB \mB \end{pmatrix}
    = \begin{pmatrix} \mI & \vzero \\ \vzero & \mI \end{pmatrix} = \mI_{2m} \enspace.
\end{equation}
Thus $\mM$ is orthonormal, so $\det \mM = \pm 1$.

\paragraph{Step 3: define $\mathcal{B}$ ---}
Consider the block-diagonal matrix $\mathrm{diag}(\mSigma_x, \widetilde\mSigma_e)$ and conjugate by $\mM$:
\begin{equation}
    \mathcal B 
    := \mM \begin{pmatrix}\mSigma_x & \vzero \\[4pt] \vzero & \widetilde\mSigma_e \end{pmatrix} \mM^\top
    = \begin{pmatrix} \mA & \mS \\[4pt] \mS^\top & \mC \end{pmatrix},
    \qquad 
    \mS := \mL \widetilde\mSigma_e \mB - \mB \mSigma_x \mL \enspace.
\end{equation}
Because $\mM$ is orthonormal and $\mSigma_x, \widetilde\mSigma_e \succ \vzero$, the matrix $\mathcal B$ is symmetric positive definite (SPD). In particular, $\mA$ and $\mC$ are also SPD (as they are principal blocks of a SPD matrix) and thus $\mA^{-1}$ exists.

\paragraph{Step 4: Schur complement of $\mA$ in $\mathcal B$ ---}
For a symmetric block matrix $\begin{pmatrix} \mA & \mS \\ \mS^\top & \mC \end{pmatrix}$ with $\mA \succ \vzero$, the Schur complement of $\mA$ is $\mC - \mS^\top \mA^{-1} \mS$, and the Schur's formula yields
\begin{equation}
    \det \begin{pmatrix} \mA & \mS \\ \mS^\top & \mC \end{pmatrix}
    = \det(\mA) \, \det \big( \mC - \mS^\top \mA^{-1} \mS \big) \enspace.
\end{equation}
Applying this to $\mathcal B$ and using multiplicativity of determinant and the fact that $\det \mM = \pm 1$ gives
\begin{equation}
    \det(\mSigma_x) \det(\widetilde\mSigma_e)
    = \det(\mathcal B)
    = \det(\mA) \, \det \big( \mC - \mS^\top \mA^{-1} \mS \big) \enspace.
\end{equation}

\paragraph{Step 5: derive a matrix inequality ---}
Because $\mathcal B \succ \vzero$ and $A \succ \vzero$, it follows from the Schur complement characterization of positive definiteness that $\mC - \mS^\top \mA^{-1} \mS \succ \vzero$. 
In particular, we have $\det(\mC - \mS^\top \mA^{-1} \mS) > 0$. 
Moreover, since $\mA \succ \vzero$, we have $\mS^\top \mA^{-1} \mS \succeq \vzero$, and thus $\mC - \mS^\top \mA^{-1} \mS \preceq \mC$.

\paragraph{Step 6: derive a determinant inequality ---}
If $\vzero \preceq \mX \preceq \mY$ and $\mY \succ \vzero$, then conjugating by $\mY^{-1/2}$ shows $\vzero \preceq \mY^{-1/2} \mX \mY^{-1/2} \preceq  \mI$. Thus, all eigenvalues of $\mY^{-1/2} \mX \mY^{-1/2}$ lie in $[0,1]$, so their product satisfies $\det(\mY^{-1/2} \mX \mY^{-1/2}) \le 1$. Hence $\det(\mX) \le \det(\mY)$. 
Applying this to $\mX := \mC - \mS^\top \mA^{-1} \mS$ and $\mY := \mC$ yields
\begin{equation}
    \det \big( \mC - \mS^\top \mA^{-1} \mS \big) \le \det(\mC) \enspace.
\end{equation}

\paragraph{Step 7: $J(\mB, \mL) \ge 0$ ---}
From Steps 4 and 6, we get
\begin{equation}
    \det(\mSigma_x) \det(\widetilde\mSigma_e)
    = \det(\mA) \, \det \big( \mC - \mS^\top \mA^{-1} \mS \big)
    \le \det(\mA) \det(\mC) \enspace.
\end{equation}
Taking natural logarithms and rearranging gives
\begin{equation}
    \log\det(\mA) + \log\det(\mC) \; \ge \; \log\det(\mSigma_x) + \log\det(\widetilde\mSigma_e) \enspace,
\end{equation}
which is precisely $J(\mB, \mL) \ge 0$.

\paragraph{Step 8: equality condition ---}
The equality $J(\mB, \mL) = 0$ holds iff 
%
\begin{equation}
    \det\big( \mC - \mS^\top \mA^{-1} \mS \big) = \det(\mC) \enspace.
\end{equation}
Write $\mK := \mC^{-1/2} \big( \mS^\top \mA^{-1} \mS \big) \mC^{-1/2} \succeq \vzero$. We have
\begin{align}
    \det\big( \mC - \mS^\top \mA^{-1} \mS \big) 
    &= \det(\mC) \det\big( \mI - \mC^{-1} \mS^\top \mA^{-1} \mS \big) \\
    &= \det(\mC) \det\big( \mC^{-\frac12} \mC^{\frac12} - \mC^{-\frac12} \mC^{-\frac12} \mS^\top \mA^{-1} \mS \mC^{-\frac12} \mC^{\frac12} \big) \\
    &= \det(\mC) \det\big( \mC^{-\frac12} \big) \det (\mI - \mK) \det\big( \mC^{\frac12} \big) \\
    &= \det(\mC) \det (\mI - \mK) \enspace.
\end{align}
Thus, the equality condition is $\det(\mI - \mK) = 1$. From Step 5, we know that $\vzero \preceq \mC - \mS^\top \mA^{-1} \mS \preceq \mC$. Conjugating by $\mC^{-\frac12}$ gives $\vzero \preceq \mI - \mK \preceq \mI$, so each eigenvalue $\lambda_i$ of $\mK$ satisfies $0 \le \lambda_i \le 1$. The eigenvalues of $\mI - \mK$ are $1 - \lambda_i \in [0, 1]$, hence
%
\begin{equation}
    \prod_i (1 - \lambda_i) = \det(\mI - \mK) = 1 \enspace.
\end{equation}
The equality thus holds iff $\lambda_i = 0$ for all $i$, hence $\mK = \vzero$.
Since $\mC^{-\frac12}, \mA^{-1} \succ \vzero$, we get $\mS = \vzero$. Conversely, $\mS = \vzero$ clearly forces equality. Therefore, the equality occurs exactly when
\begin{equation}
    \mS = \mL \widetilde\mSigma_e \mB - \mB \mSigma_x \mL = \vzero \enspace,
\end{equation}
\textit{i.e.}
\begin{equation}
    \forall i,j: \qquad b_i l_j (\mSigma_x)_{ij} = l_i b_j (\widetilde\mSigma_e)_{ij} \enspace.
\end{equation}
This completes the proof.

\subsection{Conditions for strict positivity of the likelihood-based criterion}
\label{app:ssec:conditions_for_strict_positivity_of_likelihood_based_criterion}

From now on, we assume that the model is not degenerate in the sense that disturbances have a positive variance, that is $l_i = \mathrm{var}(e_i) > 0$ for all $i$.

The criterion $\mathrm{LR}$ derived in Appendix~\ref{app:ssec:non_negativity_of_likelihood_based_criterion} is used to determine the causal direction from its sign, with $\mathrm{LR} = 0$ indicating that the two directions cannot be distinguished.  
Lemma~\ref{app:lemma:positivity_of_criterion} shows that $\mathrm{LR} = 0$ if and only if
\begin{equation}
\label{app:eq:temporary_condition}
    \mL \widetilde\mSigma_e \mB = \mB \mSigma_x \mL \enspace.
\end{equation}
Entrywise, this is equivalent to
\begin{equation}
\label{app:eq:equality_condition_scalar}
    \forall i,j: \qquad l_i b_j (\widetilde\mSigma_e)_{ij} = b_i l_j (\mSigma_x)_{ij} \enspace.
\end{equation}
Let us fix $i, j \in [\![1, m]\!]$.

\paragraph{Case 1: $b_i = 0$ and $b_j = 0$ ---} The equality immediately holds.

\paragraph{Case 2: $b_i = 0$ and $b_j \ne 0$ ---} The right-hand side in~\eqref{app:eq:equality_condition_scalar} is equal to 0. Since $l_i = \mathrm{var}(e^i) > 0$ and $b_j \ne 0$, we must have $(\widetilde{\mSigma}_e)_{ij} = 0$.

\paragraph{Case 3: $b_i \ne 0$ and $b_j = 0$ ---} The left-hand side in~\eqref{app:eq:equality_condition_scalar} is equal to 0. Since $l_j = \mathrm{var}(e^j) > 0$ and $b_i \ne 0$, we must have $(\mSigma_x)_{ij} = 0$.

\paragraph{Case 4: $b_i \ne 0$ and $b_j \ne 0$ ---} Multiplying~\eqref{app:eq:equality_condition_scalar} by $\frac{l_j}{b_i b_j^2}$, and using $l_j^2 = 1 - b_j^2$, we obtain
\begin{equation}
\label{app:eq:unidentifiability_condition_temporary}
    \frac{l_i l_j}{b_i b_j} (\widetilde{\mSigma}_e)_{ij}
    =
    \bigg( \frac{1}{b_j^2} - 1 \bigg) (\mSigma_x)_{ij} \enspace.
\end{equation}
Switching the roles of $i$ and $j$, and using the symmetry of $(\widetilde{\mSigma}_e)_{ij}$ and $(\mSigma_x)_{ij}$, yield
\begin{equation}
    \frac{l_i l_j}{b_i b_j} (\widetilde{\mSigma}_e)_{ij}
    =
    \bigg( \frac{1}{b_i^2} - 1 \bigg) (\mSigma_x)_{ij} \enspace.
\end{equation}
hence
\begin{equation}
\label{app:eq:unidentifiability_sufficient_condition_bi}
    b_i^2 (\mSigma_x)_{ij} = b_j^2 \; (\mSigma_x)_{ij} \enspace.
\end{equation}
Similarly, exchanging the roles of $(b_i, b_j, (\mSigma_x)_{ij})$ and $(l_i, l_j, (\widetilde{\mSigma}_e)_{ij})$ yields
\begin{equation}
\label{app:eq:unidentifiability_sufficient_condition_li}
    l_i^2 (\widetilde{\mSigma}_e)_{ij} = l_j^2 \; (\widetilde{\mSigma}_e)_{ij} \enspace.
\end{equation}

Now suppose $(\mSigma_x)_{ij} \neq 0$ or $(\widetilde{\mSigma}_e)_{ij} \neq 0$.  
From~\eqref{app:eq:unidentifiability_sufficient_condition_bi} and~\eqref{app:eq:unidentifiability_sufficient_condition_li} it follows that $b_i^2 = b_j^2$ or $l_i^2 = l_j^2$. Since $b_i^2 + l_i^2 = 1$, both equalities hold simultaneously:
%
\begin{equation}
    |b_i| = |b_j|
    \qquad \text{and} \qquad
    l_i = l_j
\end{equation}
where the asbolute value can be dropped for $l_i$ because $l_i = \mathrm{var}(e_i) > 0$.
Plugging these equalities into~\eqref{app:eq:unidentifiability_condition_temporary} yields
\begin{equation}
    l_j^2 (\widetilde{\mSigma}_e)_{ij}
    =
    \pm (1 - b_j^2) (\mSigma_x)_{ij}
\end{equation}
hence
\begin{equation}
    |(\widetilde{\mSigma}_e)_{ij}|
    =
    |(\mSigma_x)_{ij}| \enspace.
\end{equation}
Moreover, for~\eqref{app:eq:equality_condition_scalar} to hold we must also have the sign equality
\begin{equation}
    \mathrm{sign}(b_i) \; \mathrm{sign}(b_j)
    =
    \mathrm{sign}((\mSigma_x)_{ij}) \; \mathrm{sign}((\widetilde{\mSigma}_e)_{ij}) \enspace.
\end{equation}

Altogether, either both $(\mSigma_x)_{ij}$ and $(\widetilde{\mSigma}_e)_{ij}$ vanish, or the following four conditions hold:
\begin{align}
    |(\widetilde{\mSigma}_e)_{ij}| &= |(\mSigma_x)_{ij}| \ne 0 \label{app:eq:condition_1} \\
    |b_i| &= |b_j| \label{app:eq:condition_2} \\
    l_i &= l_j \label{app:eq:condition_3} \\
    \mathrm{sign}(b_i) \; \mathrm{sign}(b_j) &= \mathrm{sign}((\mSigma_x)_{ij}) \; \mathrm{sign}((\widetilde{\mSigma}_e)_{ij}) \enspace. \label{app:eq:condition_4}
\end{align}
%
%
Conversely, if $(\mSigma_x)_{ij} = (\widetilde{\mSigma}_e)_{ij} = 0$, or if conditions~\eqref{app:eq:condition_1}–\eqref{app:eq:condition_4} are satisfied, then~\eqref{app:eq:equality_condition_scalar} holds.

\paragraph{Conclusion ---} By contraposition,~\eqref{app:eq:temporary_condition} is violated whenever there exists $(i, j)$ such that one of the following cases holds:
\begin{itemize}
    \item $b_i = 0$, $b_j \ne 0$, and $(\widetilde{\mSigma}_e)_{ij} \ne 0$

    \item $b_i \ne 0$, $b_j = 0$, and $(\mSigma_x)_{ij} \ne 0$

    \item $b_i \ne 0$, $b_j \ne 0$, $(\mSigma_x)_{ij} \ne 0$ or $(\widetilde{\mSigma}_e)_{ij} \ne 0$, and at least one of the conditions~\eqref{app:eq:condition_1}–\eqref{app:eq:condition_4} is not met.
\end{itemize}

%
In particular, a simple sufficient condition is
\begin{equation}
    \exists i, j : \qquad \big( b_i \ne 0 \qquad \text{or} \qquad b_j \ne 0 \big) \qquad \text{and} \qquad |(\widetilde{\mSigma}_e)_{ij}| \ne |(\mSigma_x)_{ij}| \enspace.
\end{equation}
which is the basis of Assumptions~\ref{assumption:simple_diversity_condition_bivariate} and~\ref{assumption:simple_diversity_condition} given in the main paper.

\newpage
\section{Cross-covariance-based criterion}
\label{app:sec:cross_covariance_based_criterion}

Consider the setting described in Appendix~\ref{app:ssec:basic_setting}.

\subsection{Derivation of the cross-covariance-based criterion}
\label{app:ssec:derivation_of_cross_covariance_based_criterion}

A simple approach to infer the causal direction is to exploit the assumption that root variables and residuals are independent. Specifically, if the true direction is $x \rightarrow y$, then $\vx$ and $\ve$ are independent; if the true direction is $y \rightarrow x$, then $\vy$ and $\vd$ are independent. 
This implies that the cross-moment matrix $\E[\ve \vx^\top]$ should vanish when $x \rightarrow y$, while $\E[\vd \vy^\top]$ should vanish when $y \rightarrow x$. 
Comparing the Frobenius norms of these matrices therefore provides a natural criterion, based solely on cross-covariances:
\begin{equation}
    \mathrm{FC} := \left\Vert \E[\vd \vy^\top] \right\Vert_F - \left\Vert \E[\ve \vx^\top] \right\Vert_F \enspace.
\end{equation}
Intuitively, if this criterion is positive we deduce that $x \rightarrow y$, and if negative we deduce that $y \rightarrow x$.

\subsection{Conditions for strict positivity of the cross-covariance criterion}
\label{app:ssec:non_negativity_of_cross_covariance_based_criterion}

Assume without loss of generality that the true causal direction is $x \rightarrow y$, and let us analyze the sign of the criterion $\mathrm{FC}$.

Given that $x \rightarrow y$, the true model is 
\begin{equation}
    \vy = \mB \vx + \ve
\end{equation}
where $\mB = \mathrm{diag}(\mathrm{cov}(x^1, y^1), \dots, \mathrm{cov}(x^m, y^m))$, and $\ve$ is independent from $\vx$. So, we have 
\begin{equation}
\label{app:eq:cross_covariance_ex_is_0}
    \E[\vx \ve^\top] = \vzero \enspace.
\end{equation}
In the wrong direction, the residuals take the form 
\begin{equation}
    \vd = \vx - \mC \vy , \qquad \mC = \mB \enspace.
\end{equation}
The cross-covariance between observations $\vy$ and residuals $\vd$ is then
\begin{align}
    \mathbb{E} \big[ \vy \vd^\top \big] 
    &= \mathbb{E} \big[ \vy (\vx - \mB \vy)^\top \big] \\
    &= \mathbb{E} \big[ (\mB \vx + \ve) (\vx - \mB (\mB \vx + \ve))^\top \big] \\
    &= \mathbb{E} \big[ \mB \vx \vx^\top (\mI - \mB^2) \big] 
    + \mathbb{E} \big[ \ve \vx^\top (\mI - \mB^2)) \big] 
    - \mathbb{E} \big[ \mB \vx \ve^\top \mB \big] 
    - \mathbb{E} \big[ \ve \ve^\top \mB \big] \\
    &= \mB \; \mSigma_x \; (\mI - \mB^2) - \mSigma_e \; \mB \label{app:eq:cross_covariance_dy}
\end{align}
where $\mSigma_x = \mathbb{E} \big[ \vx \vx^\top \big]$ and $\mSigma_e = \mathbb{E} \big[ \ve \ve^\top \big]$.

From~\eqref{app:eq:cross_covariance_ex_is_0} and the non-negativity of the Frobenius norm, we see that $\mathrm{FC}$ is always non-negative. Moreover, using~\eqref{app:eq:cross_covariance_dy}, we deduce that $\mathrm{FC} = 0$ if and only if
\begin{equation}
    \mB \; \mSigma_x \; (\mI - \mB^2) - \mSigma_e \; \mB = \vzero \enspace.
\end{equation}
Note that the fact that $l_i > 0$, for all $i$, makes $\mL$ invertible. So, using the identity $\mI - \mB^2 = \mL^2$, and multiplying on the right by $\mL^{-1}$, we obtain
\begin{equation}
    \mB \; \mSigma_x \; \mL
    =
    \mSigma_e \; \mB \; \mL^{-1}
\end{equation}
hence
\begin{equation}
    \mB \; \mSigma_x \; \mL
    =
    \mL \; \widetilde{\mSigma}_e \; \mB
\end{equation}
since $\mSigma_e = \mL \; \widetilde{\mSigma}_e \; \mL$, and $\mB$ and $\mL$ commute.


This condition coincides exactly with~\eqref{app:eq:temporary_condition}. 
In particular, this shows that the likelihood-based criterion and the cross-covariance-based criterion share the same consistency assumptions.

\newpage
\section{Identifiability of LiMVAM}
\label{app:sec:identifiability_limvam}

\subsection{Multivariate conditions for strict positivity of a criterion} 
\label{app:ssec:multivariate_conditions_strict_positivity}

The following assumption is a multivariate generalization of the assumption detailed in~\ref{app:ssec:conditions_for_strict_positivity_of_likelihood_based_criterion} and~\ref{app:ssec:non_negativity_of_cross_covariance_based_criterion}.
%
\begin{assumption}[Diversity of the views]
\label{assumption:diversity_condition}
Let $j \ne j'$ denote any two component indices such that $\vx_j \rightarrow \vx_{j'}$.
Then, there exist two views $i$ and $i'$ such that the three following conditions hold:
\begin{enumerate}
\item $\mathrm{corr}(\vx_j^i, \vx_{j'}^i) \ne 0 \ \text{or} \ \mathrm{corr}(\vx_j^{i'}, \vx_{j'}^{i'}) \ne 0$
\item $\mathrm{corr}(\ve_j^i, \ve_j^{i'}) \ne 0$ or $\mathrm{corr}(\ve_{j'}^i, \ve_{j'}^{i'}) \ne 0$ 

\item At least one of the four following inequalities is met:
\begin{align*}
&|\mathrm{corr}(\ve_j^i, \ve_j^{i'})| \ne |\mathrm{corr}(\ve_{j'}^i, \ve_{j'}^{i'})|, 
\,
|\mathrm{corr}(\vx_j^i, \vx_{j'}^i)| \ne |\mathrm{corr}(\vx_j^{i'}, \vx_{j'}^{i'})|,
\,
\mathrm{var}(\ve_{j'}^i) \ne \mathrm{var}(\ve_{j'}^{i'}) 
\\
&\mathrm{sign}(\mathrm{corr}(\vx_j^i, \vx_{j'}^i)) \; \mathrm{sign}(\mathrm{corr}(\vx_j^{i'}, \vx_{j'}^{i'})) \ne \mathrm{sign}(\mathrm{corr}(\ve_j^i, \ve_j^{i'})) \; \mathrm{sign}(\mathrm{corr}(\ve_{j'}^i, \ve_{j'}^{i'}))
\end{align*}
%
\end{enumerate}
where $\mathrm{corr}$ denotes the correlation coefficient and $\mathrm{sign}$ denotes the sign function.
\end{assumption}
These conditions can only be met for indices $i \neq i'$, so they can be interpreted as a need for correlation and diversity across views. Because they may be hard to interpret at first glance, we next  visualize in Figure~\ref{fig:visualization_of_assumption} what these conditions mean for the causal graph.
\begin{figure}[ht]
    \centering
    \includegraphics[width=0.9\textwidth]{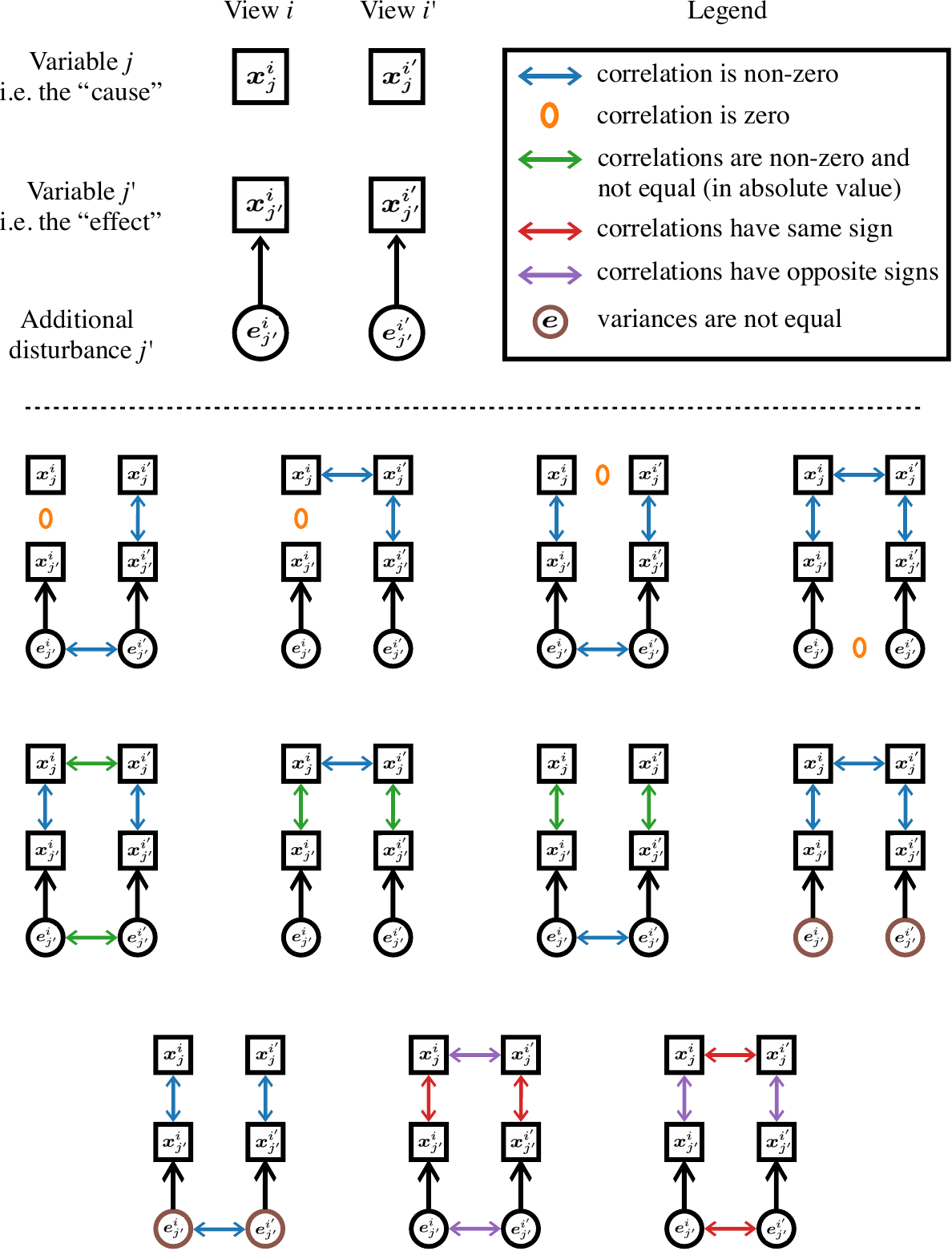}
    \caption{Assume the direction is $\vx_j \rightarrow \vx_{j'}$. Assumption~\ref{assumption:diversity_condition} corresponds to being in one of the 11 situations showed in the figure. Note that no arrow means that no information is required.}
    \label{fig:visualization_of_assumption}
\end{figure}


In particular, Assumption~\ref{assumption:diversity_condition} immediately implies its simpler version used in the main text, Assumption~\ref{assumption:simple_diversity_condition}.

\subsection{An example for interpreting the SOS-based identifiability condition in Assumption~\ref{assumption:simple_diversity_condition}}
\label{app:sec:condition_intepretation}
 
Consider one view and two variables denoted for notational simplicity $x$ and $y$. We have 
\begin{equation}
y^1=b^1 x^1 + e^1
\end{equation}
and assume for simplicity that $x^1$ and $y^1$ have been standardized. 

Now, suppose we observe two datasets from this one view, denote them by $(x^1,y^1)$ and $({x^1}',{y^1}')$, which have identical distributions and are independent of each other.
Crucially, all this data is from one view.
However, 
we could artificially generate a new “view” by adding those two datasets together:
\begin{equation}
x^2=x^1 + {x^1}'
\end{equation}
and likewise for $y^2$ and $e^2$. Clearly the new data follows the SEM with the same $b$:
\begin{equation}
y^2=b^1 x^2 + e^2
\end{equation}
since we are simply adding each of the terms separately in the two datasets.

But this “multi-view” data is degenerate. In fact, the correlation coefficients are
\begin{equation}
\text{corr}(x^1,x^2)=\frac{\text{cov}(x^1,x^1) + \text{cov}(x^1,{x^1}')}{\text{std}(x^1)\text{std}(x^2)} = \frac{\text{var}(x^1)}{\sqrt{2}\text{std}(x^1)\text{std}(x^1)} =\frac{1}{\sqrt{2}} \enspace.
\end{equation}
An exactly identical calculation applies for $\text{corr}(e^1,e^2)$ which is also equal to $1/\sqrt{2}$.

Thus, we see that this case violates the simple identifiability condition in Assumption~\ref{assumption:simple_diversity_condition}. This is understandable since no new information is created when computing the variable in the second view, and the two views are thus degenerate.


\clearpage
\newpage
\subsection{Proof of Theorem~\ref{thm:identifiability_causal_matrices_pairwise}}
\label{app:ssec:proof_of_thm1}

We proceed in three steps. (i) We show that the LiMVAM model always admits at least one root variable and that our LR and FC criteria recover one consistently. (ii) We prove a \emph{recursive residuals} lemma: removing a recovered root and regressing on it preserves the model form and the causal ordering, which implies that the full (partial) ordering is consistently recoverable. (iii) Conditional on the recovered ordering, each view-specific strictly lower-triangular system can be estimated consistently, hence the adjacency matrices are identifiable.


\paragraph{Preliminaries}
For clarity, we unpack the “shared ordering” assumption. This assumption is equivalent to the existence of a \emph{main} DAG $\cG$ that is the union of view-specific DAGs $\{\cG^i\}_{i=1}^m$: an edge appears in $\cG$ if it appears in at least one $\cG^i$. Root variables are therefore common candidates across views. Furthermore, in terms of vocabulary, we distinguish (a) \emph{directed paths} (a sequence of edges with consistent direction) from (b) \emph{direct edges} (one edge between adjacent variables). An \emph{indirect path} is a directed path with at least two edges.



\subsubsection{Recovering a root variable}
\label{app:sssec:proof_of_theorem_part_1}

Let observations follow LiMVAM in~\eqref{eq:multivariate_model}. Since all views share an ordering, there exists a (possibly non-unique) permutation $\mP$ such that
\begin{equation}
    \mT^i = \mP^\top \mB^i \mP \quad \text{is strictly lower triangular for all } i\in[\![1,m]\!] \enspace.
\end{equation}
Reordering the observations and disturbances with $\vx'^i = \mP \vx^i$ and $\ve'^i=\mP\ve^i$, the model is
\begin{equation}
    \vx'^i = \mT^i \vx'^i + \ve'^i, \qquad i\in[\![1,m]\!] \enspace.
\end{equation}
The first row of each $\mT^i$ is zero, so $x'^i_1$ is exogenous; hence a root variable always exists.

To \emph{find} a root, consider distinct indices $(j,k)$ and define $M_{jk}$ to be the criterion from~\eqref{eq:criterion} or~\eqref{eq:criterion_covariance} (with $M_{kj}=-M_{jk}$). Stack these into $\mM\in\R^{p\times p}$ with zero diagonal, and analyze signs under Assumption~\ref{assumption:diversity_condition} (which concerns only direct relations):
\begin{enumerate}
    \item \textbf{Direct edge:} $\vx_j \to \vx_k$ or $\vx_k \to \vx_j$. Under the bivariate conditions derived from Assumption~\ref{assumption:diversity_condition}, $M_{jk}>0>M_{kj}$ when $\vx_j\to \vx_k$, and $M_{jk}<0<M_{kj}$ when $\vx_k\to \vx_j$.


    \item \textbf{Indirect path:} $\vx_j \to \dots \to \vx_k$ or $\vx_k \to \dots \to \vx_j$.  
    Consider the reduced form of the model in~\eqref{eq:multi_view_ica_model}, $\vx^i = \mA^i \ve^i$ with $\mA^i = (\mI - \mB^i)^{-1}$. The entry $A^i_{kj}$ is the \emph{total causal effect} of $x_j^i$ on $x_k^i$, \textit{i.e.} the sum over all directed paths from $j$ to $k$ of the products of edge coefficients. Hence for each view $i$,
    \begin{equation}
        x_k^i = A^i_{kj} x_j^i + \epsilon^i
    \end{equation}
    where $\epsilon^i$ collects all terms not caused by $x_j^i$ and is independent of $x_j^i$. Thus, the pair $(x_j^i, x_k^i)$ behaves as if there were a direct edge with coefficient $A^i_{kj}$, and the sign analysis of Item~1 applies but with non-strict inequalities: $M_{jk}\ge 0\ge M_{kj}$ when $\vx_j \to \dots \to \vx_k$ (and conversely). If the bivariate conditions hold for $(j,k)$, the inequalities are strict; however, we do not require this here.

    \item \textbf{No relation, no common ancestor:} then $\vx_j$ and $\vx_k$ depend on disjoint disturbance sets and are independent, hence $M_{jk}=M_{kj}=0$.
    
    \item \textbf{No relation but with a common ancestor.}  
    If $\vx_j$ and $\vx_k$ are not connected by any directed path but share at least one ancestor, then their pairwise criteria $M_{jk}$ and $M_{kj}$ may take mixed signs. However, since each of $\vx_j$ and $\vx_k$ has at least one parent, there exist variables $\vx_{j'}$ and $\vx_{k'}$ such that $M_{j j'}<0$ and $M_{k k'}<0$. Thus, even if the signs of $M_{jk}$ and $M_{kj}$ are not determined, the $j$-th and $k$-th rows of $\mM$ necessarily contain negative entries.  
    For instance, with three variables, $\vx_1 \to \vx_2$ and $\vx_1 \to \vx_3$ makes $\vx_2$ and $\vx_3$ siblings: there is no directed path between them, but both rows $2$ and $3$ of $\mM$ contain negative entries due to their direct parent $\vx_1$.
\end{enumerate}


Consequently, $\vx_j$ is a root (no incoming edges) \emph{if and only if} the $j$-th row of $\mM$ has only nonnegative entries. Hence a root exists and is consistently detectable. If multiple roots exist, we pick one at random.

\subsubsection{Recovering the causal ordering}
\label{app:sssec:proof_of_theorem_part_2}

We now show the recursive step.

\begin{lemma}[Residuals preserve LiMVAM and the ordering]
    \label{app:lemma:residuals_follow_limvam}
    Let $j$ be a root variable. Regress every other variable on $\vx_j$ (within each view) and remove $\vx_j$. The residual vector still follows a LiMVAM model with the \emph{same} ordering over the remaining variables.
\end{lemma}

\begin{proof}
    The proof is a direct adaptation of~\citet[Lemma 2 and Corollary 1]{shimizu2011directlingam} to the LiMVAM model: the original result is in the single-view setting with non-Gaussian disturbances.

    Choose a permutation $\mP$ that places $\vx_j$ first for simplicity (since $\vx_j$ is a root variable, this permutation exists). In the permuted system we write
    \begin{equation}
        \vx'^i = \mT^i \vx'^i + \ve'^i, \qquad \vx'^i = \mP \vx^i, \;\; \ve'^i = \mP \ve^i
    \end{equation}
    where each $\mT^i$ is strictly lower triangular and $\vx'_1 = \vx_j$.  
    Equivalently, the model can be expressed in reduced form
    \begin{equation}
        \vx'^i = \mA^i \ve'^i, \qquad \mA^i := (\mI-\mT^i)^{-1}
    \end{equation}
    where the matrix $\mA^i$ is lower triangular with unit diagonal.
    
    The entry $A^i_{k1}$ ($k \ge 2$) quantifies the dependence of $x'^i_k$ on $e'^i_1$. Since $x'_1=e'^i_1$ is exogenous, we have
    \begin{equation}
        \E \left[ x'^i_k x'^i_1 \right] = A^i_{k1} \, \E \left[ (x'^i_1)^2 \right]
    \end{equation}
    so the least-squares regression coefficient of $x'^i_k$ on $x'^i_1$ equals $A^i_{k1}$.  
    Therefore, regressing each $x'^i_k$ ($k\ge 2$) on $x'^i_1$ yields residuals $\vr^i_{(1)} \in \R^{p-1}$ such that
    \begin{equation}
        \vr^i_{(1)} 
        = \vx'^i_{2:p} - \mA^i_{2:p,1}\,x'^i_1 
        = \mA^i_{2:p,:}\ve'^i - \mA^i_{2:p,1} e'^i_1 
        = \mA^i_{2:p,2:p}\,\ve'^i_{2:p} \enspace.
    \end{equation}
    %
    Here $\mA^i_{2:p,2:p}$ is again lower triangular with unit diagonal, and $\ve'^i_{2:p}$ has mutually independent components. Hence $\vr^i_{(1)}$ follows a LiMVAM model on $p-1$ variables. Moreover, since $\mA^i_{2:p,2:p}$ is precisely the submatrix of $\mA^i$ obtained by deleting the first row and column, the relative structure among the remaining variables is unchanged. In other words, the causal ordering of variables $2,\dots,p$ is preserved.
\end{proof}

Applying Lemma~\ref{app:lemma:residuals_follow_limvam} recursively --- find a root via $\mM$ (under Assumption~\ref{assumption:diversity_condition}), regress out its effect, remove it, and repeat --- recovers the full (partial) causal ordering consistently.



\subsubsection{Recovering the adjacency matrices}
\label{app:sssec:proof_of_theorem_part_3}

Sections~\ref{app:sssec:proof_of_theorem_part_1} and~\ref{app:sssec:proof_of_theorem_part_2} showed that, under Assumption~\ref{assumption:diversity_condition}, the algorithms PairwiseLiMVAM and DirectLiMVAM consistently recover the causal ordering. Equivalently, they recover a permutation matrix $\mP$ that makes the $\mT^i = \mP^\top \mB^i \mP$ strictly lower triangular. So, we can reorder the variables as follows:
\begin{equation}
    \vx'^i = \mT^i \vx'^i + \ve'^i \;, \qquad i \in [\![1, m]\!] \enspace.
\end{equation}

Fix $j \in [\![1, p]\!]$. We have
\begin{equation}
    x'^i_j = \sum_{k=1}^{j-1} T^i_{jk} x'^i_k + e'^i_j \;, \qquad i \in [\![1, m]\!] \enspace.
\end{equation}
Because $e'^i_j$ is independent of $(x'^i_1,\dots,x'^i_{j-1})$, the regressors are exogenous. This is a system of linear regressions across views with potentially cross-view–correlated errors. Feasible GLS is consistent in this setting: once a consistent estimate of the cross-view error covariance is obtained, the GLS estimator converges to the true coefficients~\citep{greene2003econometric}. Therefore, $\{T^i_{jk}\}_{k<j}$ are consistently estimated for every $(i,j)$, hence each $\mT^i$ is consistently recovered. Finally, $\mB^i = \mP \mT^i \mP^\top$, so the view-specific adjacency matrices are identifiable. This concludes the proof.





\newpage
\section{SOS-based Algorithms for LiMVAM}
\label{app:sec:sos_based_algorithms_limvam}

\subsection{Consistent estimator of the likelihood-based criterion}
\label{app:ssec:implementation_of_likelihood_based_criterion}

Assume we have observed $n$ independent samples, concatenated as $\mX = [\vx_1, \dots, \vx_n]$ and $\mY = [\vy_1, \dots, \vy_n]$, with sample index $j$.

In practice, the four covariance matrices in~\eqref{app:eq:true_criterion} can be estimated by consistent estimators from the samples in $\mX$ and $\mY$.
Since $b_i = c_i = \mathrm{cov}(x_i, y_i)$ for all $i$, we define
%
%
\begin{equation}
    \forall i: \qquad 
    \hat{b}_i 
    := 
    \hat{c}_i 
    :=
    \frac1n \sum_{j=1}^n x_{ij} \, y_{ij}
    \overset{p}{\longrightarrow} 
    \mathrm{cov}(x_i, y_i) 
    = 
    b_i
    =
    c_i
\end{equation}
%
so that $\hat{\bm{b}} := (\hat{b}_1, \dots, \hat{b}_m)^\top \overset{p}{\longrightarrow} \bm{b}$ and $\hat{\vc} := (\hat{c}_1, \dots, \hat{c}_m)^\top \overset{p}{\longrightarrow} \vc$. 
Using these vectors, we define the empirical covariances of the residuals
\begin{align}
    \widehat{\mSigma}_e &:= \frac1n \sum_{j=1}^n (\vy_j - \hat{\bm{b}} \odot \vx_j) (\vy_j - \hat{\bm{b}} \odot \vx_j)^\top \overset{p}{\longrightarrow} \mSigma_e \\
    \widehat{\mSigma}_d &:= \frac1n \sum_{j=1}^n (\vx_j - \hat{\bm{b}} \odot \vy_j) (\vx_j - \hat{\bm{b}} \odot \vy_j)^\top \overset{p}{\longrightarrow} \mSigma_d
\end{align}
and we also define the empirical covariances of the observations
\begin{align}
    \widehat{\mSigma}_x &:= \frac1n \sum_{j=1}^n \vx_j \vx_j^\top \overset{p}{\longrightarrow} \mSigma_x \\
    \widehat{\mSigma}_y &:= \frac1n \sum_{j=1}^n \vy_j \vy_j^\top \overset{p}{\longrightarrow} \mSigma_y \enspace.
\end{align}
We can then evaluate the criterion in these estimators.
\begin{equation}
    \widehat{\mathrm{LR}}
    :=
    \cL (\widehat{\mSigma}_x, \widehat{\mSigma}_e; x \rightarrow y)
    -
    \cL (\widehat{\mSigma}_y, \widehat{\mSigma}_d; y \rightarrow x)
    =
    - \log |\widehat{\mSigma}_x| - \log |\widehat{\mSigma}_e|
    + \log |\widehat{\mSigma}_y| + \log |\widehat{\mSigma}_d|
\end{equation}
and it follows that this empirical criterion is a consistent estimator of the true criterion:
\begin{equation}
    \cL (\widehat{\mSigma}_x, \widehat{\mSigma}_e; x \rightarrow y)
    -
    \cL (\widehat{\mSigma}_y, \widehat{\mSigma}_d; y \rightarrow x)
    \overset{p}{\longrightarrow}
    \cL (\mSigma_x, \mSigma_e; x \rightarrow y)
    -
    \cL (\mSigma_y, \mSigma_d; y \rightarrow x) 
    =: \mathrm{LR} \enspace.
\end{equation}

\subsection{Consistent estimator of the cross-covariance-based criterion}
\label{app:ssec:implementation_of_covariance_based_criterion}

Assume we have observed $n$ independent samples, concatenated as $\mX = [\vx_1, \dots, \vx_n]$ and $\mY = [\vy_1, \dots, \vy_n]$, with sample index $j$.

In practice, the matrices $\mathbb{E} \big[ \ve \vx^\top \big]$ and $\mathbb{E} \big[ \vd \vy^\top \big]$ are estimated from the data.
Performing univarite least-squares regression of $\vy$ on $\vx$ (resp. $\vx$ on $\vy$) yields the residual matrices $\widehat{\mE} = [\widehat{\ve}_1, \dots, \widehat{\ve}_n]$ (resp. $\widehat{\mD} = [\widehat{\vd}_1, \dots, \widehat{\vd}_n]$). 
The corresponding cross-covariance estimators are
\begin{align}
    \frac1n \sum_{j=1}^n \widehat{\ve}_j \vx_j^\top
    \overset{p}{\longrightarrow} 
    \mathbb{E} \big[ \ve \vx^\top \big] 
    \qquad \text{and} \qquad
    \frac1n \sum_{j=1}^n \widehat{\vd}_j \vy_j^\top
    \overset{p}{\longrightarrow} 
    \mathbb{E} \big[ \vd \vy^\top \big] \enspace.
\end{align}
Consequently, the empirical criterion
\begin{equation}
    \widehat{\mathrm{FC}}
    :=
    \left\Vert \frac1n \sum_{j=1}^n \widehat{\vd}_j \vy_j^\top \right\Vert_F
    -
    \left\Vert \frac1n \sum_{j=1}^n \widehat{\ve}_j \vx_j^\top \right\Vert_F
\end{equation}
is a consistent estimator of the population criterion:
\begin{equation}
    \widehat{\mathrm{FC}}
    \overset{p}{\longrightarrow} 
    \left\Vert \mathbb{E} \big[ \vd \vy^\top \big] \right\Vert_F
    -
    \left\Vert \mathbb{E} \big[ \ve \vx^\top \big] \right\Vert_F
    =: 
    \mathrm{FC} \enspace.
\end{equation}

\newpage
\subsection{Algorithms for estimating the two criteria}
\label{app:ssec:both_criteria_algo}

In this section, the observations $\mX, \mY \in \R^{m \times n}$ lie in two dimensions: the $m$ views, and the $n$ independent samples. We use the index $i$ for the views, and $k$ for the samples.

\begin{algorithm}[h]
    \caption{Likelihood-based criterion}
    \label{alg:likelihood_based_criterion}

    \textbf{Input:} Observations $\mX, \mY \in \R^{m \times n}$ corresponding to two variables, residuals $\mE, \mD \in \R^{m \times n}$ corresponding to the directions $x \rightarrow y$ and $y \rightarrow x$, respectively.

    \begin{enumerate}

        \item Compute the covariance matrices for the direction $x \rightarrow y$:
        \begin{align}
            \widehat{\mSigma}_x = \frac1n \sum_{j=1}^n \mX_j \mX_j^\top 
            \qquad \text{and} \qquad
            \widehat{\mSigma}_e = \frac1n \sum_{j=1}^n \mE_j \mE_j^\top
        \end{align}
        and for the direction $y \rightarrow x$:
        \begin{align}
            \widehat{\mSigma}_y = \frac1n \sum_{j=1}^n \mY_j \mY_j^\top
            \qquad \text{and} \qquad
            \widehat{\mSigma}_d = \frac1n \sum_{j=1}^n \mD_j \mD_j^\top \enspace.
        \end{align}

        \item Compute the likelihood based-criterion:
        \begin{equation}
            \widehat{\mathrm{LR}}
            =
            - \log |\widehat{\mSigma}_x| - \log |\widehat{\mSigma}_e|
            + \log |\widehat{\mSigma}_y| + \log |\widehat{\mSigma}_d| \enspace.
        \end{equation}
    \end{enumerate}
    
    \textbf{Output:} Likelihood-based criterion $\widehat{\mathrm{LR}}$.
\end{algorithm}

\begin{algorithm}[h]
    \caption{Cross-covariance-based criterion}
    \label{alg:cross_covariance_based_criterion}
    
    \textbf{Input:} Observations $\mX, \mY \in \R^{m \times n}$ corresponding to two variables, residuals $\mE, \mD \in \R^{m \times n}$ corresponding to the directions $x \rightarrow y$ and $y \rightarrow x$, respectively.

    \begin{enumerate}
        \item Compute the second-moment matrices
        \begin{align}
            \mM_1 = \frac1n \sum_{j=1}^n \mE_j \mX_j^\top 
            \qquad \text{and} \qquad
            \mM_2 = \frac1n \sum_{j=1}^n \mD_j \mY_j^\top \enspace.
        \end{align}

        \item Compute the cross-covariance-based criterion:
        \begin{equation}
            \widehat{\mathrm{FC}}
            =
            \left\Vert \mM_2 \right\Vert_F
            -
            \left\Vert \mM_1 \right\Vert_F \enspace.
        \end{equation}
    \end{enumerate}
    
    \textbf{Output:} Cross-covariance-based criterion $\widehat{\mathrm{FC}}$.
\end{algorithm}

\subsection{How to find the root variable}
\label{app:ssec:find_root_from_M}

Assume we have a scalar criterion of the causal direction between two variables $\vx_j = (x_j^1, \dots, x_j^m)^\top$ and $\vx_k = (x_k^1, \dots, x_k^m)^\top$, such that the criterion is positive if $\vx_j \rightarrow \vx_k$ and negative in the opposite direction. Let $M_{jk}$ denote the criterion obtained for the pair $(j,k)$, and collect all pairwise criteria in a matrix $\mM$ with a zero diagonal.
We now have a matrix $\mM$ whose entries  determine the causal direction by their sign: $M_{ij} > 0$ indicates $\vx_i$ causes $\vx_j$, and  $M_{ij} < 0$ indicates that $\vx_j$ causes $\vx_i$. 
In this setting, $\vx_j$ is a root variable if and only if the $j$-th row of $\mM$ contains only non-negative entries. 
We follow~\citet[Section 3.2.2]{Hyva13JMLR} to derive a principled way for finding the root variable given $\mM$. We next recall their argument.

Suppose that the entries in $\mM$ follow a normal distribution, so that $M_{ij} \sim \mathcal{N}(\mu_{ij}, \sigma^2)$ where $\sigma^2$ models the estimation error due to finite samples (and is supposed to be constant for simplicity). Then, the log-likelihood of $\vx_i$ being the root variable is
\begin{align}
    \log \prod_{j=1}^p \bbP(\mu_{ij} > 0 | M_{ij})
    &=
    \sum_{j=1}^p \log \bbP \bigg(
    \frac{\mu_{ij} - M_{ij}}{\sigma} > \frac{-M_{ij}}{\sigma} \Big| M_{ij}
    \bigg)
    \\
    &=
    \sum_{j=1}^p \log \Phi \bigg(
    \frac{M_{ij}}{\sigma}
    \bigg)
    \approx 
    -\frac{1}{2 \sigma^2}
    \sum_{j=1}^p 
    \min(0, M_{ij})^2
\end{align}
where $\Phi$ is the cumulative distribution function of a standardized Gaussian, whose log we then approximated. We therefore obtain the index of the root variable using
\begin{align}
    \arg\min_{i} \sum_{j=1}^p 
    \min(0, M_{ij})^2 \enspace.
\end{align}
%



\subsection{One-step Feasible Generalized Least Squares (FGLS)}
\label{app:ssec:one_step_FGLS}

The FGLS procedure begins with an Ordinary Least Squares (OLS) regression in each view to obtain residuals $(e'^1_j, \dots, e'^m_j)$, which are then used to estimate the cross-view covariance of the disturbances. In a second step, the OLS coefficients are adjusted  via Generalized Least Squares using this estimated covariance. In the population limit, this estimator coincides with OLS but is asymptotically more efficient whenever disturbances are correlated across views, as is the case here. 

The estimation proceeds as follows.
Assume we have observed $n$ independent samples. Here, we consider the model
\begin{equation}
    \vx'^i = \mT^i \vx'^i + \ve'^i \;, \qquad i \in [\![1, m]\!]
\end{equation}
where $\vx'^i$ and $\ve'^i$ are the reordered observations and disturbances, respectively, and $\mT^i$ are strictly lower triangular matrices. By construction, each variable $x'^i_j$ only depends on its predecessors $\vx'^i_{<j} := (x'^i_1, \dots, x'^i_{j-1})$.

Fix a variable index $j \in [\![1, p]\!]$. In view $i$, the structural equation reads
\begin{equation}
    x'^i_j = \sum_{k=1}^{j-1} T^i_{jk} \, x'^i_k + e'^i_j \enspace.
\end{equation}
Stacking all views, we obtain the block regression model
\begin{equation}
    \vx'_j
    =
    \mX_{<j} \, \mT_{<j} + \ve'_j
\end{equation}
where
\begin{itemize}
    \item $\vx'_j = \big(x'^{1}_j, \dots, x'^{m}_j \big)^\top \in \R^{mn}$ stacks all samples across the $m$ views,

    \item $\mX_{<j} = \mathrm{diag} \big(\vx'^1_{<j}, \dots, \vx'^m_{<j} \big) \in \R^{mn \times m(j-1)}$ is block-diagonal and contains the predecessors,

    \item $\mT_{<j} = \big( \mT^1_{j, <j}, \dots, \mT^m_{j, <j}\big)^\top \in \R^{m(j-1)}$ collects the coefficients,

    \item $\ve'_j = \big(e'^{1}_j, \dots, e'^{m}_j \big)^\top \in \R^{mn}$ are disturbances.
\end{itemize}
The disturbance vector $\ve'_j$ has covariance
\begin{equation}
    \Omega_j 
    := 
    \E \big[ \ve'_j (\ve'_j)^\top \big]
    =
    \mSigma_{\ve_j} \otimes \mI_n \in \R^{mn \times mn}
\end{equation}
where $\mSigma_{\ve_j} \in \R^{m \times m}$ captures cross-view covariances for component $j$.

The one-step FGLS procedure is:
\begin{enumerate}
    \item \textit{OLS step:}
    \begin{equation}
        \widehat{\mT}_{<j}^{\text{OLS}} 
        =
        \big( \mX_{<j}^\top \, \mX_{<j} \big)^{-1} \mX_{<j}^\top \, \vx'_j
    \end{equation}
    yielding residuals $\hat \ve'_{j} = \vx'_j - \mX_{<j} \, \widehat{\mT}_{<j}^{\text{OLS}}$.

    \item \textit{Covariance estimation:}
    \begin{equation}
        \widehat{\mSigma}_{\ve_j}
        = 
        \frac{1}{n} \sum_{k=1}^n \hat \ve'_{j,k} (\hat \ve'_{j,k})^\top
    \end{equation}
    where $k$ denotes the sample index.

    \item \textit{FGLS update:}
    \begin{equation}
        \widehat{\mT}_{<j}^{\text{FGLS}} 
        =
        \big( \mX_{<j}^\top \, \widehat{\mSigma}_{\ve_j}^{-1} \, \mX_{<j} \big)^{-1} \mX_{<j}^\top \, \widehat{\mSigma}_{\ve_j}^{-1} \, \vx'_j \enspace.
    \end{equation}
\end{enumerate}

This estimator reduces to OLS if $\mSigma_{\ve_j}$ is diagonal (no cross-view correlation). More generally, it achieves asymptotic efficiency by exploiting cross-view correlations.

\subsection{Computational complexity}
\label{app:ssec:computational_complexity_pairwise_direct_limvam}

The computational complexity of Algorithm~\ref{alg:pairwise}  can be broken down across the different steps. 

In Step 1, the preprocessing is in $O(m \cdot p \cdot n)$. 

In Step 2.a., we solve $O(p^2)$ LS regression problems, each of which has complexity $O(m \cdot n)$.

In Step 2.b., we compute $O(p^2)$ entries of a matrix. Each entry is computed either using the cross-covariance-based criterion in Algorithm~\ref{alg:cross_covariance_based_criterion} which is in $O(m^2 \cdot n)$, or using the likelihood-based criterion in Algorithm~\ref{alg:likelihood_based_criterion} which is in $O(m^2 \cdot n + m^3)$ due to 
computing 
log determinants. 
So the total complexity of this step is $O(m^2 \cdot p^2 \cdot n)$ or $O(m^2 \cdot p^2 \cdot n + m^3 \cdot p^2)$.

In Step 2.c., we make $O(p^2)$ operations.

In Step 2.d., we simplify redefine a matrix. 

These steps 2.a-d., are repeated $O(p)$ times, until all variables are removed, so this yields a complexity in either $O(m^2 \cdot p^3 \cdot n)$ or $O(m^2 \cdot p^3 \cdot n + m^3 \cdot p^3)$.

In Step 2.e., we reorder a matrix, which costs $O(m \cdot p \cdot n)$ operations.

In Step 3, we run the one-step Feasible GLS procedure, which is in $O(m^3 \cdot p^3 \cdot n)$. It scales cubically in the $m$ views and $p$ dimensions as it involves matrix multiplications and inversions. It then does $m$ matrix multiplications. Each of them involves permutation matrices which has a quadratic cost $O(p^2)$. So the cost of this step
is dominated by $O(m^3 \cdot p^3 \cdot n)$.

The total computational complexity of the algorithm is thus $O(m^3 \cdot p^3 \cdot n)$.

\newpage
\section{Identifiability of LiMVAM with shared disturbances}
\label{app:sec:identifiability_limvam_shared_disturbances}

In this section, we analyze the identifiability of the LiMVAM model with shared disturbances
\begin{equation}
\label{app:eq:multiview_model_shared_disturbances}
    \vx^i = \mB^i \vx^i + \mD^i \vs + \vn^i \;, \qquad i \in [\![1, m]\!]
\end{equation}
where the $\mD^i$ are diagonal matrices with positive entries on the diagonal, $\vs$ has mutually independent entries and second-order moment $\E[\vs \vs^\top] = \mI_p$, the view-specific noises are Gaussian $\vn^i \sim \cN(\vzero, \mSigma^i)$ with diagonal $\mSigma^i$, and the vectors $\vs$ and $\vn^i,i=1,\ldots,m$, are mutually independent.

\subsection{Proof of Theorem~\ref{theorem:identifiability_B_tilde_i} (first claim)}
\label{app:ssec:identifiability_B_tilde_i}

Consider two sets $\Theta = (\mD^1, \dots, \mD^m, \mSigma^1, \dots, \mSigma^m, \mB^1, \dots, \mB^m)$ 
and $\Theta' = (\mD'^1, \dots, \mD'^m, \mSigma'^1, \break \dots, \mSigma'^m, \mB'^1, \dots, \mB'^m)$ that parameterize the same statistical model in~\eqref{app:eq:multiview_model_shared_disturbances}, with $\mD^i$ and $\mD'^i$ being diagonal matrices with positive entries on the diagonal, $\mSigma^i$ and $\mSigma'^i$ being diagonal matrices with non-zero entries on the diagonal, and $\mB^i$ and $\mB'^i$ being DAG matrices.
The model in~\eqref{app:eq:multiview_model_shared_disturbances} can be reformulated as
\begin{align}
    \vx^i = (\mI - \mB^i)^{-1} \mD^i (\vs + (\mD^i)^{-1} \vn^i)
\end{align}
which corresponds to the Shared ICA model
\begin{align}
    \label{app:eq:multi_view_ICA_model}
    \vx^i = \mA^i (\vs + \tilde{\vn}^i)
\end{align}
for $\mA^i = (\mI - \mB^i)^{-1} \mD^i$ and $\tilde{\vn}^i = (\mD^i)^{-1} \vn^i$.
We can observe that the unmixing matrices $\mW^i = (\mA^i)^{-1} = (\mD^i)^{-1} (\mI - \mB^i)$ and $\mW'^i = (\mA'^i)^{-1} = (\mD'^i)^{-1} (\mI - \mB'^i)$ belong to the domain $\mathcal{W}$, which is included in the space of invertible matrices.
Thus, the two sets of $\mW^i$ and $\mW'^i$ matrices are valid sets of unmixing matrices for the same multi-view ICA model in~\eqref{app:eq:multi_view_ICA_model}.
Moreover, Assumption~\ref{assump:noise_diversity} allows the rescaled noises $\tilde{\vn}^i$ to meet the noise diversity condition of Shared ICA.
So, from the identifiability theory of multi-view ICA~\citep[Theorem $1$]{richard2021sharedica}, we know that there exist a sign-permutation matrix $\mQ$ such that for any view $i \in \llbracket 1, m \rrbracket$, we have
\begin{align}
\begin{split}
    \mW'^i
    &=
    \mQ^\top \mW^i
    \\
    \mSigma'^i
    &=
    \mQ^\top \mSigma^i \mQ \enspace.
\end{split}
\end{align}
Note that, contrary to the single-view context, the fact that we obtain $\mW$ up to \emph{sign} and permutation rather than \emph{scale} and permutation is because $\E[\vs \vs^\top] = \mI$.
Then, we apply Lemma~\ref{lemma:identifiability_lingam}, which shows that being in the domain $\mathcal{W}$ imposes $\mQ = \mI$, $\mD'^i = \mD^i$, and $\mB'^i = \mB^i$.

\subsection{Proof of Theorem~\ref{theorem:identifiability_B_tilde_i} (second claim)}
\label{app:ssec:identifiability_multiview_shared_P}

Here we prove that, under Assumptions~\ref{assump:noise_diversity} and
\ref{assumption:dense_connectivity_sos}, the decomposition of matrices $\mB^i$ into matrices $\mT^i$ and $\mP$ is unique.

In particular, Assumption \ref{assumption:dense_connectivity_sos} requires the shared ordering to be total, \textit{i.e.} there exists a directed path between any pair of variables of the (union) graph. 
This assumption can be reformulated mathematically, as follows. Let $\bar{\mP}$ be the permutation matrix of the shared ordering ($\bar{\mP}$ is unique since the ordering is total) and $\bar{\mT}^i = \bar{\mP} \bar{\mB}^i \bar{\mP}^\top$ be the corresponding strictly lower triangular matrices. We define
\begin{equation}
    \bar{\mT}^U = \sum_{i=1}^m \operatorname{abs}(\bar{\mT}^i)
\end{equation}
which is strictly lower triangular and contains the causal weights of all views, where $\operatorname{abs}$ denotes the element-wise absolute value function.
We know from graph theory that, for the adjacency matrix $\mB$ of a directed graph and any integer $k \ge 1$, the matrix $\mB^k$ counts walks of length $k$, in the sense that $\mB^k_{ij} \ne 0$ iff there is a path of length $k$ that goes from variable $j$ to variable $i$. So, the fact that the shared ordering is total implies that, for any pair of variables $i > j$, there exists an integer $k$ such that $(\bar{\mT}^U)^k_{ij} \ne 0$. Equivalently, the strictly lower triangular part of $\sum_{k=1}^{p-1} (\bar{\mT}^U)^k$ only has non-zero elements.

Now, consider the two sets $\Theta = (\mSigma^1, \dots, \mSigma^m, \mP, \mT^1, \dots, \mT^m)$ and $\Theta' = (\bar{\mSigma}^1, \dots, \bar{\mSigma}^m, \bar{\mP}, \bar{\mT}^1, \dots, \allowbreak \bar{\mT}^m)$ that parameterize the same statistical model in~\eqref{app:eq:multiview_model_shared_disturbances}, where $\mP$ is some permutation matrix, $\mT^i$ are strictly lower triangular matrices, and $\bar{\mP}$ and $\bar{\mT}^i$ are the particular matrices given by Assumption~\ref{assumption:dense_connectivity_sos}.
Note that no assumption on the sparsity of the $\mT^i$ is made.
From Theorem~\ref{theorem:identifiability_B_tilde_i}, we know that $\mSigma^i = \bar{\mSigma}^i$.

Let us define
\begin{align}
    \bar{\mW}^U = \sum_{i=1}^m \operatorname{abs}(\bar{\mW}^i)
\end{align}
where $\bar{\mW}^i = \mI - \bar{\mP}^\top \bar{\mT}^i \bar{\mP}$.
The non-zero elements of $\bar{\mP}^\top \bar{\mT}^i \bar{\mP}$ are outside of the diagonal, so matrices $\mI$ and $\bar{\mP}^\top \bar{\mT}^i \bar{\mP}$ contain non-zero elements at different locations.
Thus, we have
\begin{align}
    \operatorname{abs}(\mI - \bar{\mP}^\top \bar{\mT}^i \bar{\mP})
    =
    \mI + \operatorname{abs}(\bar{\mP}^\top \bar{\mT}^i \bar{\mP})
    \enspace.
\end{align}
Furthermore, applying $\bar{\mP}$ to the rows and columns of $\bar{\mT}^i$ only shuffles its entries, without modifying their values. So we have
\begin{align}
    \operatorname{abs}(\bar{\mP}^\top \bar{\mT}^i \bar{\mP})
    =
    \bar{\mP}^\top \operatorname{abs}(\bar{\mT}^i) \bar{\mP} 
    \enspace.
\end{align}
%
Consequently,
\begin{equation}
    \bar{\mW}^U
    =
    \sum_{i=1}^m \mI + \bar{\mP}^\top \operatorname{abs}(\bar{\mT}^i) \bar{\mP}
    =
    m\mI + \bar{\mP}^\top \bar{\mT}^U \bar{\mP}
    \enspace.
\end{equation}
%
    
Next, we apply the same reasoning to the alternative set of parameters, given by $\mW^i = \mI - \mP^\top \mT^i \mP$, and we consider $\mW^U = \sum_{i=1}^m \operatorname{abs}(\mW^i)$.
The proof of Theorem~\ref{theorem:identifiability_B_tilde_i} already implied that $\bar{\mW}^i = \mW^i$ for all $i$. Thus, we have $\bar{\mW}^U = \mW^U$ and
%
%
\begin{equation}
    \bar{\mP}^\top \bar{\mT}^U \bar{\mP} = \mP^\top \mT^U \mP
\end{equation}
where $\mT^U$ is defined in a similar way as $\bar{\mT}^U$, except that its strictly lower triangular part can be sparse.
Raising this equation to power $k$, for $k=1, \dots, p-1$, we get
\begin{equation}
\label{app:eq:identifiability_P_walks_length_k}
    \bar{\mP}^\top \left( \sum_{k=1}^{p-1} (\bar{\mT}^U)^k \right) \bar{\mP} = \mP^\top \left( \sum_{k=1}^{p-1} (\mT^U)^k \right) \mP
\end{equation}
where we recall that the strictly lower triangular part of $\sum_{k=1}^{p-1} (\bar{\mT}^U)^k$ only has non-zero elements.
Using Lemma~\ref{app:lemma:Bi_fully_connected} on~
\eqref{app:eq:identifiability_P_walks_length_k},
we obtain that $\mP = \bar{\mP}$, and thus $\mT^i = \bar{\mT}^i$ for all~$i$.
In conclusion, all the sets of DAG decompositions that parameterize the same model are equal.
We conclude that matrices $\mP$ and $\mT^i$ are unique, and thus identifiable in our terminology.

\subsection{Results for view-specific causal orderings}
\label{app:ssec:identifiability_multiview_different_Pi}

Next, we consider a case which is outside of the theory of the main paper, although a simple extension: we allow the causal orderings to be different in different views.
It turns out that in the view-specific $\mP^i$ case, identifiability is obtained by assuming that the directed acyclic graph $\mB^i$ is dense enough in each view, as formalized in the following assumption.
However, since the $\mB^i$ can be permuted to strictly lower triangular matrices, they cannot be denser than  having $\frac{p(p-1)}{2}$ non-zero entries.
\begin{assumption}[Dense connectivity in each view]
\label{assump:total_order_different_Pi}
    For each view $i$, the matrix $\mB^i$ has exactly $\frac{p(p-1)}{2}$ non-zero entries. 
\end{assumption}

Using Assumption~\ref{assump:total_order_different_Pi}, the following theorem states that, in addition to identifying $\mB^i$, $\mD^i$, and $\mSigma^i$, one can also identify $\mT^i$ and $\mP^i$.

\begin{theorem}(Identifiability of multiple causal orderings)
    \label{theorem:identifiability_multiview_different_Pi}
    Consider the LiMVAM model with multiple causal orderings. Consider the quantities $(\mD^1, \dots, \mD^m, \mSigma^1, \dots, \mSigma^m, \mP^1, \dots, \mP^m, \mT^1, \allowbreak \dots, \mT^m)$ as the set of parameters to be estimated. 
    Under Assumptions~\ref{assump:noise_diversity} and~\ref{assump:total_order_different_Pi}, all these parameters are identifiable.
\end{theorem}

In the context of view-specific causal orderings $\mP^i$, we prove that, under Assumptions~\ref{assump:noise_diversity} and~\ref{assump:total_order_different_Pi}, the decomposition of matrices $\mB^i$ into matrices $\mT^i$ and $\mP^i$ is unique.

Consider two sets of parameters $\Theta = (\mSigma^1, \dots, \mSigma^m, \mP^1, \dots, \mP^m, \mT^1, \dots, \mT^m)$ and $\Theta' = (\mSigma'^1, \dots, \allowbreak \mSigma'^m, \allowbreak \mP'^1, \dots, \mP'^m, \mT'^1, \dots, \mT'^m)$ that parameterize the same statistical model in~\eqref{app:eq:multiview_model_shared_disturbances}, where $\mP^i, \mP'^i$ are permutation matrices, and $\mT^i, \mT'^i$ are strictly lower triangular matrices.
Note that here we parameterize by $\mP^i$ and $\mT^i$ rather than $\mB^i$ or $\mW^i$.
From Theorem~\ref{theorem:identifiability_B_tilde_i}, we know that $\mSigma^i = \mSigma'^i$ and that the resulting causal matrices must be equal:
\begin{equation}\label{app:eq:PBP_prime}
    (\mP^i)^\top \mT^i \mP^i
    =
    (\mP'^i)^\top \mT'^i \mP'^i
    \enspace.
\end{equation}

Assumption~\ref{assump:total_order_different_Pi} states that, for each view $i$, the matrix $\mB^i = (\mP^i)^\top \mT^i \mP^i = (\mP'^i)^\top \mT'^i \mP'^i$ contains exactly $\frac{p (p-1)}{2}$ non-zero elements, so it is also the case for $\mT^i$ and $\mT'^i$ which thus represent fully connected graphs.
So, from Lemma~\ref{app:lemma:Bi_fully_connected}, we deduce that, for each view $i$, we have $\mP^i = \mP'^i$ and $\mT^i = \mT'^i$, which concludes the proof.

\newpage
\section{ICA-based Algorithm for LiMVAM with shared disturbances}
\label{app:sec:ica_based_algorithm_limvam}

\subsection{ICA-based Algorithm}

\begin{algorithm}[th]
    \caption{ICA-LiMVAM}
    \label{alg:ica_limvam}
    \begin{enumerate}

        \item \textit{Estimate the adjacency matrices} $\mB^i$. 

    \begin{enumerate}
        
      \item Estimate the unmixing matrices $\mW^i$ and the noise variance matrices $\mSigma^i$, by running the Shared ICA estimation algorithm on the data. \\
      The algorithm returns an estimate of the true unmixing matrix $\mW^i = \mM (\mI-\mB^i)$ but up to sign-permutation matrix $\mM$ that is the same for all views~\citep{richard2021sharedica}.
        
    \item Determine the sign-permutation indeterminacy $\mM$, using the structure of the underlying $(\mA^i)^{-1}$ that has a diagonal of ones. Thus, we can adapt the simple two-step heuristic by \citet{shimizu2006linear}.
    
    First, find a permutation matrix $\mM$ such that $\mM \mW$ has a non-zero diagonal, by solving  with the Hungarian algorithm
        \begin{align}
            \min_{\mM} \sum_{j=1}^p 
            \frac{1}{| (\mM \mW)_{jj} |}
            \;, \quad
            \mW = \sum_{i=1}^m \operatorname{abs}(\mW^i) \;,
        \end{align}
        and then rescale the rows of $\mM$ to ensure that $\mM \mW$ has a diagonal of ones
        \begin{align}
            \label{eq:rescale_diagonal}
            \mM_{ij} \leftarrow \text{sign} ((\mM \mW)_{ii}) \mM_{ij} \enspace,
        \end{align}
        where the sign function outputs $1$ or $-1$.

        \item Determine the scale matrices $\mD^i = \text{diag} \left( \frac{1}{(\mM \mW^i)_{11}}, \dots, \frac{1}{(\mM \mW^i)_{pp}} \right)$.
        
        \item Determine the causal matrices $\mB^i = \mI - \mD^i \mM \mW^i$ and update the noise variance matrices with $\mSigma^i \leftarrow (\mD^i)^2 \mSigma^i$.
    
        \end{enumerate}

        \item \textit{Determine the causal ordering} $\mP$. \\
        
        The DAG decomposition states that $\mB^i := \mP^\top \mT^i \mP$ where $\mT^i$ is lower triangular, so we penalize for that
        \begin{align}
            \min_{\mP} \sum_{l \geq k} \big(
            \mP \mB {\mP}^\top
            \big)_{kl}^2
            \;, \quad
            \mB = \sum_{i=1}^m \operatorname{abs}(\mB^i) \;.
        \end{align}
        We approximately minimize this objective using the heuristic algorithm presented in~\citet[Algorithm~C]{shimizu2006linear}.
    \end{enumerate}
\end{algorithm}

Note that the algorithm described in~\citet{chen2024individual} finds the same adjacency matrices as we do. This is because of the sign-permutation matrix $\mM$ from ICA: to determine that matrix, we ensure $\mM \mW$ has a diagonal of ones by dividing by the scalings $\mD^i$. For~\cite{chen2024individual}, these scalings are not part of the model: they are simply a part of the algorithm which determines $\mM$. For us, these scalings are part of our model. 

\subsection{Computational complexity}
\label{app:ssec:computational_complexity_ica_limvam}

We next detail the worst-case computational complexity of the ICA-LiMVAM algorithm which is in $O(tnmp^3 + mp^5)$, for $t$ iterations of the SharedICA-ML algorithm, $n$ samples, $m$ views and $p$ components. Note that the term in $p^5$ comes from Algorithm C of the original LiNGAM and is worst-case: in practice, it took a couple of iterations only in our experiments. Below are the detail of the derivations for each step of Algorithm~\ref{alg:ica_limvam}.

\paragraph{Computational complexity}
The computational complexity of ICA-LiMVAM can be broken down across the four steps in Algorithm~\ref{alg:ica_limvam}.
Step 1 calls SharedICA, with its most costly variant (SharedICA-ML) running in $O(t \cdot n \cdot m \cdot p^3)$, where $t$ is the number of iterations, $n$ the number of samples, $m$ the number of views, and $p$ the number of components.
Step 2 solves the permutation indeterminacy via a linear sum assignment in $O(m \cdot p^3)$, instead of trying $p!$ permutations.
Step 3 involves $m$ multiplications with permutation matrices and additions, costing $O(m p^2)$.
Step 4 optimizes over permutations. Instead of trying $p!$ permutations, we use the heuristic Algorithm C from LiNGAM. It is $O(m \cdot p^5)$ in the worst case (and much lower in practice). Further details are provided next:

\begin{enumerate}

    \item Run the ICA algorithm SharedICA-ML, $O(t n m p^3)$

    Each of the $t$ iterations of the optimization performs the so-called “E-step” and “M-step” of the E-M algorithm, detailed in Section 4 of~\citet{richard2021sharedica}. Their complexity is in $m n p^3$: the intuition is that for each view $m$ and each sample $n$ require some matrix operations (\textit{e.g.} addition, multiplication, computing gradients and Hessians) all of which are dominated by $p^3$. So the final complexity is in $O(t n m p^3)$.

    \item Solve the permutation-sign indeterminacy, $O(m p^3)$

    Computing the objective can be done in $O(p^2)$. The detail of this is: we begin by computing the matrix $W$ which is in $O(m p^2)$, then multiplying it with a permutation matrix which is in $O(p^2)$, and finally summing the diagonal elements which is in $O(p)$.
    
    Finding the correct permutation matrix can be done by solving a linear-sum assignment problem; this can be done using the Hungarian algorithm which is in $O(p^3)$. 

     \item Update the entries of a matrix which is in $O(p^2)$.

     \item Update the entries of a matrix which is in $O(p^2)$.

    Adding the complexities thus far leads to a complexity of $O(mp^2 + p^3)$, which is dominated by $O(mp^3)$, for Step 2.

    \item Compute the causal matrices, $O(m p^2)$

    Each of the $m$ causal matrices requires multiplying a dense matrix with a permutation which is in $O(p^2)$ and then adding a matrix which is in $O(p^2)$. The final cost is $O(m p^2)$. 

    \item Find the causal ordering(s), $O(m p^5)$

    Algorithm B in the original LiNGAM paper finds a row that has all zeros $O(p^2)$, removes it, does this $p$ times. Its complexity is in $O(p^3)$. 
    
    Algorithm C in the original LiNGAM paper calls Algorithm B at most $p (p-1) / 2 = O(p^2)$ times, so its final complexity is $O(p^2 p^3) = O(p^5)$. 
    
    When the causal ordering is shared across views: first we compute sum $m$ matrices with $p^2$ entries which is in $O(m p^2)$, and then we use Algorithm C once which is in $O(p^5)$. The final complexity is in $O(m p^2 + p^5)$ which is dominated by  $O(m p^5)$.
    
    When the causal ordering is view-dependent: for each of the $m$ views, we use Algorithm C which is in $O(p^5)$, so the final complexity is in $O(m p^5)$. 

\end{enumerate}

\newpage
\section{Experiments}
\label{app:sec:experiments}

\subsection{Synthetic experiments}
\label{app:ssec:synthetic_experiments}

\subsubsection{Data generation in Figure~\ref{fig:simulation_study}}
\label{app:sssec:details_about_figure1}

The causal effect matrices $\mB^i = \mP^\top \mT^i \mP$ are generated from a random permutation $\mP$ and strictly lower triangular $\mT^i$ obtained from a standard Gaussian; the diagonal matrices $\mD^i$ are drawn from a uniform density between~0.1 and~3; the common disturbances in $\vs$ can be either Gaussian or non-Gaussian: Gaussian disturbances $s_j$ are generated from a standardized Gaussian and their corresponding noises $n^i_j$ have standard deviation $\Sigma^i_{jj}$ obtained by sampling from a uniform density between 0 and 1, while non-Gaussian disturbances are generated from a Laplace distribution (with scale parameter equal to $\frac12$) and their corresponding noises all have a std of $\frac12$.
We use $m=5$ views, $p=4$ variables, and vary the number of samples $n$ between $10^2$ and $10^4$.

\subsubsection{Simulation study in higher dimension}
\label{app:sssec:simulation_higher_dimension}

We evaluate the methods by measuring the estimation error on the matrices $\mB^i$ across a range of settings involving more views and components than those considered in the main text. 
Specifically, we vary the number of views among \{3, 5, 8, 12, 16, 20\} and the number of components among \{3, 6, 9, 12\}, using 1000 samples and 50 random seeds for each configuration.

\begin{figure*}[ht]
    \centerline{\includegraphics[width=0.95\columnwidth]{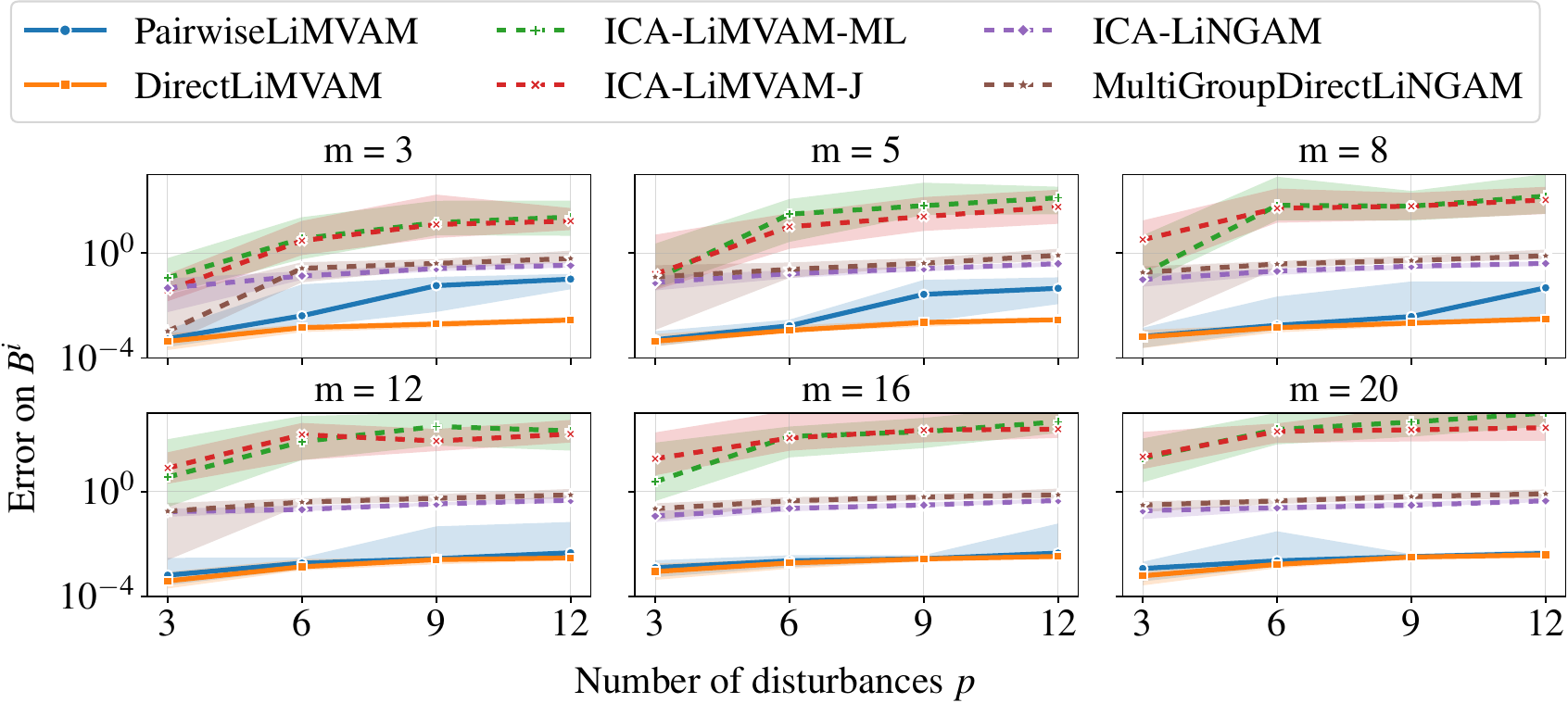}}
    \caption{
    Separation performance of three pairwise-comparisons-based algorithms --- MultiGroupDirectLiNGAM from~\citet{shimizu2012joint}, and DirectLiMVAM and PairwiseLiMVAM which use the criteria in~\eqref{eq:criterion} and~\eqref{eq:criterion_covariance}, respectively --- and three ICA-based algorithms --- ICA-LiNGAM, which naively uses~\citet{shimizu2006linear} separately in each view, and our implementation of the two variants of ICA-LiMVAM~\cite{chen2024individual}.    
    We varied the number of views and components.
    Disturbances are generated using the generalized normal distribution with shape parameter $\beta \in \{1.5, 2, 2.5\}$.
    The metric used is the $\ell_2$-distance between true and estimated causal effect matrices $\mB^i$ (lower is better).
    }
    \label{app:fig:simulation_higher_dimension}
\end{figure*}

Data are generated from the model in~\eqref{eq:multivariate_model}, where disturbances $\ve_j$ are split into three types—sub-Gaussian, Gaussian, and super-Gaussian—which motivates the choice of disturbance counts as multiples of 3. 
Sub-Gaussian disturbances follow a generalized normal distribution with shape parameter $\beta = 2.5$; super-Gaussian disturbances use $\beta = 1.5$; Gaussian disturbances correspond to $\beta = 2$. 
Each $\mT^i$ is sampled as a strictly lower-triangular matrix with Gaussian entries, and a common permutation matrix $\mP$ is used to construct the $\mB^i = \mP^\top \mT^i \mP$.
For each method, we report the median error on the matrices $\mB^i$, with the 25th and 75th percentiles shown as error bars.

As shown in Fig.~\ref{app:fig:simulation_higher_dimension}, DirectLiMVAM and PairwiseLiMVAM consistently achieve lower errors than the other approaches. DirectLiMVAM performs best when the number of views is limited (3, 5, or 8) while the number of components remains relatively high (9 or 12), thereby outperforming PairwiseLiMVAM in these settings.
This slightly surprising observation that DirectLiNGAM is a bit better than PairwiseLiNGAM is perhaps due to the fact that there the particular non-Gaussianities used here are ill-suited for the Gaussian likelihood underlying PairwiseLiNGAM.

In contrast, ICA-LiNGAM and MultiGroupDirectLiNGAM achieve slightly better-than-average performance when both the number of views and the number of components are small, but their performance quickly degrades to the average level as dimensionality increases. This behavior reflects their inability to properly handle Gaussian disturbances. Finally, both ICA-LiMVAM variants fail completely, as they are not designed to operate without shared disturbances. Overall, the estimation error tends to increase with the number of components. 


\subsubsection{Comparison with a multi-domain method}
\label{app:sssec:comparison_multi_domain}

Here, we illustrate the advantage of using multi-view methods like ours when cross-view correlations are present in the data. Specifically, we compare our DirectLiMVAM to the method of \citet{perry2022causal}, MSS (see Appendix~\ref{app:sec:distinction_multi_view_multi_domain}). We used their implementation with the KCI estimator, as this variant performs best in their paper. Since MSS recovers only the common DAG structure (and not the causal weights), we measure performance in terms of recovering the causal ordering.


\begin{wrapfigure}{l}{0.5\columnwidth}
  \vspace{-0.2cm}
  \centering
  \includegraphics[width=0.5\columnwidth]{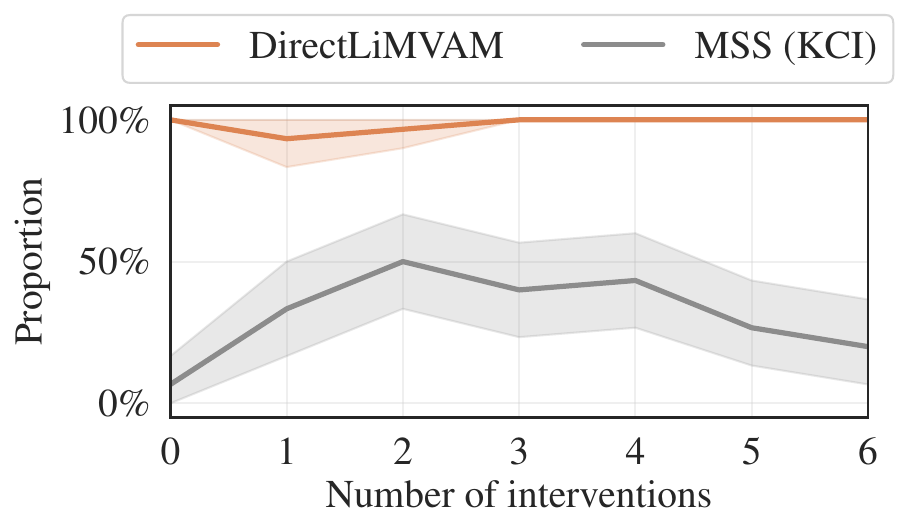}
  \caption{
  Proportion of runs in which the ordering is perfectly recovered (higher is better). The number of interventions corresponds to the number of entries in the $\mB^i$ that are allowed to vary across views. We compare our multi-view algorithm DirectLiMVAM to the multi-domain algorithm MSS of \citet{perry2022causal} (using the KCI estimator).}
  \label{app:fig:simulation_comparison_with_perry}
  \vspace{-0.3cm}
\end{wrapfigure}

We used the following experimental setup. Data were generated according to~\eqref{eq:causal_model_original} with 6 views, 4 Gaussian disturbances, 500 samples, and the experiment was repeated over 30 random seeds. The noise variances were all fixed to one, and some entries in the matrices $\mathbf{B}^i$ were allowed to vary across views. The number of such varying entries corresponds to the number of “interventions”, \textit{i.e.} changes in causal mechanisms in the sense of \citet{perry2022causal}. 

Figure~\ref{app:fig:simulation_comparison_with_perry} shows that DirectLiMVAM always recovers the correct ordering, regardless of how many entries in $\mathbf{B}^i$ are allowed to vary. By contrast, MSS fails to recover the DAG when the number of interventions is either too small or too large (in line with their Figures 4 and 5), and even in the regime of sparse interventions it perfectly recovers the ordering in only about 50\% of the runs. Finally, their method is substantially slower than ours (not shown in the graph).

These results illustrate the distinction between multi-view and multi-domain methods: by explicitly exploiting cross-view correlations, DirectLiMVAM can achieve stable and accurate recovery in settings where a multi-domain method like MSS fails. Furthermore, this experiment confirms that our methods do not require changes in the $\mathbf{B}^i$ as long as there are sufficient correlations between views, thereby extending the identifiability theories in the work just discussed.

{ 
\subsubsection{Testing Assumptions~\ref{assumption:simple_diversity_condition_bivariate} and~\ref{assumption:simple_diversity_condition}}
\label{app:sssec:assumption_1_2}

To assess whether Assumption~\ref{assumption:simple_diversity_condition} (or, in the bivariate setting, Assumpion~\ref{assumption:simple_diversity_condition_bivariate}) holds in practice, we developed a simple score taking values between 0 and 1. A score close to 0 indicates that the assumption is unlikely to hold, whereas a score close to 1 suggests that it is satisfied.

In more detail, once one of our algorithms has recovered the permutation $\mP$, variables can be reordered as $\vx'^i = \mP \vx^i$, such that each variable $x'^i_j$ depends only on its predecessors $x'^i_1, \dots, x'^i_{j-1}$. We then decompose the problem into a collection of bivariate subproblems: for each cause-and-effect pair $(\vx'_{j-1}, \vx'_j)$, we quantify to what extent Assumption~\ref{assumption:simple_diversity_condition_bivariate} is satisfied for the direction $\vx'_{j-1} \rightarrow \vx'_j$. Specifically:
\begin{enumerate}
    \item For each view $i$, we measure cross-variable correlation by computing the scalar $\mathrm{corr}(x_{j-1}^i, x_j^i)$. Intuitively, this quantity is represented by one of the two vertical arrows in Figure~\ref{fig:visualization_assumption_1}. These scalars are stored in a length-$m$ vector, from which we compute: a) the $\ell_2$-norm, to quantify the presence of cross-variable correlation, and b) the standard deviation, to quantify the diversity of these correlations.
    \item For each pair of views $(i, i')$, we regress $x'^i_j$ on $x'^i_{j-1}$ to get residuals $e'^i_j$, and similarly for view $i'$. We then measure cross-view correlation by computing $|\mathrm{corr}(x_{j-1}^i, x_{j-1}^{i'})|$ and $|\mathrm{corr}(e_{j}^i, e_{j}^{i'})|$. Intuitively, these quantities are represented by the horizontal arrows in Figure~\ref{fig:visualization_assumption_1}. The resulting scalars are stored in two different $m \times m$ matrices. We then compute: a) their $\ell_2$-norms, to quantify the presence of cross-view correlations, and b) the standard deviation (or the $\ell_2$-norm) of their difference, to quantify the diversity of these correlations.
    \item Finally, these quantities ($\ell_2$-norms and standard deviations) are aggregated (multiple aggregation schemes can be considered) and standardized to produce a final score between 0 and 1.
\end{enumerate}

We emphasize that this score is only heuristic and could likely be improved in several ways. First, it only considers consecutive pairs of variables $(\vx'_{j-1}, \vx'_j)$ rather than comparing $\vx'_j$ to all of its predecessors. Second, it relies on Assumption~\ref{assumption:simple_diversity_condition}, which itself is a simplification of Assumption~\ref{assumption:diversity_condition}. Finally, several design choices remain heuristic. Nevertheless, the score appears informative in the following simulation study.

\begin{figure}[h]
    \centering
    \includegraphics[width=0.75\linewidth]{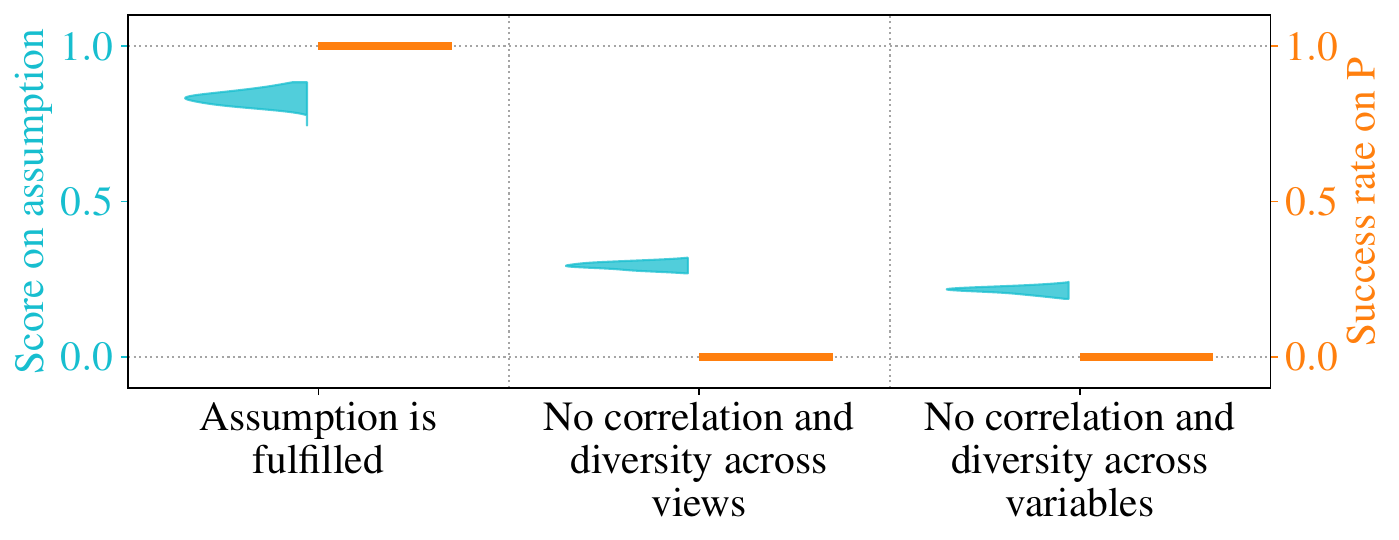}
    \caption{
    { 
    Proportion of runs in which DirectLiMVAM correctly recovers the causal ordering (orange bars), together with the proposed score evaluating whether Assumption~\ref{assumption:simple_diversity_condition} is satisfied (cyan violins). When the assumption holds, the ordering is consistently recovered and the assumption score is high; when the assumption does not hold, the recovery rate and the score are both low.}}
    \label{app:fig:assumptions_1_2}
\end{figure}

Figure~\ref{app:fig:assumptions_1_2} evaluates the practical relevance of the proposed score. We used the following experimental setup. Data were generated according to~\eqref{eq:causal_model_original} with 10 views, 5 Gaussian disturbances, 1000 samples, and the experiment was repeated over 100 random seeds. We considered three regimes when generating the disturbances: 1) Assumption~\ref{assumption:simple_diversity_condition} holds, 2) the assumption fails because neither correlation nor diversity is present across views, and 3) the assumption fails because neither correlation nor diversity is present across variables. 

For each run, DirectLiMVAM was used to recover the ordering $\mP$, which was then compared to the ground-truth ordering to compute a success rate (orange bars). We also computed the proposed score to assess whether Assumption~\ref{assumption:simple_diversity_condition} was satisfied (cyan violins).

As expected, the recovery rates and the proposed score are strongly aligned. In the first regime, where the assumption holds, both quantities are high. In the two remaining regimes, where the assumption is violated, both quantities are low. These results therefore support the practical relevance of the proposed score.
}

{ 
\subsubsection{Effect of vanishing cross-view correlations}
\label{app:sssec:cross_view_correlations_decrease}
}


\begin{wrapfigure}{r}{0.5\columnwidth}
  \vspace{0.2cm}
  \centering
  \includegraphics[width=0.5\columnwidth]{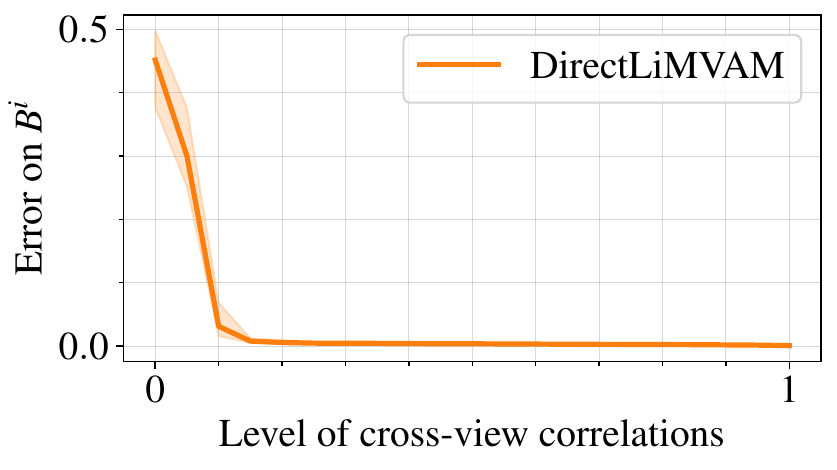}
  \caption{
  Performance of DirectLiMVAM under varying levels of cross-view correlation (lower is better). A small amount of cross-view correlation is sufficient for accurate recovery of the adjacency matrices.}
  \label{app:fig:cross_view_correlations_decrease}
  \vspace{0.3cm}
\end{wrapfigure}

Multi-view methods rely on the assumption that views are correlated. Specifically, DirectLiMVAM and PairwiseLiMVAM require both correlation and diversity (in the correlations) across views, as detailed in Assumptions~\ref{assumption:simple_diversity_condition} or~\ref{assumption:diversity_condition}. Here, we investigate the behavior of DirectLiMVAM as cross-view correlations vanish.

In Figure~\ref{app:fig:cross_view_correlations_decrease}, we vary the level of cross-view correlation from 0 (\textit{i.e.} no cross-view correlation) to 1, where cross-view correlations are generated from independent standard Gaussian distributions, as in other simulations in this paper. Data are generated according to~\eqref{eq:causal_model_original} with 6 views, 5 disturbances (drawn from sub-Gaussian, Gaussian, and super-Gaussian distributions), 1000 samples, and the experiment is repeated over 100 random seeds.

For each run, DirectLiMVAM is used to recover the adjacency matrices $\mB^i$, which are then compared to the ground-truth matrices.

We observe that the estimation error rapidly decreases to near 0: this already occurs at a cross-view correlation level of 0.1. This suggests that DirectLiMVAM requires only a small amount of cross-view correlation to perform accurately.

\vspace{0.55cm}

\subsubsection{Effect of having views with different orderings on DirectLiMVAM}
\label{app:sssec:subset_of_views_with_different_orderings}

Contrary to ICA-LiMVAM, which can deal with views associated with different causal orderings, both PairwiseLiMVAM and DirectLiMVAM require all views to share the same causal ordering.


To assess the robustness of DirectLiMVAM against violations of this assumption, we conduct an additional simulation study. Data are generated according to~\eqref{eq:causal_model_original} with 20 views, 5 disturbances (drawn from sub-Gaussian, Gaussian, and super-Gaussian distributions), 1000 samples, and the experiment is repeated over 100 random seeds. 

We first generate a main causal ordering, and progressively vary the number of views whose ordering differs from this main ordering, from 0 (\textit{i.e.} all views share the same ordering) to 20 (\textit{i.e.} each view is assigned an independently generated ordering). Importantly, alternative orderings are generated independently rather than obtained through small perturbations of the main ordering.

\begin{wrapfigure}{l}{0.5\textwidth}
  \vspace{0.5cm}
  \centering
  \includegraphics[width=0.5\textwidth]{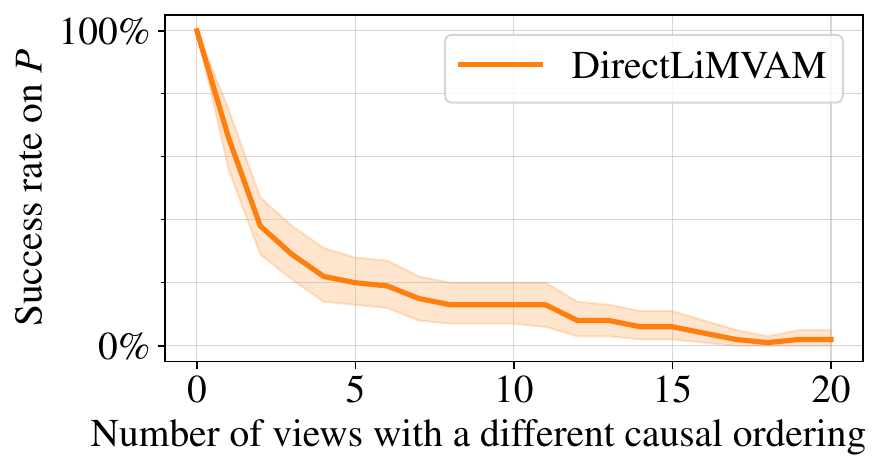}
  \caption{
  Proportion of runs in which the causal ordering is perfectly recovered (higher is better). Views with a different ordering are assigned independently generated causal orderings, which are therefore typically unrelated to the main ordering. Performance rapidly decreases as the number of views with different orderings increases.}
  \label{app:fig:subset_of_views_with_different_orderings}
  \vspace{0.2cm}
\end{wrapfigure}

For each run, DirectLiMVAM is used to estimate a single ordering $\mP$, which is then compared to the ground-truth orderings (\textit{i.e.} the main one and eventually the individual ones) to compute a success rate. Since DirectLiMVAM estimates only one ordering shared across all views, the success rate can at best decrease linearly as the number of views with different orderings increases.

In Figure~\ref{app:fig:subset_of_views_with_different_orderings}, we observe that performance rapidly deteriorates when some views have different causal orderings. This suggests that DirectLiMVAM (and similarly PairwiseLiMVAM) should primarily be used in settings where the causal ordering is strongly shared across views. Nevertheless, the degradation observed here may correspond to a particularly challenging scenario, since the alternative orderings are entirely independently generated rather than only slightly modified versions of the main ordering.

\subsubsection{Testing Assumption~\ref{assump:noise_diversity}}
\label{app:sssec:assumption_1}


\begin{wrapfigure}{r}{0.5\textwidth}
  \vspace{-0.2cm}
  \centering
  \includegraphics[width=0.5\textwidth]{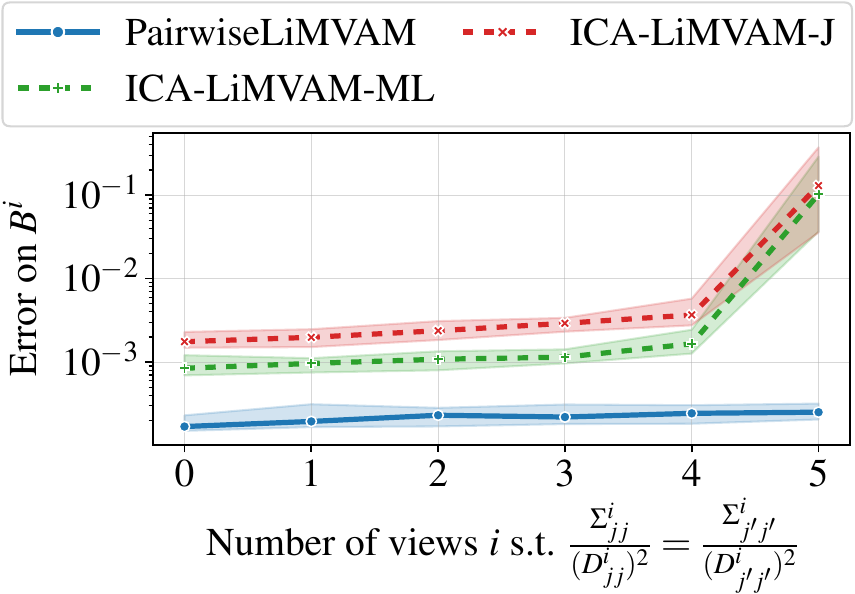}
  \caption{
  Effect of violating the noise diversity assumption (Assumption~\ref{assump:noise_diversity}) on the estimation error of the causal matrices $\mB^i$. We report the average $\ell_2$ error over 50 repetitions for ICA-LiMVAM and PairwiseLiMVAM, as the number of views with identical scaled noise variances increases.}
  \label{app:fig:assumption_noise_diversity}
\end{wrapfigure}

Assumption~\ref{assump:noise_diversity} provides a sufficient condition for the identifiability of the causal matrices $\mB^i$ in the LiMVAM model with shared, Gaussian disturbances. 
To evaluate the practical impact of this assumption, we simulated data from the model in~\eqref{eq:multivariate_model}, with shared disturbances as in~\eqref{eq:disturbance_decomposition}. We used 5 views, 4 common disturbances in $\vs$ (2 Gaussian and 2 Laplacian), and 1000 samples.

The scale matrices $\mD^i$ were drawn uniformly from the interval $[0.5, 2]$, and the strictly lower triangular matrices $\mT^i$ were sampled from a Gaussian distribution and then permuted to form the causal matrices $\mB^i$. The noise variances $\mSigma^i_{jj}$ corresponding to the Laplacian disturbances were fixed to $\frac12$, while the noise variances corresponding to the two Gaussian disturbances, indexed by $j$ and $j'$, were initially sampled uniformly in $[0, 1]$. To test the assumption, we selected an increasing number of views $i$ such that the scaled noise variances $\frac{\mSigma^i_{jj}}{\mD^i_{jj}} = \frac{\mSigma^i_{j'j'}}{\mD^i_{j'j'}}$, ranging from 0 to 5. Note that Assumption~\ref{assump:noise_diversity} only fails when this condition holds across all 5 views. We repeated this experiment over 50 random seeds.

Figure~\ref{app:fig:assumption_noise_diversity} shows that ICA-LiMVAM maintains low estimation error on the $\mB^i$ matrices as long as the noise diversity assumption holds, but its performance deteriorates sharply once the assumption is violated. This confirms the practical relevance of the assumption. Interestingly, PairwiseLiMVAM appears unaffected by the violation of this condition.

\subsubsection{Effect of the adjacency matrices' sparsity}
\label{app:sssec:assumption_2_and_3}

Assumption~\ref{assumption:dense_connectivity_sos} guarantees that the \emph{shared causal ordering} $\mP$ and the $\mT^i$ can be uniquely recovered, as soon as there exists a directed path between any two variables in the graph. Mathematically, this corresponds to assuming that $\sum_{k=1}^{p-1} \left( \sum_{i=1}^m \operatorname{abs} (\mT^i) \right)^k$ only has non-zero entries below the diagonal. In other words, it corresponds to assuming that the graph of $\sum_{i=1}^m \operatorname{abs} (\mT^i)$ is sufficiently dense and has a structure that leaves no indeterminacy between two variables.

In the following experiment, we consider a simpler version of this assumption and examine the performance of PairwiseLiMVAM and ICA-LiMVAM when varying the sparsity of the $\mT^i$. Intuitively, if the $\mT^i$ are dense enough, then Assumption~\ref{assumption:dense_connectivity_sos} will automatically be fulfilled.

In the experiment, we also study the case of \emph{view-specific (or multiple) causal orderings} --- introduced in Appendix~\ref{app:ssec:identifiability_multiview_different_Pi} --- which is outside of the theory of the paper. When causal orderings are not shared across views, assuming that all $\mB^i$ are dense (in the sense that each $\mB^i$ has exactly $\frac{p(p-1)}{2}$ non-zero entries) is sufficient to make the (multiple) orderings $\mP^i$ and the $\mT^i$ identifiable. In Appendix~\ref{app:ssec:identifiability_multiview_different_Pi}, this assumption is named Assumption~\ref{assump:total_order_different_Pi}.


\begin{figure*}[ht]
    \centerline{\includegraphics[width=0.83\columnwidth]{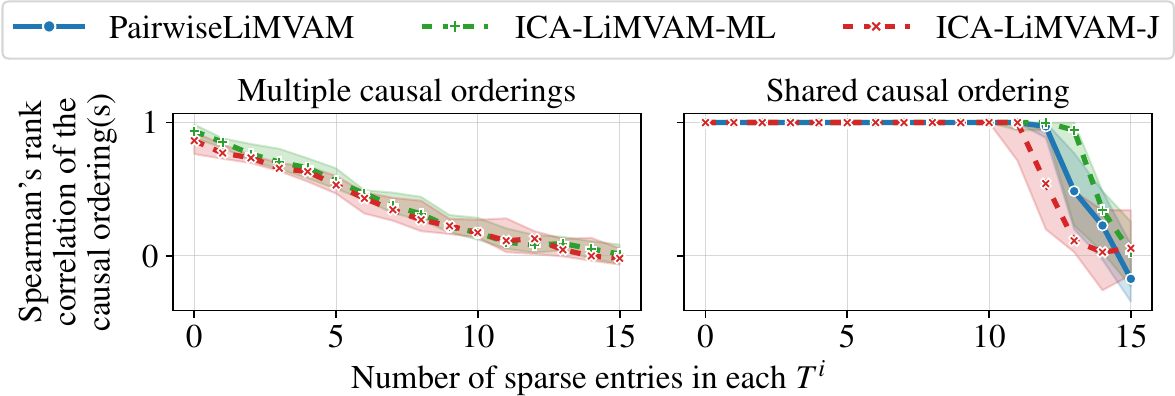}}
    \caption{
    Spearman’s rank correlation between true and estimated causal orderings as a function of the number of sparse entries in each $\mT^i$. Left: In the multiple-ordering setting, performance of ICA-LiMVAM-ML and ICA-LiMVAM-J degrades linearly with increasing sparsity, as the assumption of full support is violated. PairwiseLiMVAM is excluded since it assumes a shared ordering. Right: In the shared-ordering setting, all methods maintain near-perfect recovery up to 11 sparse entries, after which performance drops sharply.
    }
    \label{app:fig:sparsity_Ti}
\end{figure*}

To evaluate the practical impact of these two assumptions, we simulated data from the model in~\eqref{eq:multivariate_model}, with shared disturbances as in~\eqref{eq:disturbance_decomposition}. We used 8 views, 6 common disturbances (two sub-Gaussians sampled form a generalized normal distribution with shape parameter $\beta = 2.5$; two super-Gaussians for $\beta = 1.5$; two Gaussians corresponding to $\beta = 2$), and 1000 samples.
The scale matrices $\mD^i$ were drawn uniformly from the interval $[0.5, 2]$, the strictly lower triangular matrices $\mT^i$ were sampled from a Gaussian distribution and then permuted to form the causal matrices $\mB^i$, and the variances $\mSigma^i$ were sampled uniformly in $[0, 1]$.
The experiment was repeated 30 times.

Fig.~\ref{app:fig:sparsity_Ti} reports the Spearman’s rank correlation between the true and estimated causal orderings as a function of the number of sparse entries in each $\mT^i$, randomly chosen among the $\frac{6 \times 5}{2} = 15$ non-zero entries.

In the multiple (or view-specific) causal orderings setting (left panel), both ICA-LiMVAM variants exhibit a roughly linear drop in performance as sparsity increases. This behavior is expected, as Assumption~\ref{assump:total_order_different_Pi} no longer holds once any entry in the strictly lower triangular part of $\mT^i$ is set to zero. PairwiseLiMVAM is not included in this setting, as it is designed under the assumption of a shared causal ordering.

In the shared causal ordering setting (right panel), all methods benefit from the shared structure and recover the true ordering $\mP$ reliably up to 11 sparse entries. Beyond this point, performance degrades sharply for all methods, reflecting the increasing violation of Assumption~\ref{assumption:dense_connectivity_sos}.

These results empirically support the practical relevance of Assumptions~\ref{assumption:dense_connectivity_sos}
and~\ref{assump:total_order_different_Pi}.

\subsection{Real data experiments}
\label{app:ssec:real_data_experiments}

\subsubsection{Details on the preprocessing of MEG data}
\label{app:sssec:preprocessing_meg}

The MEG data measures participants' responses to auditory and visual stimuli. The auditory stimuli were binaural pure tones.
The visual stimuli were checkerboards presented both to the left and right of a central fixation for 34-ms duration.
This task leads to strong signal power modulations during the motor preparation and motor execution.

The original MEG data were acquired with 306 sensors, recorded at 1000~Hz, and band-pass filtered between 0.03 and 330~Hz. 
All MEG processing was done using the MNE-Python library~\citep{gramfort2013mnepython,gramfort2014mnepython} and we largely followed the pre-processing steps used in~\citet{power2023camcandictionarylearning}.
We applied a Maxwell filter \cite{Taulu_2006} to improve data quality and a band-pass filter between 8 and 27~Hz to focus on power effects spread over the alpha and beta bands of the brain.
This range of waves is supposed to be particularly active in sensorimotor tasks, especially during movement preparation and execution, and it typically shows a characteristic suppression (event-related desynchronization) during movement, followed by a rebound (event-related synchronization) after movement cessation, which is thought to reflect sensorimotor processing and inhibitory mechanisms.
In particular, the suppression and rebound are reflected in the energies of the signals, not raw signals.

The data were then parsed into trials synchronized to each button press, with a duration of 4.5~s, including a 1.5~s pre-movement interval.
The 4.5~s window length was selected to ensure a sufficient post-movement interval to capture the entire beta rebound response.
Trials were excluded if the button press occurred more than 1~s after the audiovisual cue (indicating poor task performance) or if another button press occurred within the time window.
Then, a baseline correction was applied using the pre-movement interval $(-1.5\,\text{s}, -1\,\text{s})$. The procedure led to about 60 trials per participant on average.

\begin{figure}[ht]
    \centering
    \includegraphics[width=\columnwidth]{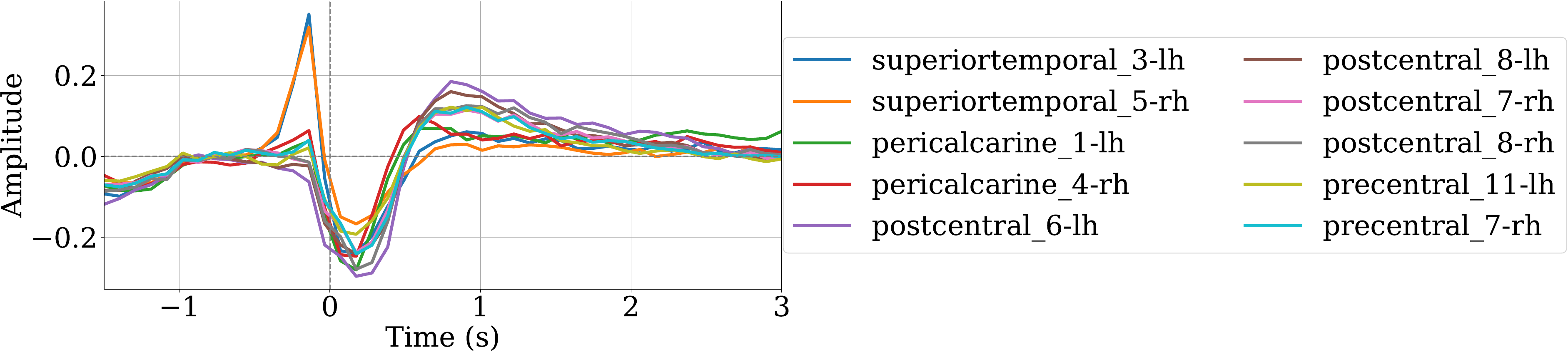}
    \caption{Example of time series obtained after the preprocessing. We averaged the data across subjects and trials, resulting in 10 time series, one for each brain region.}
    \label{app:fig:average_envelope}
\end{figure}

Next, we performed a cortical projection.
First, each participant’s MRI (additionally given in the Cam-CAN database) was segmented using FreeSurfer~\cite{dale1999cortical}.
The segmentation provided a digitization of the cortical surface for source estimation, a transformation to the average brain (\textit{i.e.} fsaverage) for spatial normalization and group statistics, and a boundary element model of the head to provide more accurate calculation of the forward solution.
The inverse solution, based on the MNE method \cite{hamalainen-mne}, allowed to consider cortical region activations for further analysis.


We used the cortical parcellation from~\cite{KHAN201857} to divide the cortical mantel (both hemispheres) into 448 distinct regions and summarized each region by an averaged time course.
Then, we selected 10 of the 448 regions based on their known importance in sensorimotor tasks.
Specifically, we picked for each hemisphere three regions in the motor cortex (two parcels in the “postcentral” and one in the “precentral”; visible in blue in Fig.~\ref{fig:brain_pairwise}), one region in the auditory cortex (“superiortemporal” parcel; highlighted in pink), and one region in the visual cortex (“pericalcarine” parcel; not visible in the figure).
Note again that the task for the participant is active as right index button presses are triggered by audiovisual stimulations.

For each participant, we thus extracted one time series for each of the 10 regions and fixed the number of trials to 40.
Participants with less than 40 trials were discarded and extra trials were averaged.
Participants for whom some of the parcels did not have any vertex in the source space were also discarded, resulting in a total of 98 available participants.
We performed a Z-score normalization of the time series to correct the depth bias and, importantly, computed the Hilbert \textit{envelope} of the signals to study modulations in cortical source power.
Finally, the signals were centered and downsampled from 1000~Hz to 10~Hz due to the slowly changing nature of the envelopes (energies).
The final dataset consisted of 98 subjects, 10 brain regions, and 1760 time points.

Figure~\ref{app:fig:average_envelope} shows the typical time series obtained after preprocessing. For visualization, we averaged the time courses across participants and trials. As expected, we observe a prominent peak in the “superiortemporal” auditory regions shortly before the button press ($t=0$\,s), reflecting the stimulus onset. In the motor regions, a characteristic rebound pattern emerges: the signals decrease around the time of the button press, consistent with event-related desynchronization of beta rhythms, and subsequently increase after approximately 0.3 seconds, indicating beta resynchronization. This figure also labels the ten cortical parcels selected for analysis.

\subsubsection{MEG experiment using ICA-LiMVAM-ML}
\label{app:sssec:MEG_using_ICSL}

In the following, we present additional analyses that extend the experiments described in Section~\ref{ssec:real_data_experiments}.

\begin{wrapfigure}{r}{0.3\textwidth}
    \vspace{-0.7cm}
    \centering
    \includegraphics[width=0.28\textwidth]{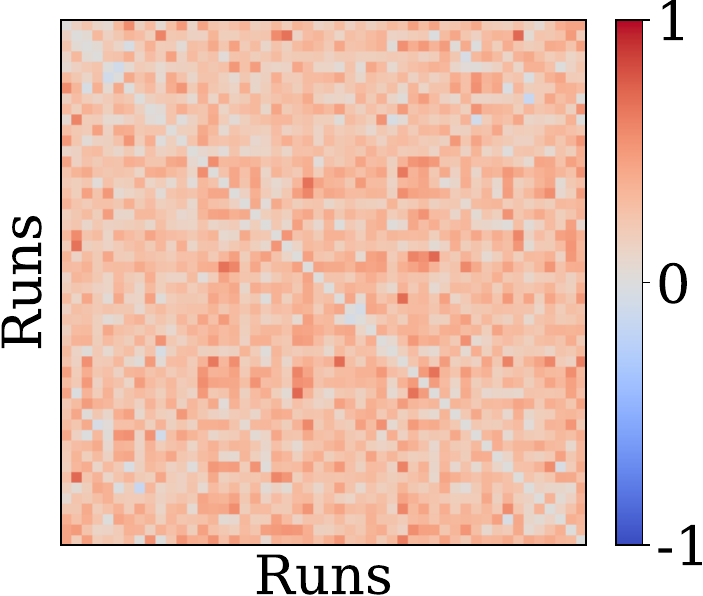}
    \caption{Pearson correlations between median causal effect matrices across 50 runs of ICA-LiMVAM-ML, each using a different random subset of 30 participants from the Cam-CAN cohort.}
    \label{app:fig:histogram_MICaDo}
    \vspace{-0.2cm} 
\end{wrapfigure}

Figure~\ref{app:fig:six_brains} displays median causal effect maps estimated using ICA-LiMVAM-ML on the Cam-CAN dataset. For each of the six panels, ICA-LiMVAM-ML was applied to a randomly chosen subset of 50\% of the participants. These results are consistent with the patterns observed in Figure~\ref{fig:brain_pairwise}, notably the frequent presence of directed connections between motor areas in opposite hemispheres. This further supports the hypothesis of consistent inter-hemispheric causal influences in sensorimotor processing.

\begin{figure*}[ht]
    \vspace{-0.2cm}
    \centerline{\includegraphics[width=0.9\columnwidth]{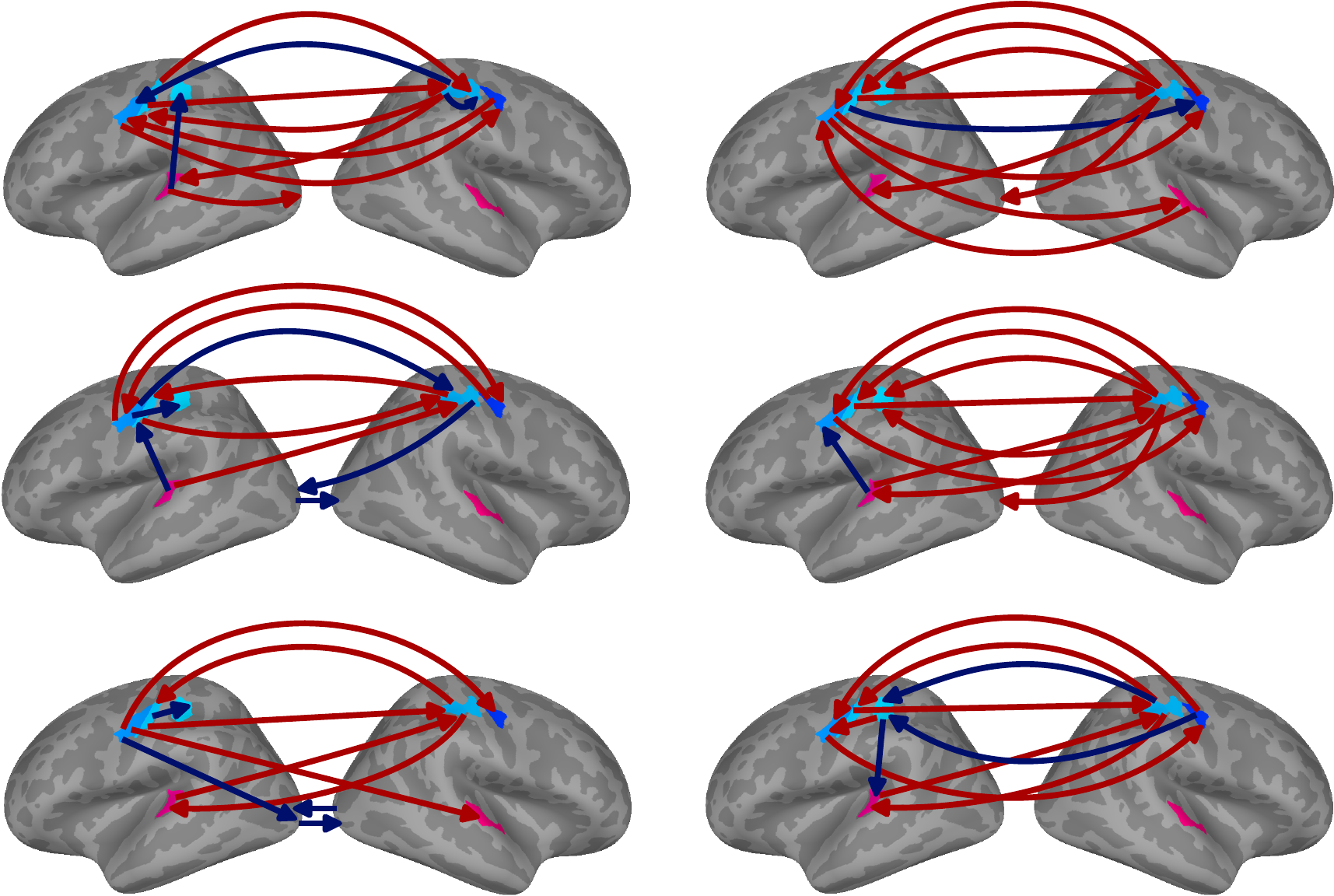}}
    \caption{
    Top ten strongest median causal effects (estimated by ICA-LiMVAM-ML) of six different runs. Each run was performed on the data of 49 randomly chosen subjects.
    Red arrows represent positive effects and blue arrows negative effects.}
    \label{app:fig:six_brains}
\end{figure*}

Finally, we conducted the same robustness analysis for ICA-LiMVAM-ML as previously done for PairwiseLiMVAM (Figure~\ref{fig:pearson_coefs_B}). Specifically, we ran ICA-LiMVAM-ML 50 times, each time using a randomly selected subset of 30 participants from the full cohort, and computed the Pearson correlations between the element-wise median of the estimated individual matrices $\mB^i$ across runs. The resulting correlations are shown in Figure~\ref{app:fig:histogram_MICaDo}. While the correlations remain predominantly positive (with an average of 0.27), they are noticeably lower than those obtained with PairwiseLiMVAM. This suggests that, in terms of stability across subsets, PairwiseLiMVAM exhibits greater consistency than ICA-LiMVAM-ML.

\clearpage
\newpage
\section{Distinction between multi-view and multi-domain frameworks}
\label{app:sec:distinction_multi_view_multi_domain}

Several frameworks study causal discovery from multiple related datasets that share the same causal structure. While \textit{multi-view} methods (like ours) model \textit{correlated} views arising from a joint (non-factorial) distribution, \textit{multi-domain} methods treat datasets as \textit{independent} domains drawn from related but distinct distributions. Consequently, multi-domain methods cannot leverage cross-view correlations and instead rely exclusively on distributional shifts across domains.
\textit{Multi-environment} methods are a special case of multi-domain methods, where distributional shifts arise from \textit{interventions} (whether designed by the practitioner or arising from uncontrolled environment differences) on the causal mechanisms. In the following, we review some multi-domain/multi-environment methods that are related to our work.

\citet{ghassami2018multi} propose a multi-domain method that, like ours, studies linear causal discovery from multiple datasets that share the same underlying DAG in the causally sufficient setting and without a non-Gaussianity assumption. However, they assume that causal mechanisms $\mathbb{P}(x^i_j | \mathbf{PA}^i_j)$ are \textit{independent} (both within and across domains), where $\mathbf{PA}^i_j$ denotes the parents of $x^i_j$. Our approach does not require such an independence assumption. For identifiability of the true DAG, their method further assumes that noise variances or causal weights vary across domains, and one of their two criteria additionally requires these changes to be sparse. In our paper, we show that such changes in noise variances are sufficient for identifiability (see Assumption~\ref{assumption:diversity_condition}), but not necessary when the views exhibit diverse correlations. 

\citet{adams2021identification} study multi-domain identifiability in the more complex setting of \textit{latent confounders}. Their main contributions are necessary and sufficient conditions (“bottleneck” and “strong non-redundancies”) that are specific to the confounded setting and vanish in our causally sufficient case. In particular, in this causally sufficient case, their conditions reduce to assuming heterogeneous variances (see their Theorem~1), which is precisely one of the sufficient conditions in our Assumption~\ref{assumption:diversity_condition}. Moreover, they require the causal weights $\mathbf{B}^i$ to \textit{remain constant} across domains, which is a strong assumption we do not make. Thus, in the causally sufficient case, our method encompasses theirs and goes much further by leveraging correlations (across views and within a view), while allowing causal weights to differ. They suggested, in fact, that their assumptions were too strong in Section~3: “Note that this theorem gives sufficient conditions; our empirical results suggest that they are not necessary.”

\citet{peters2016causal} propose a multi-environment method that aims to identify the parents of \textit{a single target variable} rather than recovering the full graph, while assuming purely \textit{Gaussian} distributions. Their approach relies on \textit{invariance} of the target’s causal mechanism across environments, which implicitly requires knowing where interventions occurred (and that the target was not intervened upon), an assumption that is often unrealistic in practice.

Similarly, \citet{perry2022causal} propose another multi-environment method that leverages invariance and can be more easily applied to causal discovery of the whole structure. However, their method considers the more complex case of \textit{nonlinear} relationships.
Like \citet{ghassami2018multi}, they rely on changes in $\mathbb{P}(x^i_j | \mathbf{PA}^i_j)$ across environments, but instead of measuring independence between mechanisms, they count how often these mechanisms change and select the DAG that minimizes this number. Their identifiability results require sparsity in the number of changes; in particular, not all causal mechanisms are allowed to differ across environments. In contrast, our theory accommodates arbitrarily many changes in $\mathbf{B}^i$ and in the distribution of $\mathbf{e}^i$; moreover, these changes are not required for identifiability if cross-view correlations are diverse.

We finally briefly mention two further works that are related to the multi-domain setting. 
A general framework, allowing for various kinds of interventions, was proposed by \citet{mooij2020joint}, but this was not claimed to improve identifiability. \citet{sturma2023unpaired} proposed a related ``multi-context" approach with the goal of estimating a representation as well as in causal representation learning (CRL).


In neuroimaging, this distinction between the two frameworks makes multi-view methods particularly well-suited to experiments centered around stimulus onsets ("phase-locked", as in our MEG and fMRI experiments), while multi-domain methods are more naturally aligned with resting-state experiments (as in the fMRI experiment of \citet{ghassami2018multi}).

\end{document}